\documentclass[lettersize,journal]{IEEEtran}
\usepackage{amsmath,amsfonts}
\usepackage{algorithmic}
\usepackage{algorithm}
\usepackage{array}
\usepackage[caption=false,font=normalsize,labelfont=sf,textfont=sf]{subfig}
\usepackage{textcomp}
\usepackage{stfloats}
\usepackage{url}
\usepackage{verbatim}
\usepackage{graphicx}
\usepackage{cite}
\usepackage{amsthm}
\usepackage{amssymb}
\usepackage{mathrsfs}
\usepackage{makecell}
\usepackage{multirow}

\newtheorem{theorem}{\bf Theorem}
\newtheorem{lemma}[theorem]{\bf Lemma}
\newtheorem{definition}[theorem]{\bf Definition}
\newtheorem{remark}[theorem]{Remark}
\newtheorem{corollary}[theorem]{Corollary}
\usepackage{balance}
\begin{document}

\title{A Homotopy Invariant Based on Convex Dissection Topology and a Distance Optimal Path Planning Algorithm}

\author{Jinyuan Liu, Minglei Fu, Andong Liu, Wenan Zhang,~\IEEEmembership{Member,~IEEE,} and Bo Chen,~\IEEEmembership{Member,~IEEE}
  \thanks{Manuscript received Month xx, 2xxx; revised Month xx, xxxx; accepted Month x, xxxx. This work was supported by the National Key Research and Development Program of China under Grant No. 2022YFE0121700.}
  \thanks{Jinyuan Liu, Minglei Fu, Andong Liu, Wenan Zhang, and Bo Chen are with the College of Information Engineering, Zhejiang University of Technology, Hangzhou, 310023, China.}
  \thanks{Minglei Fu is the corresponding author, and the phone: +86-571-85292552; fax: 86-571-85292552; e-mail: fuml@zjut.edu.cn.}}
\markboth{Journal of \LaTeX\ Class Files,~Vol.~14, No.~8, August~2021}%
{Shell \MakeLowercase{\textit{et al.}}: A Sample Article Using IEEEtran.cls for IEEE Journals}

\IEEEpubid{0000--0000/00\$00.00~\copyright~2021 IEEE}

\maketitle

\begin{abstract}
	The concept of path homotopy has received widely attention in the field of path planning in recent years. In this article, a homotopy invariant based on convex dissection for a two-dimensional bounded Euclidean space is developed, which can efficiently encode all homotopy path classes between any two points. Thereafter, the optimal path planning task consists of two steps: (i) search for the homotopy path class that may contain the optimal path, and (ii) obtain the shortest homotopy path in this class. Furthermore, an optimal path planning algorithm called CDT-RRT* (Rapidly-exploring Random Tree Star based on Convex Division Topology) is proposed. We designed an efficient sampling formula for CDT-RRT*, which gives it a tendency to actively explore unknown homotopy classes, and incorporated the principles of the Elastic Band algorithm to obtain the shortest path in each class. Through a series of experiments, it was determined that the performance of the proposed algorithm is comparable with state-of-the-art path planning algorithms. Hence, the application significance of the developed homotopy invariant in the field of path planning was verified. 
\end{abstract}


\begin{IEEEkeywords}
  Path Planning, Topology, Path Homotopy, Optimal Distance, Mobile Robots
\end{IEEEkeywords}

\section{Introduction}
\IEEEPARstart{P}{ath} planning, in robotics  \cite{tzafestas2018mobile} and autonomous driving \cite{li2015real}, aims to determine a continuous trajectory for a robot in an obstacle-free space, connecting the starting and target positions \cite{zhang2018multilevel,bopardikar2015multiobjective}. Path planning algorithms have various applications, including, service robot navigation \cite{navya2021analysis}, automatic inspection systems \cite{jianning2021beetle}, industrial automation \cite{lopez2017predictive}, autonomous vehicle systems \cite{berntorp2017path}, and robotic surgery \cite{zhao2022surgical}. Graph-search and sampling-based methods are two popular techniques for path planning in robotics. Graph-search algorithms, such as Dijkstra's algorithm \cite{dijkstra1959note} and A* \cite{hart1968formal}, solve the path planning problems applying graph theories on a discretised state space. These algorithms depend on the space discretisation and the computational cost grows exponentially as the problem scale increases. Sampling-based algorithms, such as Rapidly-exploring Random Tree (RRT) \cite{lavalle2001randomized} and Probabilistic RoadMap (PRM) \cite{kavraki1996probabilistic}, avoid the discretisation problems of graph-search techniques by randomly sampling the continuous state space. Furthermore, optimal variants, such as RRT* and PRM* \cite{karaman2011sampling}, are probabilistically complete and asymptotically optimal; hence, the probability to determine a solution approaches 100\% as the number of iterations approaches infinity. However, one of the drawbacks is that they suffer slow convergence to the optimal solution. Research on efficient and stable methods to determine optimal paths has attracted considerable attention in the field of path planning, and calculating the shortest path to the target position is essential for devices with limited energy \cite{muralidharan2019path,saito2017optimal}.\IEEEpubidadjcol

To reduce the computation time and memory usage of RRT*, the authors of RRT* suggested a pruning algorithm named Branch-and-Bound \cite{karaman2011anytime}, which determines the optimal route by periodically removing nodes with a cost higher than the optimal route cost. This method can effectively reduce the computational resource occupation of RRT*; however, it cannot improve the convergence speed of the solution to the optimal solution. The Quick RRT* (Q-RRT*) algorithm in \cite{jeong2019quick} accelerates the convergence of the RRT* algorithm by avoiding small turns through triangular inequalities in a Euclidean space. However, the optimisation effect of Q-RRT* on a path near obstacles is not evident.

Non-uniform sampling of the state space increases the convergence speed. Based on deductions, the sampling probability of some key state spaces can be increased, thus increasing the probability that the optimal route will be sampled. For example, Informed-RRT* considers that the optimal path has a higher probability between the starting point $x_{init}$ and ending point $x_{goal}$, thus the sampling probability increases within the hyperellipse, with $x_{init}$ and $x_{goal}$ as focal points, to improve the convergence speed \cite{gammell2014informed,gammell2018informed}. The authors of Informed-RRT* also combine Informed-RRT* with algorithms such as A* \cite{hart1968formal} and Fast Marching Tree* (FMT*) \cite{janson2015fast}. The resulting algorithm, Batched Informed Tree (BIT*), uses heuristics for all aspects of path cost to prioritise the search for high-quality paths and focus the search on improvements \cite{gammell2020batch}. Such methods guarantee probabilistic completeness and asymptotic optimality; however, they do not guarantee that the search acceleration for the optimal path will be efficient for all cases.

Graph-based path planning methods have been effective in solving road connectivity and fast planning problems. Kloetzer \textit{et al}. \cite{kloetzer2015optimizing} use the connection relationship of cells to determine a collision-free path that the robot can run. The map is decomposed into a set of simple cells, and the adjacency relationships among the cells are computed. Li Z \textit{et al}. \cite{li2020new} used a concave polygon convex decomposition and artificial bee colony algorithm, which has a high computational efficiency in specifying the optimal path. Path planning methods based on generalised Voronoi diagrams (GVD) have been widely studied recently \cite{chi2021generalized,wang2020optimal,huang2021path}. The premise of these methods is to use GVD to create heuristic topological maps, to guide the sampling process of RRT and further improve its path planning efficiency. However, these heuristic paths are not inevitably optimal and the GVD initialisation time for complex maps is extensive.

In the abovementioned studies, the planning method for graph-search typically increases the complexity of the graph to determine the optimal path. For the sampling-based planning method, the level of path optimality is proportional to the number of states sampled. However, these two methods significantly increase the time cost and memory usage of the algorithm. Therefore, in this article, we consider the following two questions for distance-optimal path planning: (i) Is it possible to effectively distinguish some similar paths (path homotopy), thus avoiding unproductive time used to determine the same class of paths? (ii) Is there some feature of the path class where the global optimal path is located to avoid wasting time on irrelevant path classes?

Hence, we have developed a study of homotopic path class encoders in two-dimensional (2D) bounded spaces. The contributions of our work can be summarized as follows:
\begin{itemize}
  \item{A topological encoder based on convex dissection for a 2D bounded Euclidean space is proposed, which can efficiently encode all homotopy path classes between any two points.}
  \item{A novel fast distance optimal path planning methods is developed for agents in 2D.}
  \item{A rigorous proof is provided for the theorems mentioned in the article, which is a guide to the application of the homotopy path class encoder based on convex dissection topology in path planning.}
\end{itemize}

The remainder of this article is organised as follows. The problem formulation and related mathematical definitions are presented in Section II. Section III describes the method of constructing, and the relevant properties, of a homomorphic topological graph. Section IV explains the details of our algorithms. The experimental results are presented and analysed in Section V. Section VI concludes this article.
\section{Problem Formulation}
\subsection{Mathematical Definitions}
\begin{definition}[Path]
  \label{defPath}
  Let $X$ be a continuous space, and the path in $X$ is a continuous function $f:I\to X$ from the interval $I=[0,1]$ to $X$. All paths in $X$ are denoted as $P(X)$. All paths starting at $x_0$ and ending at $x_1$ are denoted as $P(X;x_0,x_1)$. All loops with $x$ as the base point are denoted as $\Omega (X,x)$.
\end{definition}
\begin{definition}[Path Homotopy]
  \label{defPathH}
  Let $f,g\in P(X;x_0,x_1)$. If there is a continuous map $F:I\times I\to X$ such that $F(t,0)=f(t)$, $F(t,1)=g(t)$, $F(0,\tau )=x_0$, $F(1,\tau )=x_1$, then we say $f$ and $g$ are path homotopic, denoted as $f \backsimeq _p g$. The equivalence class of a path $f$ under this relation is called the homotopy class of $f$, often denoted $[f]$.
\end{definition}
\begin{definition}[Path Time Homotopy]
  \label{defPathTH}
  Let $f,g\in P(X)$, and $f$, $g$ are said to be path time homotopy if there is a continuous map $\varphi :I\to I$ and $\dot{\varphi}\geqslant 0$ such that $f(t)=g(\varphi(t))$, denoted as $f \cong _p g$.
\end{definition}
\begin{definition}[Product of Paths]
  \label{defProPath}
  For $f\in P(X;x_0,x_1)$,  $g\in P(X;x_1,x_2)$, define the product $*$ of $f$ and $g$ as,
  \begin{equation}
    \label{eq1}
    {f * g(t)} =
    \begin{cases}
      f(2 t),   & t \in[0,0.5),  \\
      g(2 t-1), & t \in[0.5,1] .
    \end{cases}
  \end{equation}
  Let $h(t)=f * g(t)$ and $h \in P(X; x_0, x_2)\subset P(X)$, denoted as $f*g\cong _p h$. The set $P(X)$ and the $*$ operation from a Groupoid $\left(P(X), *\right)$. The $*$ operation has the following properties:
  \begin{enumerate}
    \item{Associativity:  If $f*g$ and $g*h$ are meaningful, then $(f * g) * h \cong f *(g * h)$.}
    \item{Left and Right identity element: Given $x\in X$, let $e_x:I\to x$. If $f\in P(X;x_0,x_1)$, then
                \begin{equation}
                  \label{eq2}
                  e_{x_0} * f \cong_p f \text { and } f * e_{x_1} \cong_p f.
                \end{equation}
          }
    \item{Inverse element: Given $f\in P(X;x_0,x_1)$, we define the inverse of $f$ as $\bar{f}(t)=f(1-t)$, then
                \begin{equation}
                  \label{eq3}
                  f * \bar{f} \simeq_p e_{x_0} \text{ and } \bar{f} * f \simeq_p e_{x_1}.
                \end{equation}
          }
  \end{enumerate}
\end{definition}
For convenience of expression, the text will use $l^{x_1}_{x_0}$ to denote the line path from point $x_0$ to $x_1$, i.e., $l^{x_1}_{x_0}=x_1t+x_0(1-t)$.

\subsection{Path Planning Problem Formulation}
This study aims to develop a method to determine the distance optimal path of a mobile robot in a 2D bounded space. The problem description is defined as follows:

Let $X\subset\mathbb{R} ^2$ be the state space, the obstacle space and free space are denoted as $X_{obs}$ and $X_{free}=X/X_{obs}$, respectively. The path planning algorithm aims to determine a feasible path $f(t)$ such that
\begin{equation}
  \label{eq4}
  f\in P(X_{free};x_{init},x_{goal}),
\end{equation}
where $x_{init}$ denotes the initial state and $x_{goal}$ denotes the goal state. Let $P_f=P(X_{free};x_{init},x_{goal})$. The optimal path\footnote{In this study, $f^\circledast$ represents the global optimal path with the same start and end points as $f$, and $f^*$ represents the local optimal path in $[f]$.} planning problem can be defined as
\begin{equation}
  \label{eq5}
  f^\circledast = \arg \min _{f \in P_{f}} S(f),
\end{equation}
where $S:P(X)\to \mathbb{R}_{\geqslant 0}$ is a function of the length of the path, which is defined as\footnote{As the derivative of a path may not necessarily exist, defining the length function of a path as (6) is more rigorous than defining it as $S(f) = \int\|\dot{f}\|\,dt$.}
\begin{equation}
  \label{eq6}
  S(f)=\lim _{n\to  N^+}\sum_{t=0}^{n}\|f(t+n^{-1})-f(t)\|,
\end{equation}
where $\|\cdot\|$ is the 2-Norm in Euclidean space and $N^+$ is a sufficiently large positive integer.
\section{Homotopy Path Class Encoder}
In this section, we first present the method of performing a convex division of $X_{free}$ and constructing a topological graph that is homomorphic to $X_{free}$. Thereafter, the analysis of the characteristics of the topological graph is provided, and an encoder of the homotopy path class is constructed based on this topological graph.
\subsection{Polygon Fitting and Convex Division of Free Space}
\begin{figure}[!t]
  \centering
  \subfloat[]{\includegraphics[width=1.0in]{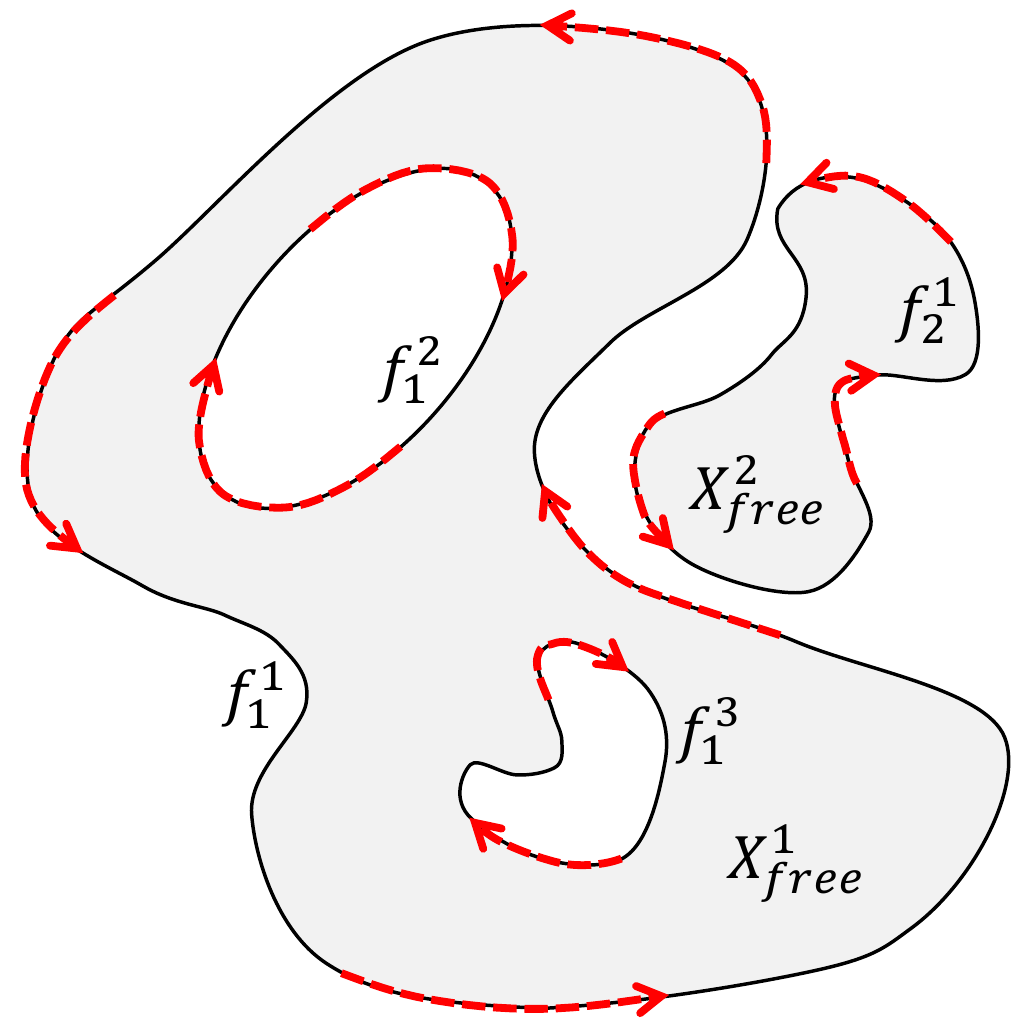}
    \label{fig_1_1}}
  \hfil
  \subfloat[]{\includegraphics[width=1.0in]{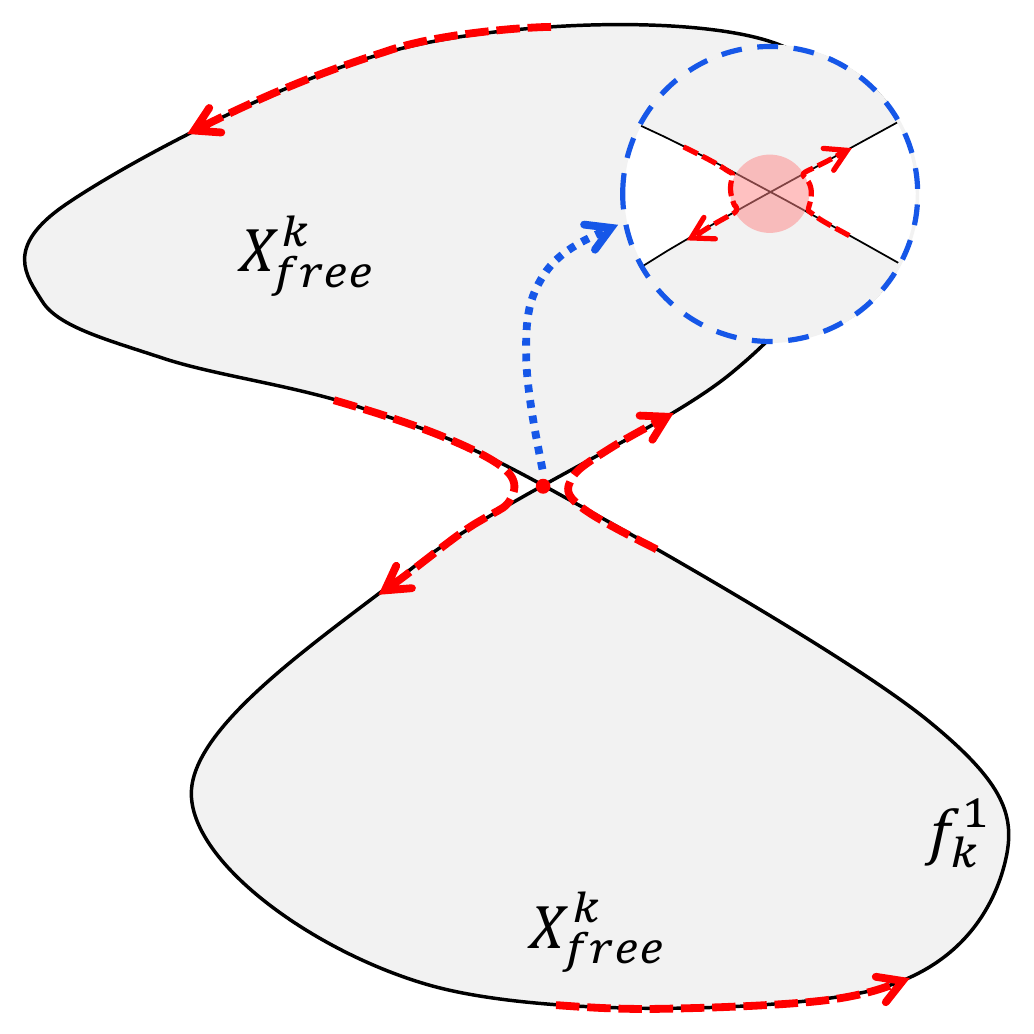}
    \label{fig_1_2}}
  \hfil
  \subfloat[]{\includegraphics[width=1.0in]{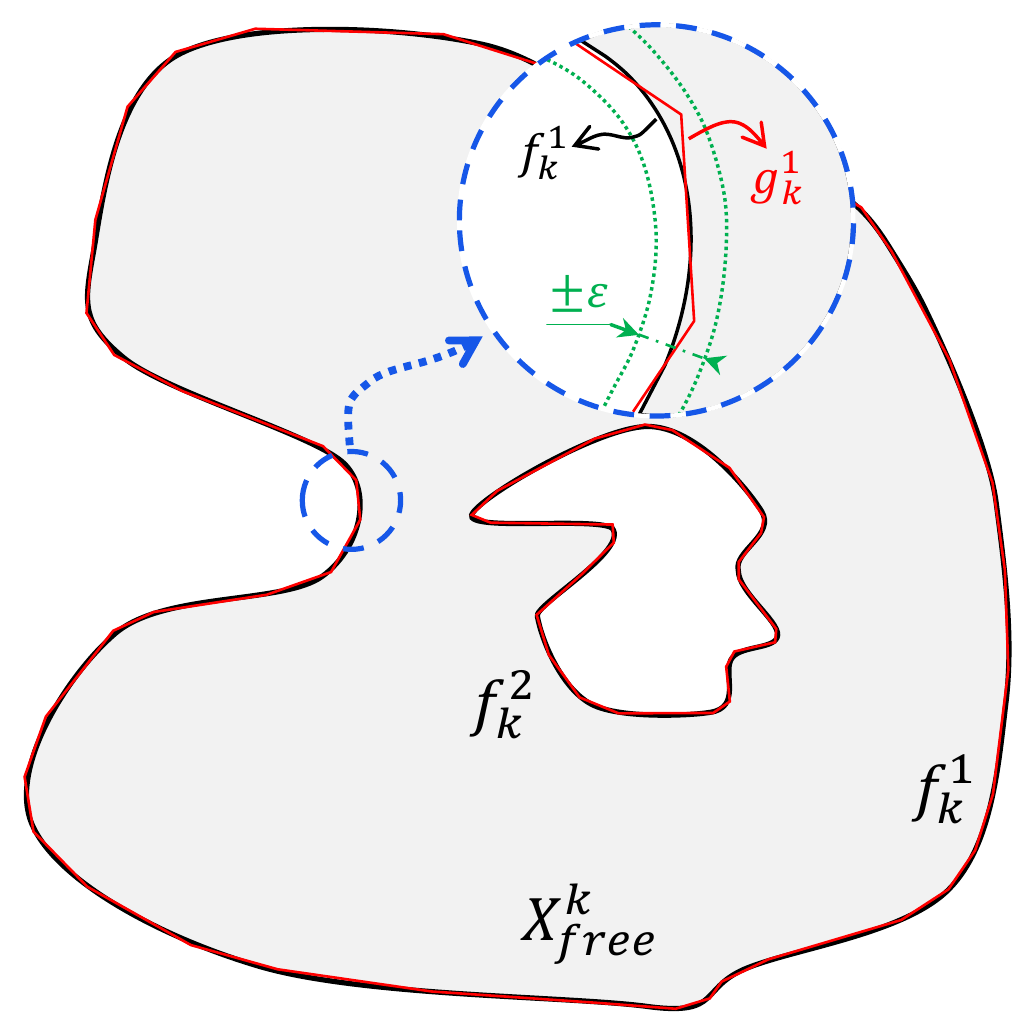}
    \label{fig_1_3}}
  \caption{Illustration of Free Space Fitting. (a) 2D bounded free space can be decomposed into a number of connected branches ($X^1_{free}$, $X^2_{free}$), which can generally be enclosed by a number of disjoint loops ($f^1_1$, $f^2_1$, $f^3_1$, $f^1_2$). (b) The neighborhood of the special point is merged into pointconnected component such that $X_{free}^k$ becomes non pointconnected component. (c) Use polygons ($g_k^1$, $g_k^2$) to fit the boundaries of connected branches, and then the simple polygons composed of $g_k^1$, $g_k^2$ realize the fitting of connected branches.}
  \label{fig_1}
\end{figure}
$X_{free}$ can be expressed as the union of multiple connected components as follows:
\begin{equation}
  \label{eq7}
  X_{free}=\bigcup_{n} X_{free }^{n},
\end{equation}
where $X_{free }^{n}$ is a connected component, and the various connected components are not connected to each other, as shown in Fig.~\ref{fig_1} (a). In a 2D space, the boundary of components that are disconnected can be represented by $K$ loops $f^k_n$ as follows:
\begin{equation}
  \label{eq8}
  f^k_n\in \Omega (X^n_{free}), \forall k\in \mathbb{N} ^K_1,
\end{equation}
\begin{equation}
  \label{eq9}
  f^{k_1}_n(t_1)\neq f^{k_2}_n(t_2), \forall k_1,k_2\in \mathbb{N} ^K_1, t_1,t_2\in [0,1),
\end{equation}
where $\mathbb{N} ^K_1$ represents a natural number from $1$ to $K$.

For pointconnected component, we can merge the domain of special points in $\mathbb{R} ^2$ into component; resulting in disconnected component points, as shown in Fig.~\ref{fig_1} (b). For any boundary $f^{k}_n(t)$ in $X_{free }^{n}$, there exists a point sequence $\{x_0,x_1,\cdots, x_M\}\subset \mathbb{R} ^2$, such that
\begin{equation}
  \label{eq10}
  g^k_n \cong_{p} \prod_{m=1}^M l^{x_m}_{x_{m-1}},
\end{equation}
and $\forall \varepsilon >0$ have
\begin{equation}
  \label{eq11}
  \sup (\| g^k_n(t)-f^k_n(t) \| ) < \varepsilon ,
\end{equation}
\begin{figure}[!t]
  \centering
  \subfloat[]{\includegraphics[width=1.5in]{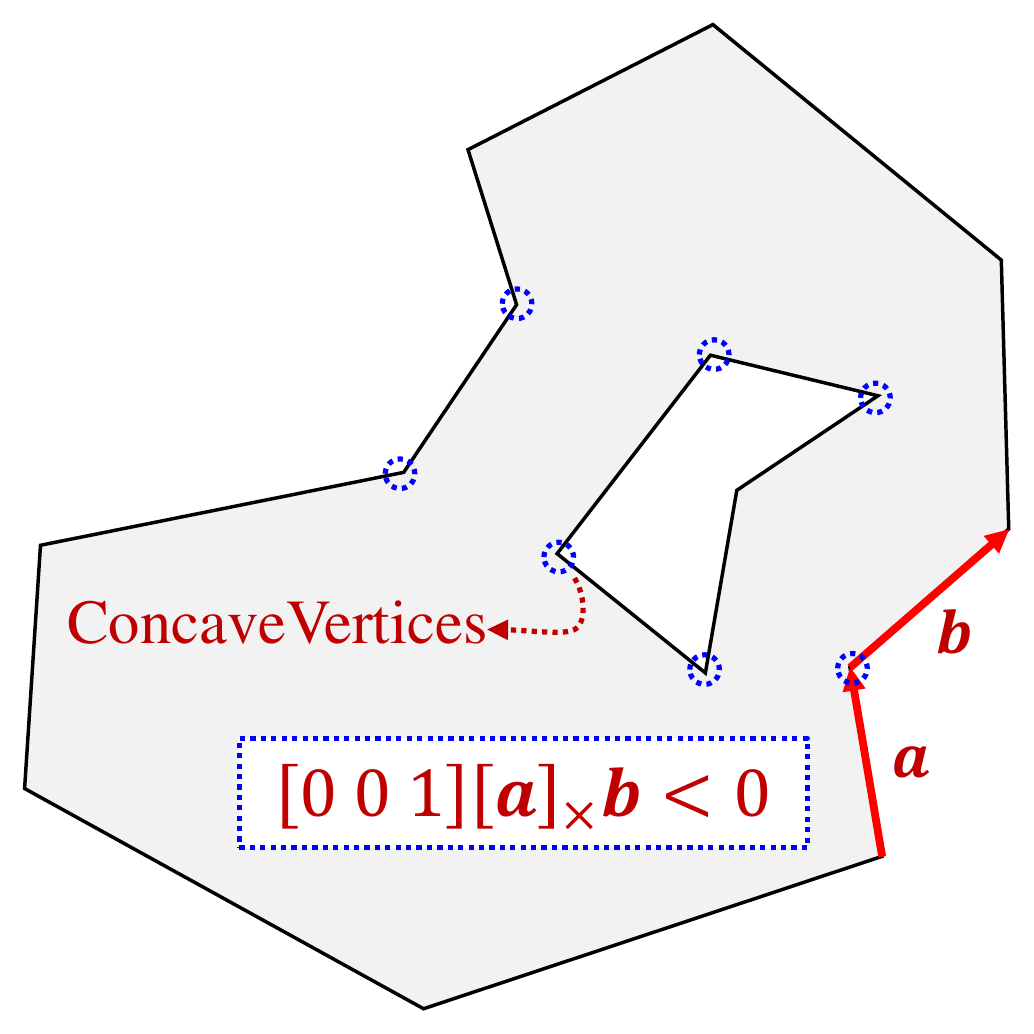}
    \label{fig_2_1}}
  \hfil
  \subfloat[]{\includegraphics[width=1.5in]{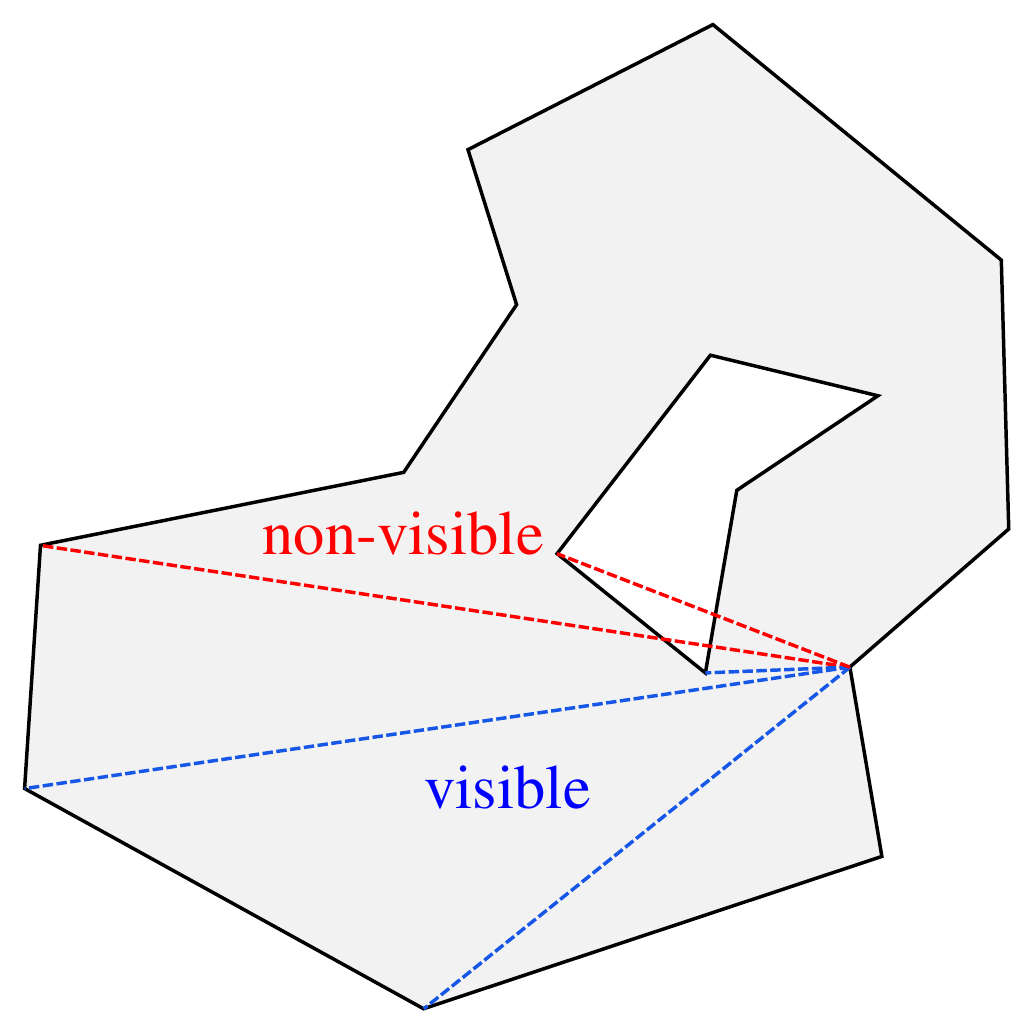}
    \label{fig_2_2}}
  \\
  \subfloat[]{\includegraphics[width=1.5in]{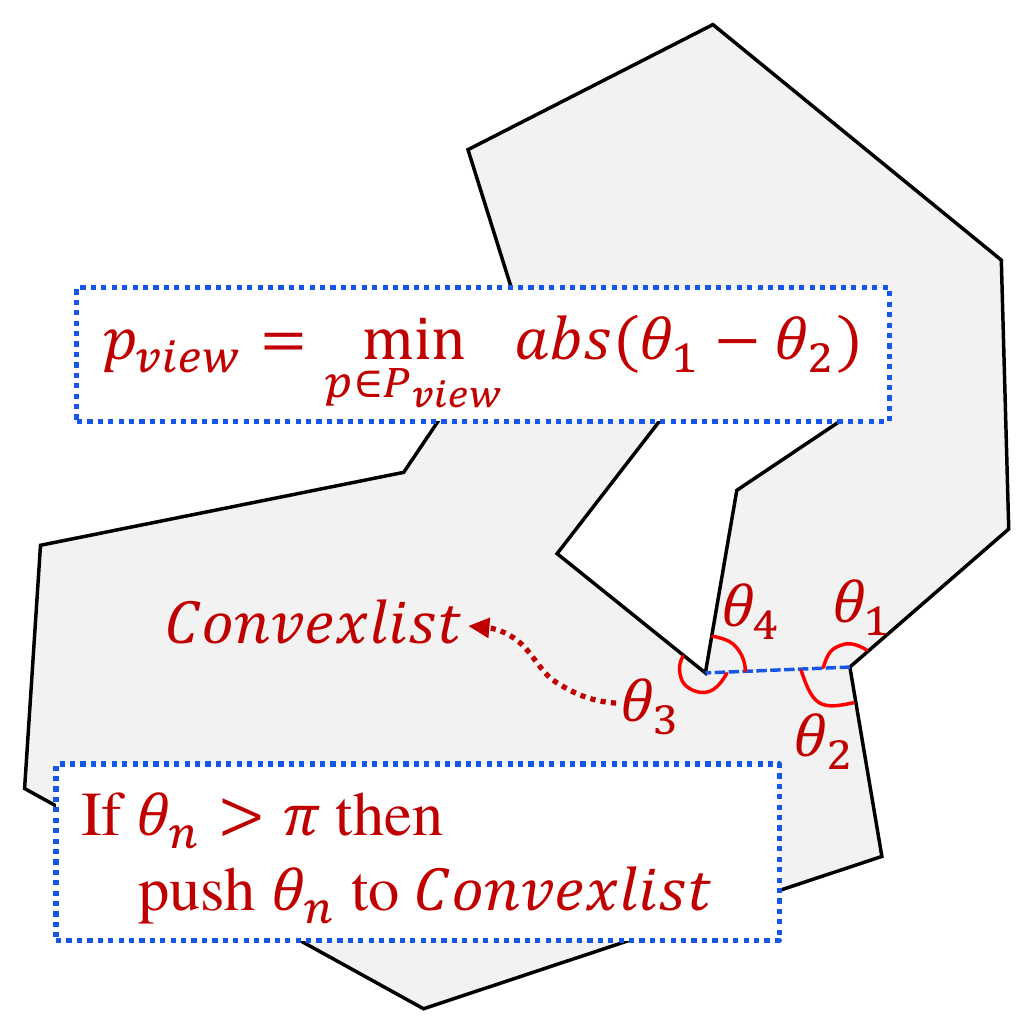}
    \label{fig_2_3}}
  \hfil
  \subfloat[]{\includegraphics[width=1.5in]{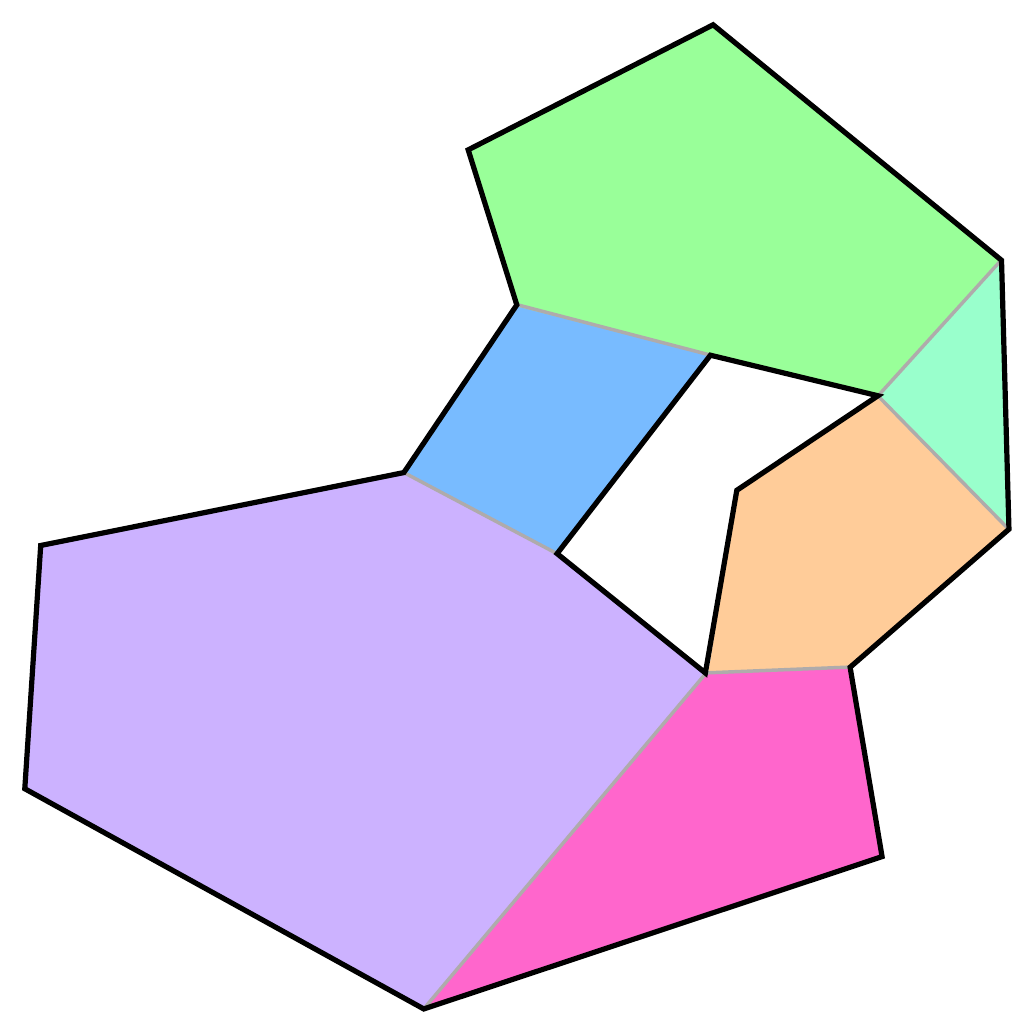}
    \label{fig_2_4}}
  \caption{Illustration of a simple polygon convex division. (a) Define the boundary direction of a simple polygon by the right-hand rule. A point where the cross product of two sides is less than 0 is a concave point. (b) Two points are not visible if the line connecting them intersects other lines. (c) When the weight is divided, the algorithm will select a relatively balanced visible point to divide the concave vertices. If the operation creates a new concave corner, it needs to be pushed into Concavelist. (d) Decomposition of a simple polygon into the union of multiple convex polygons.}
  \label{fig_2}
\end{figure}
Therefore, each connected component can be fitted by a simple polygon\footnote{A simple polygon is a polygon whose boundary does not intersect itself. Its boundary can be partitioned into two parts, an interior and an exterior.}. Then, a convex division of the simple polygon corresponding to $X^n_{free}$ is implemented, as shown in Algorithm~\ref{alg:alg1}. We firstly find all the concave vertices in the $SimplePolygon$, and push these concave vertices into $Concavelist$ (line 2). Thereafter, the iteration is commenced with the condition that  $Concavelist$ is not empty. At each iteration, a concave vertex $a_{con}$ is popped from the end of $Concavelist$ (line 4). Thereafter, the set of all viewable points of $A_{view}$ is obtained (line 5) \cite{navya2021analysis}.\footnote{According to the two ears theorem, a simple polygon always has at least one viewable points at each of its concave vertices.} Consequently, the algorithm uses WeightCut to select a vertex in $A_{view}$ to segment $a_{con}$, and obtain a cutline $l_{cut}$ and set $A_{cnew}$, here $A_{cnew}$ contains the newly generated concave vertices (line 6). The strategy we use for WeightCut is to make the two angles of the cut as equal as possible.\footnote{It is important to note that the WeightCut approach is used to select a suitable viewable point in $A_{view}$ for partitioning a concave vertex. Other strategies can be adopted for WeightCut, such as selecting the viewable points closest to the concave vertex or the viewable points with the maximum angle. However, the choice of strategy for selecting the viewable points does not affect our subsequent research.} Thereafter, $l_{cut}$ is added to the polygon (line 7). The newly generated concave vertex is pushed to the end of $Concavelist$ (line 8-9). Finally, Merging the edges in the polygons and the cutlines to generate the $polygonlist$ (line 10). Using this algorithm, depicted in Fig.~\ref{fig_2}, a convex division for each $X^n_{free}$ was achieved. Compared with the Voronoi Graph and Triangulation methods, this graph division method has the following characteristics:
\begin{enumerate}
  \item{$X_{free}$ is divided into the union of multiple convex polygons.}
  \item{No new vertices are created after the division, and the endpoints of the cutline remain the vertices of the original polygon.}
  \item{Paths from a convex polygon can only pass through its cutlines into other convex polygons.}
  \item{For any two points $x_0$, $x_1$ in a convex polygon $X_{con}$ can be connected by a line path $l^{x_1}_{x_0}$, and $l^{x_1}_{x_0} \in P(X_{con})$.}
\end{enumerate}
\begin{algorithm}[H]
  \caption{Convex Division of The Simple Polygon.}\label{alg:alg1}
  \begin{algorithmic}[1]
    \STATE {\textbf{Input: }}$SimplePolygon$
    \STATE $Concavelist \gets $FindConcave($SimplePolygon$)
    \STATE \textbf{while} $Concavelist\neq \varnothing $ \textbf{do}
    \STATE \hspace{0.5cm}$a_{con} \gets $PopEnd($Concavelist$)
    \STATE \hspace{0.5cm}$A_{view} \gets $ViewablePoint($a_{con}$)
    \STATE \hspace{0.5cm}$l_{cut},A_{cnew} \gets $ WeightCut($a_{con},A_{view}$)
    \STATE \hspace{0.5cm}AddLineToPolygon($l_{cut}$)
    \STATE \hspace{0.5cm}\textbf{if}  $A_{cnew}\neq \varnothing $ \textbf{then}
    \STATE \hspace{1.0cm}Push $A_{cnew}$ to the end of $Concavelist$
    \STATE Merging edges and cutlines get $polygonlist$.
    \STATE \textbf{return} $polygonlist$
  \end{algorithmic}
  \label{alg1}
\end{algorithm}
\subsection{Topology Graph and Encoder construction}
In this subsection, we present the construction of a topological graph using the divided polygons, discuss its relevant properties, and propose a homotopy path class encoder.

According to convex division, we can consider each convex polygon as a node and the cutting line as the line connecting these nodes. The construction of the topological graph is completed by connecting these nodes and lines. We denote the space of this topological graph as $X_{topo}$. The relevant mathematical representation of a path in $X_{topo}$ is defined by analogy to Euclidean space as follows:
\begin{figure}[!t]
  \centering
  \subfloat[]{\includegraphics[width=1.5in]{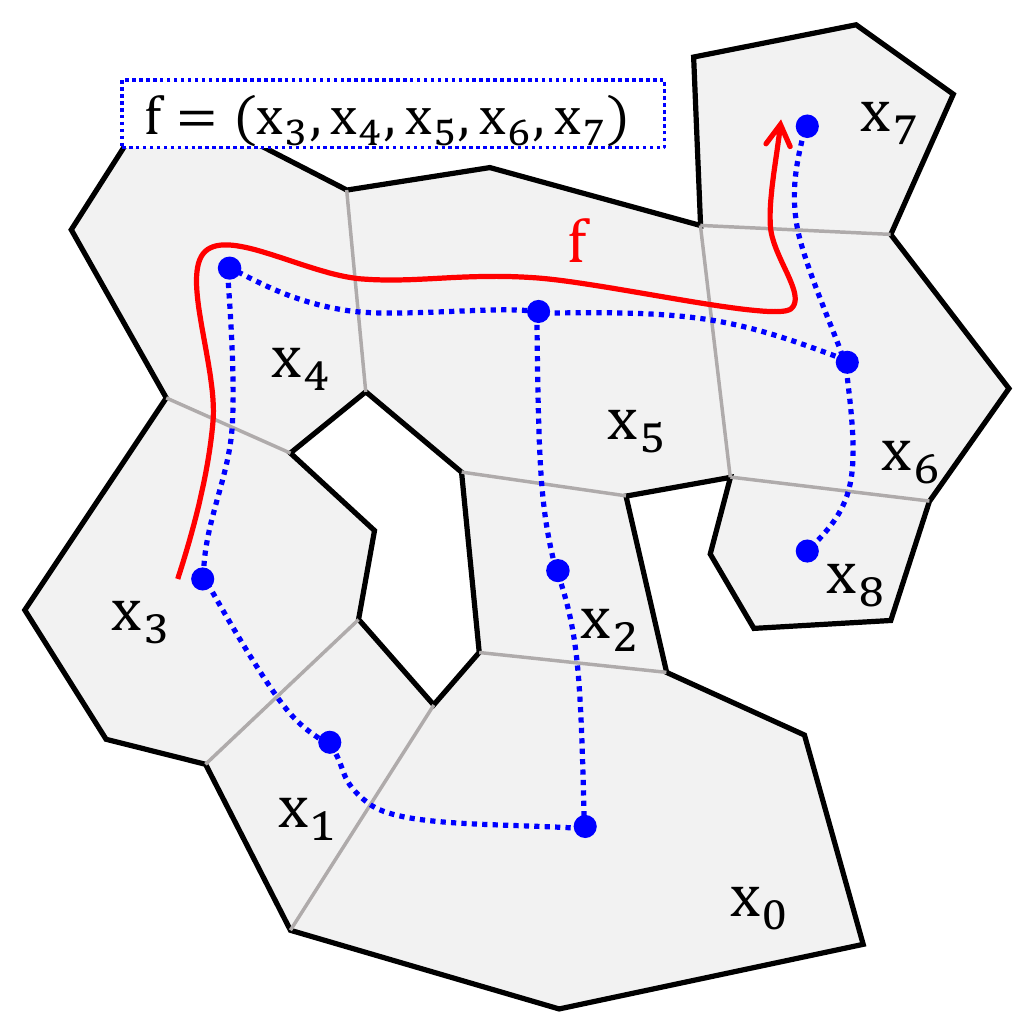}
    \label{fig_3_1}}
  \hfil
  \subfloat[]{\includegraphics[width=1.5in]{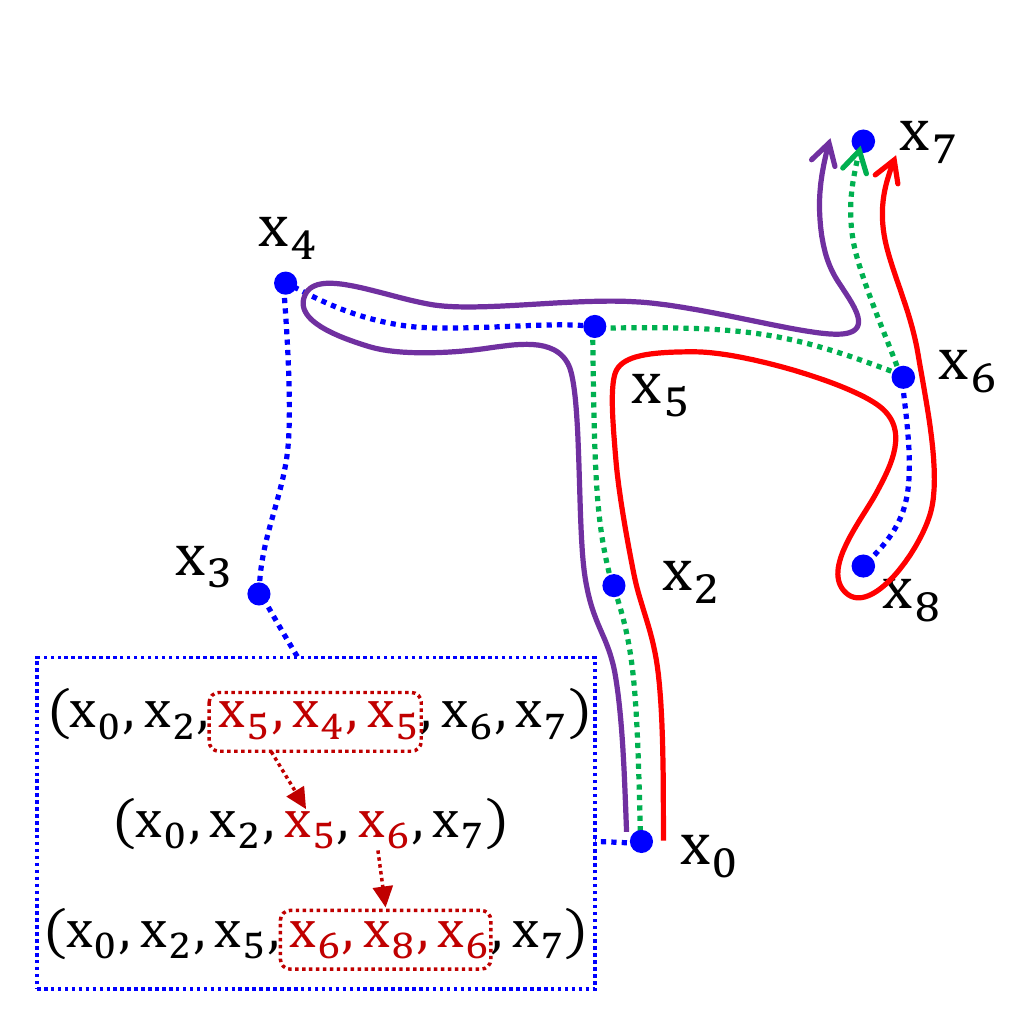}
    \label{fig_3_2}}
  \\
  \subfloat[]{\includegraphics[width=1.5in]{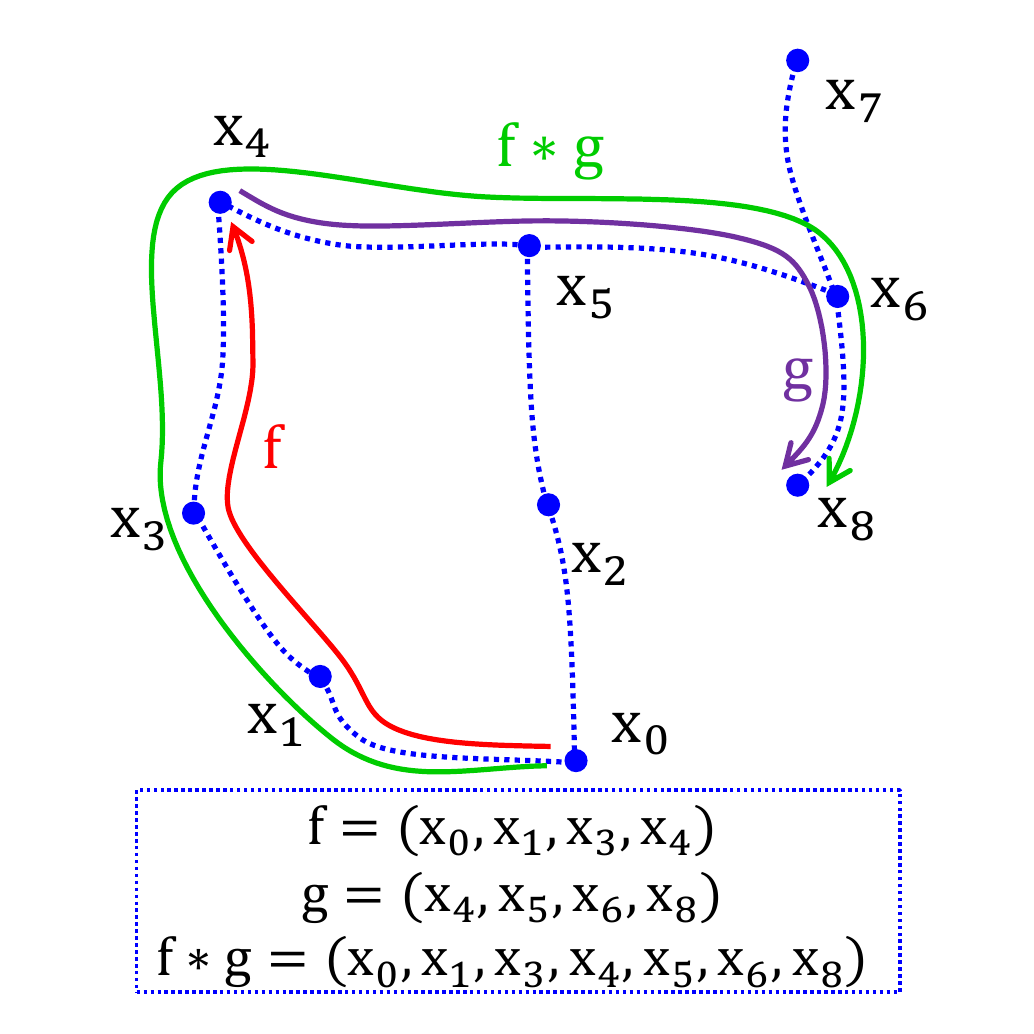}
    \label{fig_3_3}}
  \hfil
  \subfloat[]{\includegraphics[width=1.5in]{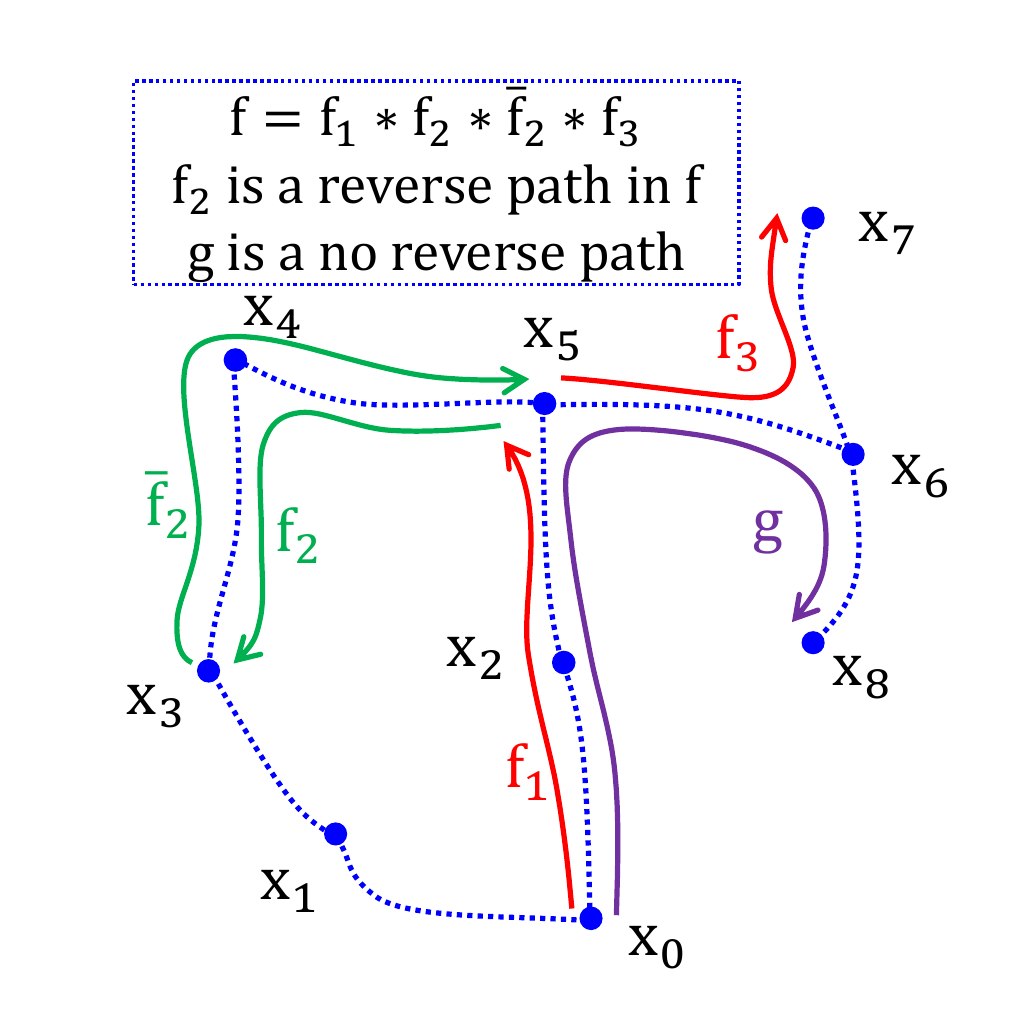}
    \label{fig_3_4}}
  \caption{Illustration of constructing a topology graph based on convex division. (a) The nodes and edges in the topology graph correspond to convex polygons and cutlines, respectively, and the paths in the topology graph are defined as sequences of adjacent nodes. (b) The purple path and the red path can build a path sequence that conforms to definition 6 so they are homotopy. (c) Product of Paths in $X_{topo}$, the green path is equivalent to concatenating the tail of the red path with the head of the purple path. (d) For path $f$ the green part is the rollback path it contains, and the purple path $g$ is no rollback path.}
  \label{fig_3}
\end{figure}
\begin{definition}[Path in $X_{topo}$]
  \label{defTPPath}
  The path $\mathrm{f}$ in $X_{topo}$ is a finite sequence. $\mathrm{f} \subset X_{topo}$, and $\forall t\in \mathbb{N} ^{T_\mathrm{f}}_2$, $\mathrm{f}(t)$, $\mathrm{f}(t-1)$ are adjacent. The $\mathrm{f}(t)$ represents the $t-\mathrm{th}$ element of the sequence $\mathrm{f}$, the $T_\mathrm{f}$ represents the length of the sequence $\mathrm{f}$. For convenience in notation, we use $\mathrm{f}(-t)$ to denote the $t-\mathrm{th}$ element from the end of the sequence $\mathrm{f}$, i.e., $\mathrm{f}(-t) = \mathrm{f}(T_\mathrm{f}-t+1)$. We denote the set of all paths in $X_{topo}$ as $P(X_{topo})$. All paths starting at $\mathrm{x}_0$ and ending at $\mathrm{x}_1$ are denoted as $P(X_{topo};\mathrm{x}_0,\mathrm{x}_1)$.
\end{definition}
\begin{definition}[Path Homotopy in $X_{topo}$]
  \label{defTPPathH}
  \ \newline \indent
  Let $\mathrm{f}, \mathrm{g}\in P(X_{topo};\mathrm{x}_0,\mathrm{x}_1)$. If there exists a finite sequence $(\mathrm{h_1, h_2,\cdots, h_{\mathit{n}}})\subset P(X_{topo};\mathrm{x}_0,\mathrm{x}_1)$, such that $\mathrm{h_1}=\mathrm{f}$, $\mathrm{h_{\mathit{n}}}=\mathrm{g}$, and for each sequence $\mathrm{h_{\mathit{k}}}$, one of the following two operations is applied on the basis of the previous sequence $\mathrm{h_{\mathit{k-1}}}$:
  \begin{enumerate}
    \item{Extension: An element $\mathrm{x}$ is replaced in the sequence $\mathrm{h_{\mathit{k-1}}}$ with $(\mathrm{x, x})$ or $(\mathrm{x, y, x})$, where $\mathrm{y}$ is a node in $X_{topo}$;}
    \item{Contraction: A fragment of the form $(\mathrm{x, x})$ or $(\mathrm{x, y, x})$ is replaced in the sequence $\mathrm{h_{\mathit{k-1}}}$ with $\mathrm{x}$.}
  \end{enumerate}
  This is path homotopy in $X_{topo}$, denoted as $\mathrm{f} \backsimeq _p \mathrm{g}$. The equivalence class of a path $\mathrm{f}$ under this relation is called the homotopy class of $\mathrm{f}$, often denoted by $[\mathrm{f}]$. An illustration of this definition is shown in $\mathrm{Fig.~3(b)}$.
\end{definition}
\begin{remark}
  \label{remakeMapEqu}
  Path homotopy in $X_{topo}$ is an equivalence relation.
\end{remark}
\begin{proof}
  \ \newline
  Reflexivity: For any path $\mathrm{f}$ in $X_{topo}$, there is a singleton sequence $(\mathrm{f})$ such that $\mathrm{f}$ is congruent with itself.
  \ \newline
  Symmetry: If there exists a path sequence such that $\mathrm{f} \simeq_p \mathrm{g}$, then the inverse of this path sequence can be such that $\mathrm{g} \simeq_p \mathrm{f}$.
  \ \newline
  Transitivity: If $\mathrm{f} \simeq_p \mathrm{g}$ and $\mathrm{g} \simeq_p \mathrm{h}$, then splicing the corresponding path sequences of both can obtain a sequence that makes $\mathrm{f} \simeq_p \mathrm{h}$.
\end{proof}
\begin{definition}[Path Time Homotopy in $X_{topo}$]
  \label{defTPPathTH}
  \ \newline \indent
  Let $\mathrm{f}, \mathrm{g}\in P(X_{topo};\mathrm{x}_0,\mathrm{x}_1)$. If there is a map $\varphi :\mathbb{N}_1 \to \mathbb{N}_1 $ and $\forall t\in \mathbb{N}_1$, $\varphi(t)\leqslant \varphi(t+1)$ such that $\forall t\in \mathbb{N} ^{T_\mathrm{f}}_1$, $\mathrm{f}(T)=\mathrm{g}(\varphi(t))$ or $\forall t\in \mathbb{N} ^{T_\mathrm{g}}_1$, $\mathrm{g}(T)=\mathrm{f}(\varphi(t))$, is said to be path time homotopic, denoted as $\mathrm{f}\cong _p \mathrm{g}$.
\end{definition}
\begin{definition}[Product of Paths in $X_{topo}$]
  \label{defProTPPath}
  \ \newline \indent
  Let $\mathrm{f}\in P(X_{topo};\mathrm{x}_0,\mathrm{x}_1)$ and $\mathrm{g}\in P(X_{topo};\mathrm{x}_1,\mathrm{x}_2)$. Define the product $\mathrm{f}*\mathrm{g}$ of $\mathrm{f}$ and $\mathrm{g}$ as the path $\mathrm{h}$,
  \begin{equation}
    \label{eq12}
    {\mathrm{h}(t)} =
    \begin{cases}
      \mathrm{f}(t),                           & t \in \mathbb{N}_{1}^{T_{\mathrm{f}}}                                 \\
      \mathrm{g}\left(t-T_{\mathrm{f}}\right), & t \in \mathbb{N}_{T_{\mathrm{f}}+1}^{T_{\mathrm{g}}+T_{\mathrm{f}}-1}
    \end{cases}
  \end{equation}
  The operation $*$ in $\left(P(X_{topo}), *\right)$ has the following properties:
  \begin{enumerate}
    \item{Associativity: If $\mathrm{f*g}$ and $\mathrm{g*h}$ are meaningful, then $\mathrm{(f * g) * h} \cong_p \mathrm{f *(g * h)}$.}
    \item{Left and Right identity element: Given $\mathrm{x}\in X_{topo}$, let $\mathrm{e_x}:1\to \mathrm{x}$. If $\mathrm{f}\in P(X_{topo};\mathrm{x_0,x_1})$, then
                \begin{equation}
                  \label{eq13}
                  \mathrm{e_{x_0} * f} \cong_p \mathrm{f} \text { and } \mathrm{f * e_{x_1}} \cong_p \mathrm{f}.
                \end{equation}
          }
    \item{Inverse element: Given $\mathrm{f}\in P(X;\mathrm{x_0,x_1})$, we define the inverse of $\mathrm{f}$ as $\mathrm{\bar{f}}(t)=\mathrm{f}(T_{\mathrm{f}}-t+1)$, then
                \begin{equation}
                  \label{eq14}
                  \mathrm{f * \bar{f}} \simeq_p \mathrm{e_{x_0}} \text{ and } \mathrm{\bar{f} * f} \simeq_p \mathrm{e_{x_1}}.
                \end{equation}
          }
  \end{enumerate}
\end{definition}
\begin{definition}[Rollback Path in $X_{topo}$]
  \label{defTPRP}
  If the path $\mathrm{f}\in P(X_{topo})$ can be written as $\mathrm{f} \cong_p \mathrm{f_1*f_2*\bar{f}_2*f_3}$, and $\mathrm{f_2}$ is not an identity element path, then $\mathrm{f_2} * \mathrm{\bar{f}_2}$ is a rollback path in $\mathrm{f}$. If $\mathrm{f}$ cannot be written in this form, we call $\mathrm{f}$ is no rollback path.
\end{definition}
\begin{theorem}
  \label{theorem1}
  For any path $\mathrm{f}$ in $P(X_{topo})$, there is only one no rollback path $\mathrm{g}\in P(X_{topo})$ such that $\mathrm{f} \simeq_p \mathrm{g}$.
\end{theorem}
\begin{proof}
  \ \newline
  Existence: According to \textbf{Definition~\ref{defTPPath}}, it is known that the path in $X_{topo}$ is a finite sequence. Therefore, the contraction operation according to \textbf{Definition~\ref{defTPPathH}} must be possible to shorten the path to the no rollback path.\ \newline
  Uniqueness: Assume that the path $\mathrm{f}$ in $X_{topo}$ has two different path homotopic no rollback paths given by $\mathrm{g}$, $\mathrm{h}$. However, the two operations according to \textbf{Definition~\ref{defTPPathH}} can only add and remove segments shaped like $(\mathrm{x,x})$ or $(\mathrm{x,y,x})$ from the path sequence. Hence, $\mathrm{g}$ and $\mathrm{h}$ cannot be homotopic by \textbf{Definition~\ref{defTPPathH}}. However, homotopy is an equivalence relation, and $\mathrm{f}$ cannot be homotopic to $\mathrm{g}$ and $\mathrm{h}$ simultaneously. Therefore, the uniqueness of the proposition is maintained.
\end{proof}
\begin{lemma}
  \label{theorem2}
  For any path $\mathrm{f}$, $\mathrm{g}$ in $P(X_{topo})$, the necessary and sufficient condition for $\mathrm{f} \simeq_p \mathrm{g}$ is that their no rollback paths are equal.
\end{lemma}
\begin{proof}
  Suppose $\mathrm{f}$, $\mathrm{g}$ are any two paths in $P(X_{topo})$, and the no rollback paths of $\mathrm{f}$, $\mathrm{g}$ are $\mathrm{f^*}$, $\mathrm{g^*}$, according to \textbf{Theorem~\ref{theorem1}}, $\mathrm{f}\simeq_p \mathrm{f^*}$ and $\mathrm{g}\simeq_p \mathrm{g^*}$.\ \newline
  Necessary: When $\mathrm{f} \simeq_p \mathrm{g}$ because path homotopy is an equivalence relation in $P(X_{topo})$, we can get $\mathrm{f}\simeq_p \mathrm{f^*} \simeq_p \mathrm{g}\simeq_p \mathrm{g^*}$, on the basis of the uniqueness of \textbf{Theorem~\ref{theorem1}}, such that $\mathrm{f^*}$ and $\mathrm{g^*}$ are the same path.\ \newline
  Sufficient: If $\mathrm{f}$, $\mathrm{g}$ have the same no rollback paths ($\mathrm{f^*} = \mathrm{g^*}$), then $\mathrm{f}\simeq_p  \mathrm{f^*} = \mathrm{g^*}\simeq_p \mathrm{g}$, such that $\mathrm{f} \simeq_p \mathrm{g}$.
\end{proof}
Further, we give another formulation of \textbf{Lemma~\ref{theorem2}} as \textbf{Corollary~\ref{theorem3}}.
\begin{corollary}
  \label{theorem3}
  For any homotopic class $[\mathrm{f}]$ in $P(X_{topo})$ sharing a no rollback path.
\end{corollary}
Next, we construct the relationship between the path space $P(X_{free})$ and $P(X_{topo})$ by defining two mappings $\Gamma $, $\Gamma^g $.
\begin{definition}[Mapping $\Gamma :P(X_{free})\to P(X_{topo})$]
  \label{defMAPT}
  \ \newline \indent
  For $f\in P(X_{free})$, $\Gamma \circ f$ is the sequence of nodes in $P(X_{topo})$ corresponding to the convex polygons passed by $f(t)$ with $t$ in the order from $0$ to $1$.
\end{definition}
\begin{definition}[Mapping $\Gamma^g :P(X_{topo})\to P(X_{free})$]
  \label{defMAPTg}
  \ \newline \indent
  For $\mathrm{f} \in P(X_{topo})$ and $\mathrm{f}=(\mathrm{x_1,x_2,\cdots ,x_n})$, $\Gamma^g \circ \mathrm{f}=f_1*f_2*\cdots *f_{n-1}$, where $f_k$, $k\in \mathbb{N} ^{n-1}_1$ is $\Gamma^g \circ (\mathrm{x}_k,\mathrm{x}_{k+1})$ and $f_k=l^{x_c}_{x_k}*l^{x_{k+1}}_{x_c}$, where $x_k$, $x_{k+1}$ are the centroid of the convex polygon corresponding to $\mathrm{x}_k$, $\mathrm{x}_{k+1}$. $x_c$ is the midpoint of the cutline between $\mathrm{x}_k$ and $\mathrm{x}_{k+1}$.
\end{definition}
\begin{figure}[!t]
  \centering
  \subfloat[]{\includegraphics[width=2.8in]{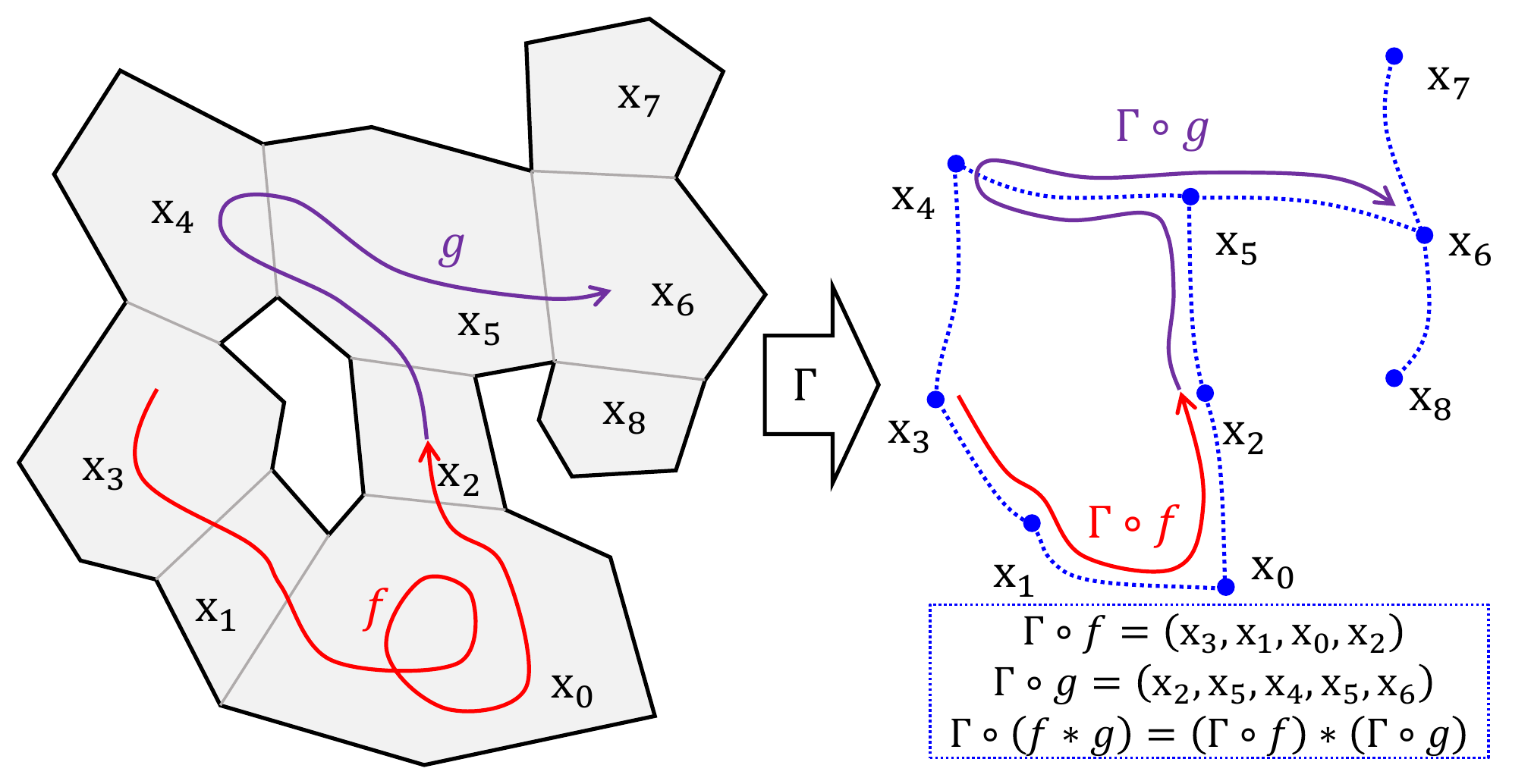}
    \label{fig_4_1}}
  \\
  \subfloat[]{\includegraphics[width=2.8in]{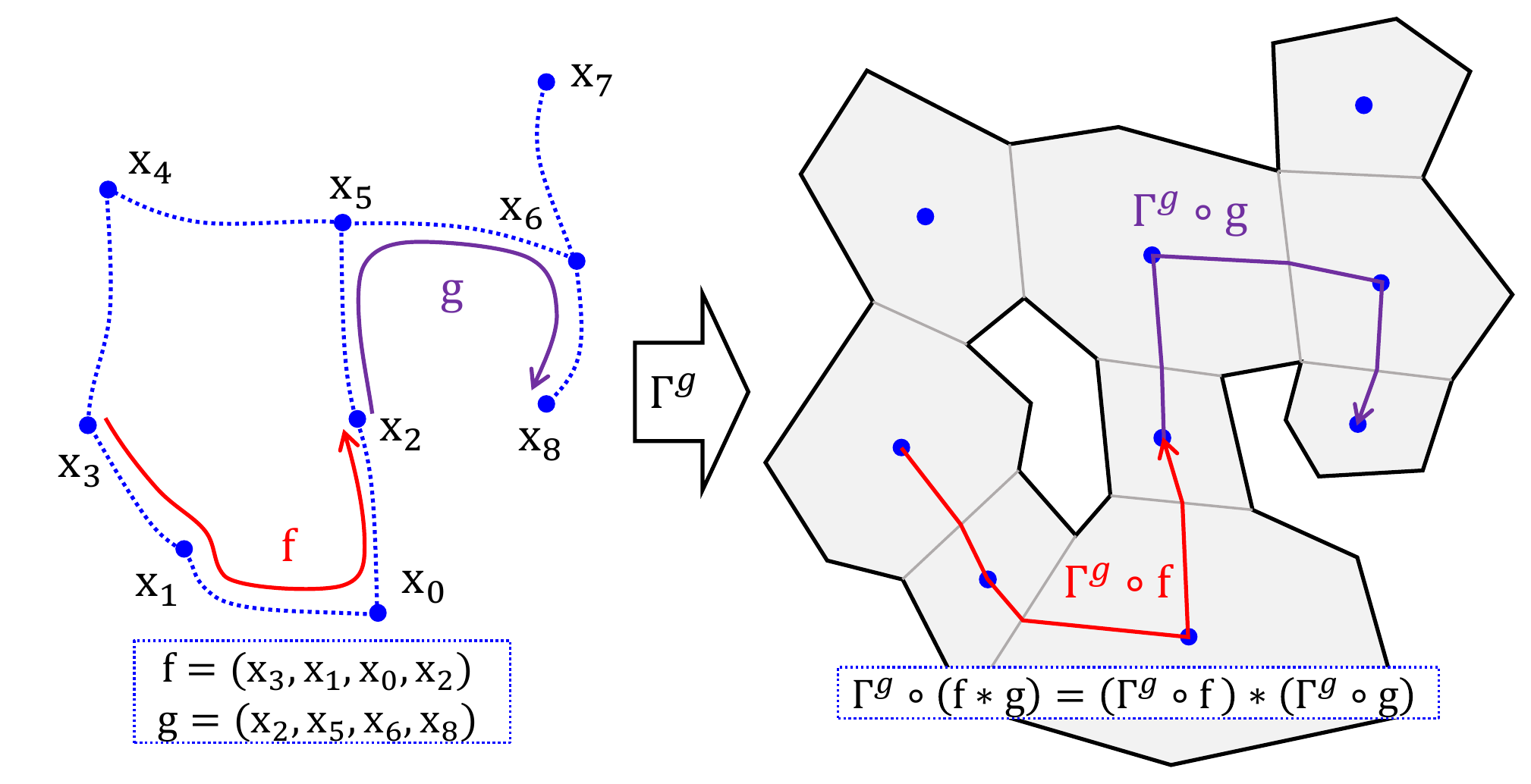}
    \label{fig_4_2}}
  \\
  \subfloat[]{\includegraphics[width=2.8in]{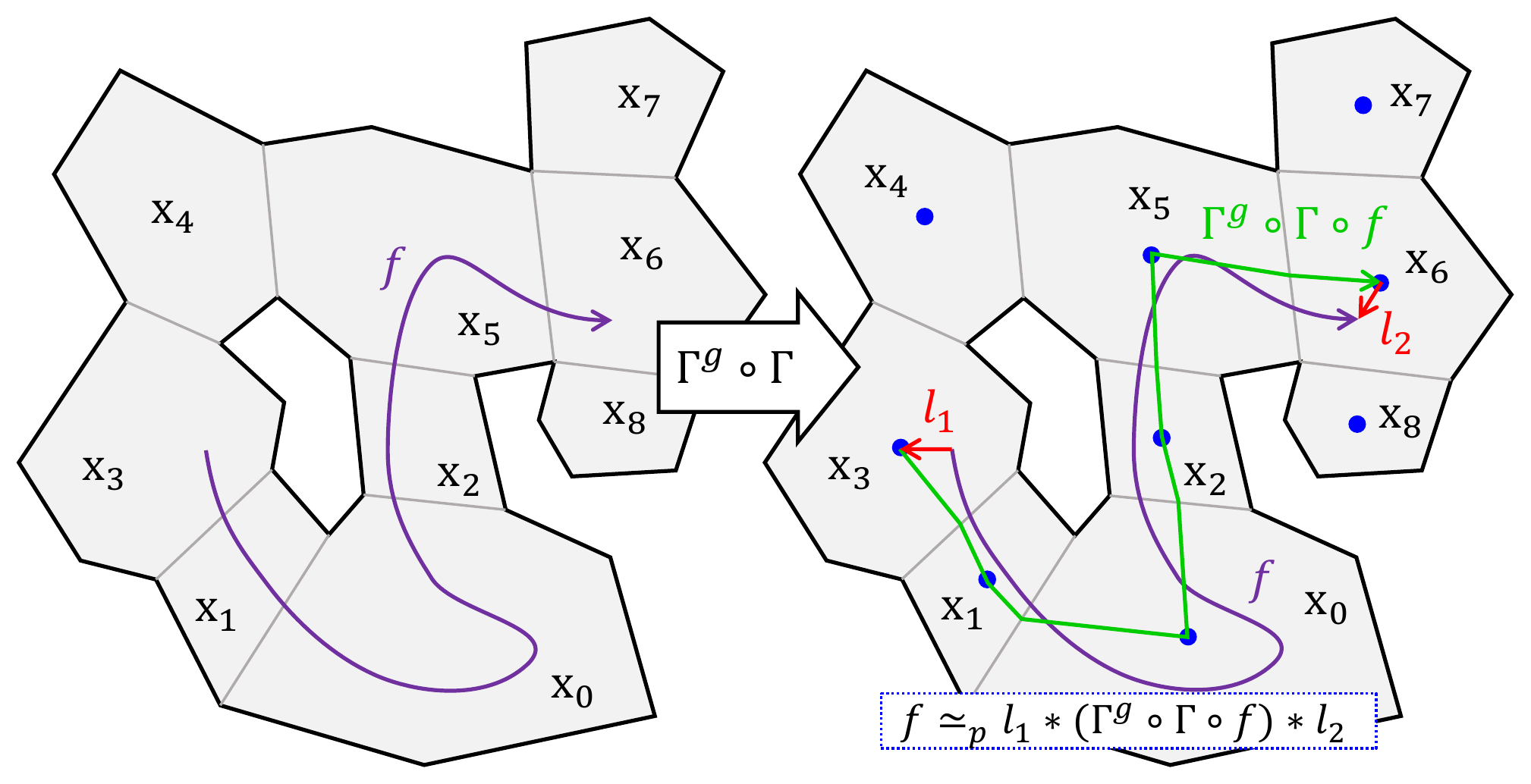}
    \label{fig_4_3}}
  \caption{ Illustration of mapping $\Gamma $ and $\Gamma ^g$. (a) The purple path $f$ in $P(X_{free})$ maps to $P(X_{topo})$ as a sequence of convex polygons that $f$ traverses in turn, obviously $\Gamma $ is a homomorphic mapping of $P(X_{free})$  to $P(X_{topo})$. (b) Shows how $\Gamma ^g$ is mapped, $\Gamma ^g$ is a homomorphic mapping from $P(X_{topo})$ to $P(X_{free})$. (c) The path  $f$ can use $\Gamma ^g \circ \Gamma$ to quickly find a homotopy path $l_1*(\Gamma ^g \circ \Gamma \circ f)*l_2$.}
  \label{fig_4}
\end{figure}
\begin{remark}
  \label{remarkMapT}
  $\Gamma$ is evidently a homomorphism of $P(X_{free})$ to $P(X_{topo})$, because $\forall f,g\in P(X_{free})$ and $f*g$ exists; hence, we obtain
  \begin{equation}
    \label{eq15}
    \Gamma \circ (f*g) \cong_p (\Gamma \circ f)*(\Gamma \circ g).
  \end{equation}
  The mapping $\Gamma^g$ also satisfies the following relationship:\ \newline
  $\forall \mathrm{f,g}\in P(X_{topo})$ and $\mathrm{f*g}$ exists; hence, we obtain
  \begin{equation}
    \label{eq16}
    \Gamma^g \circ (\mathrm{f*g}) \cong_p (\Gamma^g \circ \mathrm{f})*(\Gamma^g \circ \mathrm{g}).
  \end{equation}
  As shown in $\mathrm{Fig.~\ref{fig_4}(c)}$, we can quickly find a path that is homotopic to the original path by using the mapping $\Gamma^g \circ \Gamma$:
  \begin{equation}
    \label{eq17}
    f\cong_p l_1*(\Gamma^g \circ \Gamma \circ f)*l_2.
  \end{equation}
\end{remark}
\begin{definition}[Mapping $\Gamma^* :P(X_{free})\to P(X_{topo})$]
  \label{defMAPTstar}
  \ \newline \indent
  For $f \in P(X_{free})$, $\Gamma^* \circ f$ is the corresponding no rollback path of $\Gamma \circ f$. And in the subsequent text, for $\mathrm{f} \in P(X_{topo})$, we can use $\Gamma^* \circ \Gamma^g \circ \mathrm{f}$ to denote the no rollback path homotopic to $\mathrm{f}$.
\end{definition}
\begin{lemma}
  \label{theorem4}
  If the sequence $\mathrm{f}_n \in P(X_{topo})$ conforms to the condition given in \textbf{Definition~\ref{defTPPathH}}, then all the elements in the sequence  $(\Gamma^g \circ \mathrm{f}_n)$ are homotopic.
\end{lemma}
\begin{proof}
  For this proposition, we are required to prove that for the path $\mathrm{f}'$ obtained by any path $\mathrm{f} \in P(X_{topo})$ through the transformations of \textbf{Definition~\ref{defTPPathH}}, $\Gamma^g \circ \mathrm{f} \simeq_p \Gamma^g \circ\mathrm{f}'$ is satisfied.
  \ \newline \indent
  For the extension transformation in \textbf{Definition~\ref{defTPPathH}}, $\mathrm{f}$ and $\mathrm{f}'$ are expressed respectively as follows:
  \begin{equation}
    \label{eq18}
    \mathrm{f} \cong_p \mathrm{f}_s * \mathrm{f}_e \text{ and }
    \mathrm{f}' \cong_p \mathrm{f}_s * \mathrm{f}'_{new} * \overline{\mathrm{f}'_{new}} * \mathrm{f}_e.
  \end{equation}
  Here, $\mathrm{f}'_{new} * \overline{\mathrm{f}'_{new}}$ is the newly inserted path segment of $\mathrm{f}'$ relative to $\mathrm{f}$. $\mathrm{f}_s$ and $\mathrm{f}_e$ are the first and second half of the path, respectively, of $\mathrm{f}$ at the insertion position of $\mathrm{f}'_{new} * \overline{\mathrm{f}'_{new}}$. Furthermore, the following equations can be obtained using (\ref{eq16}):
  \begin{equation}
    \label{eq19}
    \Gamma^g \circ \mathrm{f} \cong_p (\Gamma^g \circ \mathrm{f}_s) * (\Gamma^g \circ \mathrm{f}_e) ,
  \end{equation}
  \begin{equation}
    \label{eq20}
    \Gamma^g \circ \mathrm{f}' \cong_p (\Gamma^g \circ \mathrm{f}_s) * (\Gamma^g \circ \mathrm{f}_{new}) * (\Gamma^g \circ \overline{\mathrm{f}'_{new}}) * (\Gamma^g \circ \mathrm{f}_e).
  \end{equation}
  According to the definition of $\Gamma^g$, equation (\ref{eq20}) can be expressed as
  \begin{equation}
    \label{eq21}
    \Gamma^g \circ \mathrm{f}' \cong_p (\Gamma^g \circ \mathrm{f}_s) * (\Gamma^g \circ \mathrm{f}_{new}) * \overline{(\Gamma^g \circ \mathrm{f}'_{new})} * (\Gamma^g \circ \mathrm{f}_e).
  \end{equation}
  Hence, it can be concluded that:
  \begin{equation}
    \label{eq22}
    \Gamma^g \circ \mathrm{f}' \simeq_p (\Gamma^g \circ \mathrm{f}_s) * (\Gamma^g \circ \mathrm{f}_e)  \simeq_p \Gamma^g \circ \mathrm{f}.
  \end{equation}
  The proof for the contraction operation is similar.
\end{proof}
\begin{figure}[!t]
  \centering
  \subfloat[]{\includegraphics[width=1.5in]{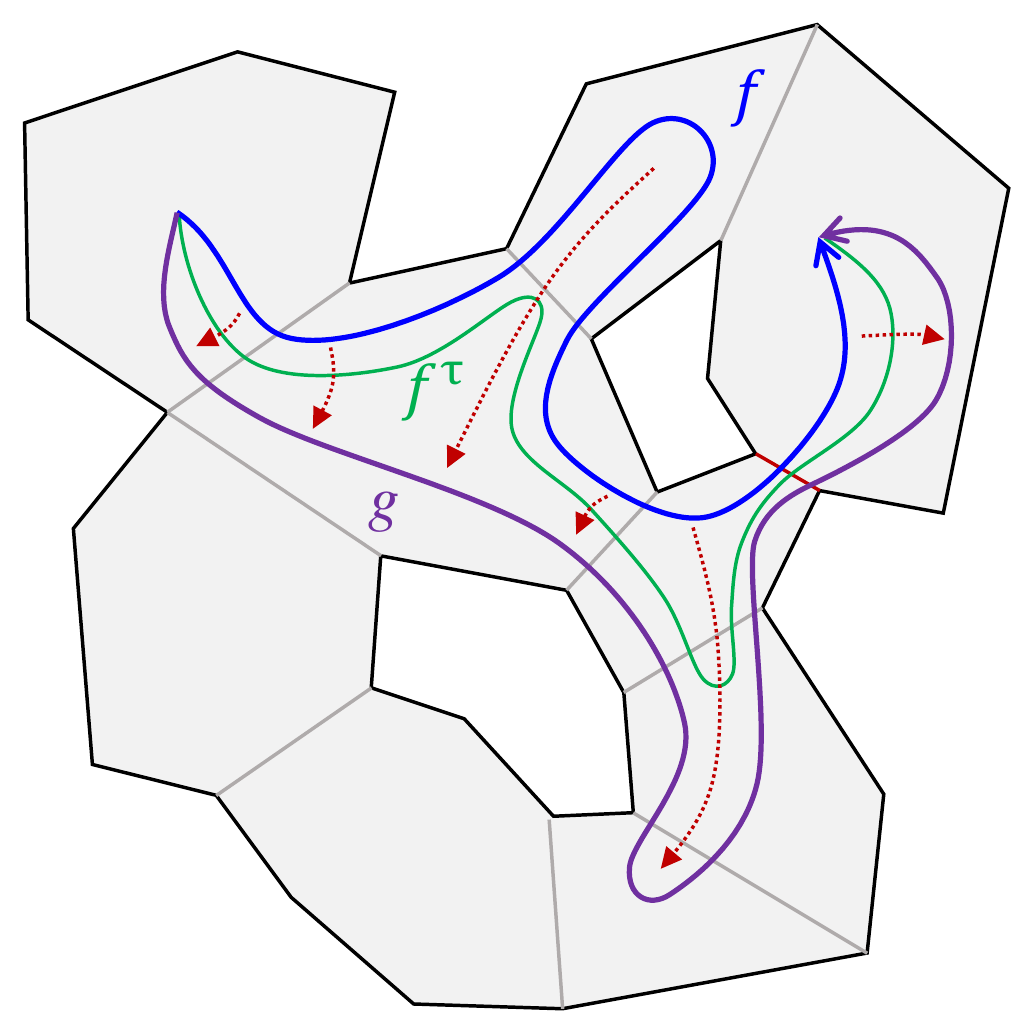}
    \label{fig_5_1}}
  \hfil
  \subfloat[]{\includegraphics[width=1.5in]{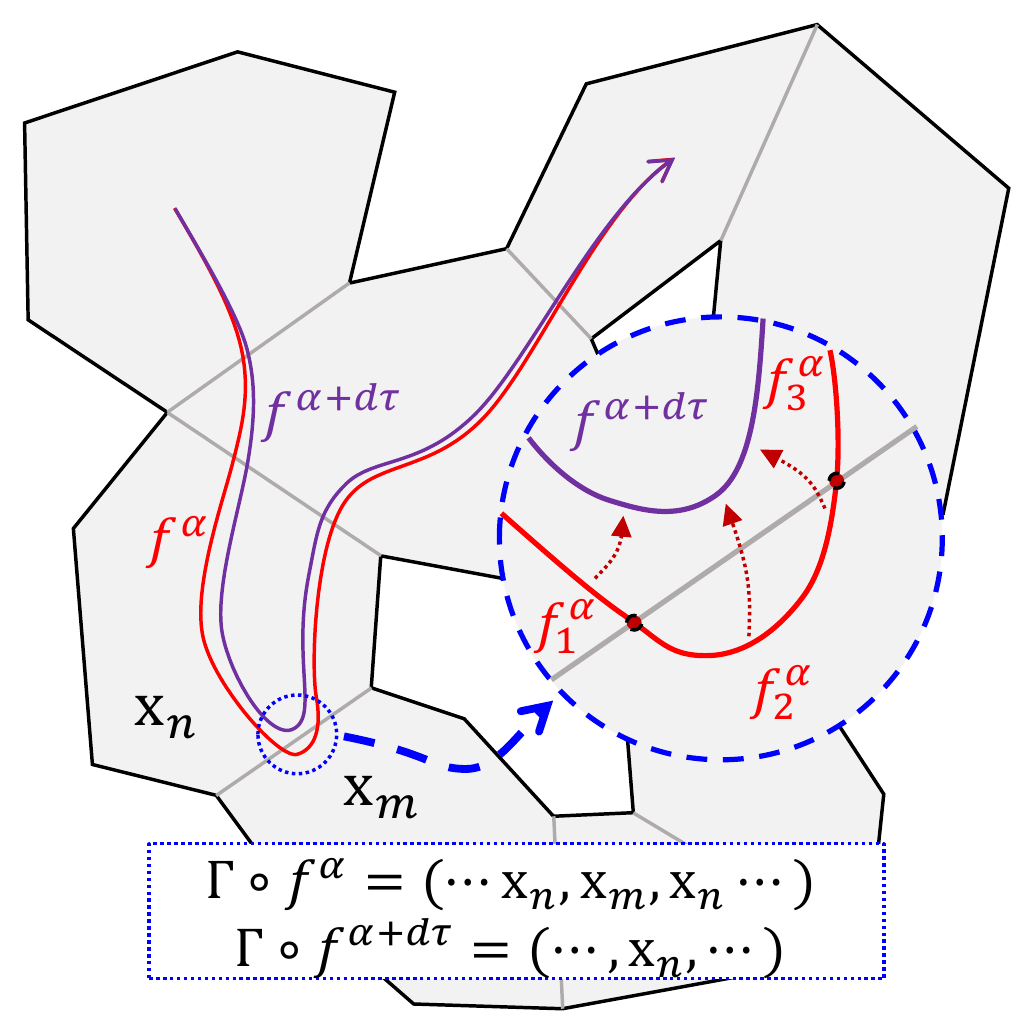}
    \label{fig_5_2}}
  \caption{Illustration of homotopy paths and their passage through convex polygons. (a)  In the continuous change of path $f$ to path $g$, it continuously enters new convex polygons and exits old convex polygons. (b) The mapping corresponding to $\Gamma$ before and after each time the path enters or exits the convex polygon conforms to Definition~\ref{defTPPathH}.}
  \label{fig_5}
\end{figure}
\begin{lemma}
  \label{theorem5}
  Given that $f,g\in P(X_{free};x_0,x_1)$, a necessary and sufficient condition for $f \simeq_p g$ is $\Gamma \circ f \simeq_p \Gamma \circ g$.
\end{lemma}
\begin{proof}
  \ \newline
  Necessary: For $f \simeq_p g$, by \textbf{Definition~\ref{defPathH}}, a continuous function $F:I\times I\to X$ can be obtained, such that $f$ varies continuously to $g$. As shown in Fig.~\ref{fig_5}, during this continuous change, the path $f^{\tau}$ will gradually enter the convex polygons through which $g$ passes, where $f^{\tau}(t)=F(\tau,t)$. Let $\alpha \in I$ be any time for $f^{\tau}$ enters a new convex polygon, we can write $f^{\alpha}$ as
  \begin{equation}
    \label{eq23}
    f^{\alpha} \cong_p f^{\alpha}_1 * f^{\alpha}_2 * f^{\alpha}_3,
  \end{equation}
  where $f^{\alpha}_1$ is the part of $f^{\alpha}$ from $f^{\alpha}(0)$ to the boundary of the new convex polygon, $f^{\alpha}_3$ is the part of $f^{\alpha}$ from the boundary of the new convex polygon to $f^{\alpha}(1)$, and $f^{\alpha}_2$ is the section of $f^{\alpha}$ in the new convex polygon, according to characteristics 2) and 3) of the division method mentioned in the previous subsection. The endpoint of each cutline is the boundary of $X_{free}$ and therefore, $f^{\alpha}_2(0)$ and $f^{\alpha}_2(1)$ are on the same cutline. Hence, the changes of these two process paths are mapped into $P(X_{topo})$ using $\Gamma$, which must be equal to the process described in \textbf{Definition~\ref{defTPPathH}}.
  \ \newline \indent
  Therefore, with the change from $f$ to $g$, we can define a sequence that conforms to \textbf{Definition~\ref{defTPPathH}}, such that $\Gamma \circ f \simeq_p \Gamma \circ g$.
  \ \newline
  Sufficient: For $\Gamma \circ f \simeq_p \Gamma \circ g$, we can determine a sequence of paths that fits the description of \textbf{Definition~\ref{defTPPathH}}. The first element of the sequence is $\Gamma \circ f$ and the last element is $\Gamma \circ g$. Hence, according to \textbf{Lemma~\ref{theorem4}}, we obtain
  \begin{equation}
    \label{eq24}
    \Gamma^g \circ \Gamma \circ f \simeq_p \Gamma^g \circ  \Gamma \circ g.
  \end{equation}
  According to (\ref{eq17}) we obtain the homotopic path of $f$ and $g$ as follows:
  \begin{equation}
    \label{eq25}
    f' \cong_p l_s*(\Gamma^g \circ \Gamma \circ f)*l_e \simeq_p f,
  \end{equation}
  \begin{equation}
    \label{eq26}
    g' \cong_p l_s*(\Gamma^g \circ \Gamma \circ g)*l_e \simeq_p g,
  \end{equation}
  where $l_s$ is a line connecting point $x_0$ and the starting point of $\Gamma \circ f$, $l_e$ is a line connecting the ending point of $\Gamma \circ f$ and point $x_1$. Using (\ref{eq24}), (\ref{eq25}) and (\ref{eq26}) we can obtain $f' \simeq_p g'$; hence, $f \simeq_p g$.
\end{proof}
\begin{theorem}
  \label{theorem6}
  For any homotopy path class $[f]$ in $P(X_{free})$, $\Gamma^* \circ [f]$ has a unique result.
\end{theorem}
\begin{proof}
  According to \textbf{Lemma~\ref{theorem5}}, the homotopy path in $P(X_{free})$ remains homotopic when mapped to $P(X_{topo})$, i.e., $\Gamma^* \circ [f]=[\Gamma^* \circ f]$. Furthermore, according to \textbf{Corollary~\ref{theorem3}}, we can obtain that for every homotopy class in $P(X_{topo})$ there is one no rollback path that is homotopic to $\Gamma \circ [f]$. In other words, $\Gamma^* \circ [f]$ has a unique result.
\end{proof}
According to the above theorems, we obtained an encoder $\Gamma^*$ that can efficiently encode homotopy classes in $P(X_{free})$. For any paths with the same starting point and ending point in $P(X_{free})$, we can determine if they belong to the same homotopy class by determining if their mappings using $\Gamma^*$ are equal. In addition, $\Gamma^g$ can be used as a decoder for $\Gamma^*$.
\begin{remark}
  \label{remarkDecouple}
  Using this encoder and decoder, we can moreover decouple the optimal path planning process into the following two tasks:
  (i) Search for the homotopy path class that may contain the optimal path, and (ii) Search for the shortest homotopic path in the homotopy path class.
\end{remark}


\begin{theorem}
  \label{theorem7}
  For any $f \in P(X_{free})$, its locally optimal path $f^*$ satisfies $\Gamma \circ f^* = \Gamma^* \circ f$.
\end{theorem}
\begin{proof}
  For this proposition, it is only necessary to prove that $\Gamma^* \circ f^* = \Gamma \circ f^*$, i.e., $\Gamma \circ f^*$ is a no rollback path. Assume there exists the shortest path $f^*$ such that $\Gamma \circ f^*$ contains a rollback path. By the definition of $\Gamma$, there exists a situation where $f^*$ enters a convex polygon and then returns to the original convex polygon. At this point using a straight line connecting two points of $f^*$ in the original convex polygon, it is simple to obtain a homotopy path that is shorter than $f^*$. This contradicts the fact that $f^*$ is the shortest of $[f]$, therefore the original proposition holds.
\end{proof}
\begin{theorem}
  \label{theorem8}
  For $\forall f,g \in P(X_{free};x_0,x_1)$, if $\Gamma^* \circ f$, $\Gamma^* \circ g$ can be written as follows:
  $\Gamma^* \circ f \cong_p \mathrm{h_1*h_2*h_3}$ and $\Gamma^* \circ g \cong_p \mathrm{h_1*h_3}$.
  Then $S(f^*) \geqslant S(g^*)$.
\end{theorem}
\begin{proof}
  When $\mathrm{h_2}$ is an identity element,  $\Gamma^* \circ f = \Gamma^* \circ g$, hence $f^* \cong_p g^*$. Therefore $S(f^*) = S(g^*)$.

  When $\mathrm{h_2}$ is not an identity element, According to \textbf{Theorem~\ref{theorem7}} we obtain
  \begin{equation}
    \label{eq27}
    \Gamma \circ f^*=\Gamma^* \circ f \cong_p \mathrm{h_1*h_2*h_3}.
  \end{equation}
  According to \textbf{Remark~\ref{remarkMapT}}, $f^*$ can be decomposed into three parts: $f^* \cong_p f^*_1 * f^*_2 * f^*_3$, and they satisfy
  \begin{equation}
    \label{eq28}
    \Gamma \circ f^*_1 \cong_p \mathrm{h_1},\
    \Gamma \circ f^*_2 \cong_p \mathrm{h_2}\ \text{and} \
    \Gamma \circ f^*_3 \cong_p \mathrm{h_3}.
  \end{equation}
  Since $h_1*h_2*h_3$ and $h_1*h_3$ is meaningful, we have $ h_2(1)=h_1(-1)=h_3(1)=h_2(-1)$. Thus, both $f^*_2(0)$ and $f^*_2(1)$ are in the convex polygon $\mathrm{h_2}(1)$. It is easy to obtain the globally shortest path $l^{f^*_2(1)}_{f^*_2(0)}$ between $f^*_2(0)$ and $f^*_2(1)$. Therefore, it is simple to construct a path
  \begin{equation}
    \label{eq29}
    f' \cong_p f^*_1 * l^{f^*_2(1)}_{f^*_2(0)} * f^*_3,
  \end{equation}
  and
  \begin{equation}
    \label{eq30}
    \begin{aligned}
      S(f') & = S(f^*_1 * l^{f^*_2(1)}_{f^*_2(0)} * f^*_3)       \\
            & = S(f^*_1) + S(l^{f^*_2(1)}_{f^*_2(0)}) + S(f^*_3) \\
            & < S(f^*_1) + S(f^*_2) + S(f^*_3)                   \\
            & = S(f^*_1 * f^*_2 * f^*_3)                         \\
            & = S(f^*).
    \end{aligned}
  \end{equation}
  As $\Gamma \circ l^{f^*_2(1)}_{f^*_2(0)} = e_{h_2(1)}$, it follows that
  \begin{equation}
    \label{eq31}
    \Gamma \circ f' \cong_p h_1*e_{h_2(1)}*h_3 \cong_p \Gamma^* \circ g \cong_p \Gamma \circ g^*.
  \end{equation}
  Therefore, according to \textbf{Lemma~\ref{theorem5}}, it follows that $f' \simeq_p g^*$, and $S(f^*) > S(f') \geqslant S(g^*)$.
\end{proof}
According to \textbf{Theorem~\ref{theorem7}} and \textbf{Theorem~\ref{theorem8}}, we can easily obtain an interesting corollary.
\begin{corollary}
  \label{theorem9}
  For any $f\in P(X_{free})$, if the sequence $\Gamma \circ f$ contains duplicate elements, then $f$ must not be the global optimal path between $f(0)$ and $f(1)$.
\end{corollary}

\begin{lemma}
  \label{NSCLOP} 
  A necessary and sufficient condition for $f^* \in P(X_{free})$ to be the shortest path in $[f^*]$ is that, for any $t_1,t_2 \in [0,1]$, there exists the shortest path $g(t)=f^*(t_1+(t_2-t_1)t)$ in $[g]$.
\end{lemma}
\begin{proof}

  Necessary: The law of proof by contradiction is used herein. Suppose $f^*$ is the shortest path in $[f^*]$, and $\exists t_1, t_2 \in [0,1]$ such that $g(t)=f^*(t_1+(t_2-t_1)t)$ is not the shortest path in $[g]$; that is, there exists $g' \simeq_p g$ such that $S(g')<S(g)$. The path $f^*$ can be expressed as $f^* \cong_p g_s*g*g_e$, where
  \begin{align}
    \label{eq_NSCLOP1}
     & g_s(t) = f^*(a \cdot t),   & a = \min (t_1,t_2), \\
     & g_e(t) = f^*(b + (1-b) t), & b = \max (t_1,t_2).
  \end{align}
  Hence,
  \begin{equation}
    \label{eq_NSCLOP2}
    \begin{split}
      S(f^*) &= S(g_s) + S(g) + S(g_e)\\
      &< S(g_s) + S(g') + S(g_e)\\
      &< S(g_s * g' * g_e).
    \end{split}
  \end{equation}
  Therefore, there exists a path $g_s * g' * g_e$ in $[f^*]$ that is shorter than $f^*$. This contradicts our original hypothesis. Therefore, the sufficiency condition is true.

  Sufficient: When the latter of the proposition holds, let $t_1=0$ and $t_2=1$. At this moment $g(t)=f^*(t)$. Therefore, $f^*$ is the shortest path in $[f^*]$. Therefore, the necessary condition is true.
\end{proof}
\begin{theorem}
  \label{SSDR} 
  For any optimal path $f^* \in P(X_{free})$ where $x_s=f^*(0)$, $x_e=f^*(1)$. The $f^*$ must be expressed in the form of multiple line splices as follows.
  When $x_s$ and $x_e$ are within the same convex polygon,
  \begin{equation}
    \label{eq_SSDR1}
    f^* \cong_p l^{x_e}_{x_s}.
  \end{equation}
  When $x_s$ and $x_e$ are not within the same convex polygon,
  \begin{equation}
    \label{eq_SSDR2}
    f^* \cong_p l_{x_0}^{x_1} * l_{x_1}^{x_2} * \cdots * l_{x_{n-1}}^{x_n},
  \end{equation}
  where $x_0=x_s$, $x_n=x_e$, $x_1,\dotsc,x_{n-1}$ is the intersection of $f^*$ in turn with the cutlines.
\end{theorem}
\begin{proof}
  According to the properties of convex polygons, (\ref{eq_SSDR1}) is established, and according to the characteristics of the convex dissection method in Subsection III.A, the points $x_0,x_1,\cdots ,x_n$ in (\ref{eq_SSDR2}) exist. Therefore, $f^*$ can be expressed as follows:
  \begin{equation}
    \label{eq_SSDR3}
    f^* \cong_p f_1 * f_2 * \cdots * f_n,
  \end{equation}
  where $f_k \in P(X_{free},x_{k-1},x_k)$. Moreover, based on \textbf{Lemma~\ref{NSCLOP}}, the optimal path $f^*$ must guarantee that each of its segments is optimal; hence, for any $f_k$, $f_k = l^{x_k}_{x_{k-1}}$.
\end{proof}

According to \textbf{Theorem~\ref{SSDR}}, we only need to consider the points on the cutlines. The state space of the optimal path planning task is reduced from the 2D space of $X_{free}$ to a 1D space on the cutlines.

\section{Path Planning Algorithm}
In this section, we present an optimal path planning algorithm CDT-RRT* (Rapidly-exploring Random Tree based on Convex Division Topology), as an example of the application of homotopy path class encoder in path planning.

CDT-RRT* path planning process is based on the following four main ideas:
\begin{enumerate}
  \item{According to \textbf{Theorem~\ref{SSDR}}, CDT-RRT* only samples points on the cutlines, reducing the sampling space from 2D to 1D.}
  \item{After $x_{init}$ and $x_{goal}$ are determined, CDT-RRT* prunes the topological graph according to \textbf{Corollary~\ref{theorem9}}, further reducing the size of the state space that needs to be sampled.}
  \item{CDT-RRT* adopts a non-uniform sampling strategy that is biased towards exploring unknown homotopy classes, which accelerates the convergence of the algorithm towards the global optimal path.}
  \item{The sampling process of CDT-RRT* is essentially a search for path homotopy classes. When searching for a new homotopy class, it uses dynamic programming to calculate the locally optimal path in that homotopy class and attempts to update the globally optimal path.\footnote{According to \textbf{Remark~\ref{remarkDecouple}}, the optimal path planning task is decoupled into two sub-tasks: (i) and (ii). The RRT* method in CDT-RRT* is mainly used to handle task (i), while task (ii) is handled by \textbf{Algorithm~\ref{alg5}}.}}
\end{enumerate}

\begin{figure*}[!t]
  \centering
  \subfloat[]{\includegraphics[width=2.0in]{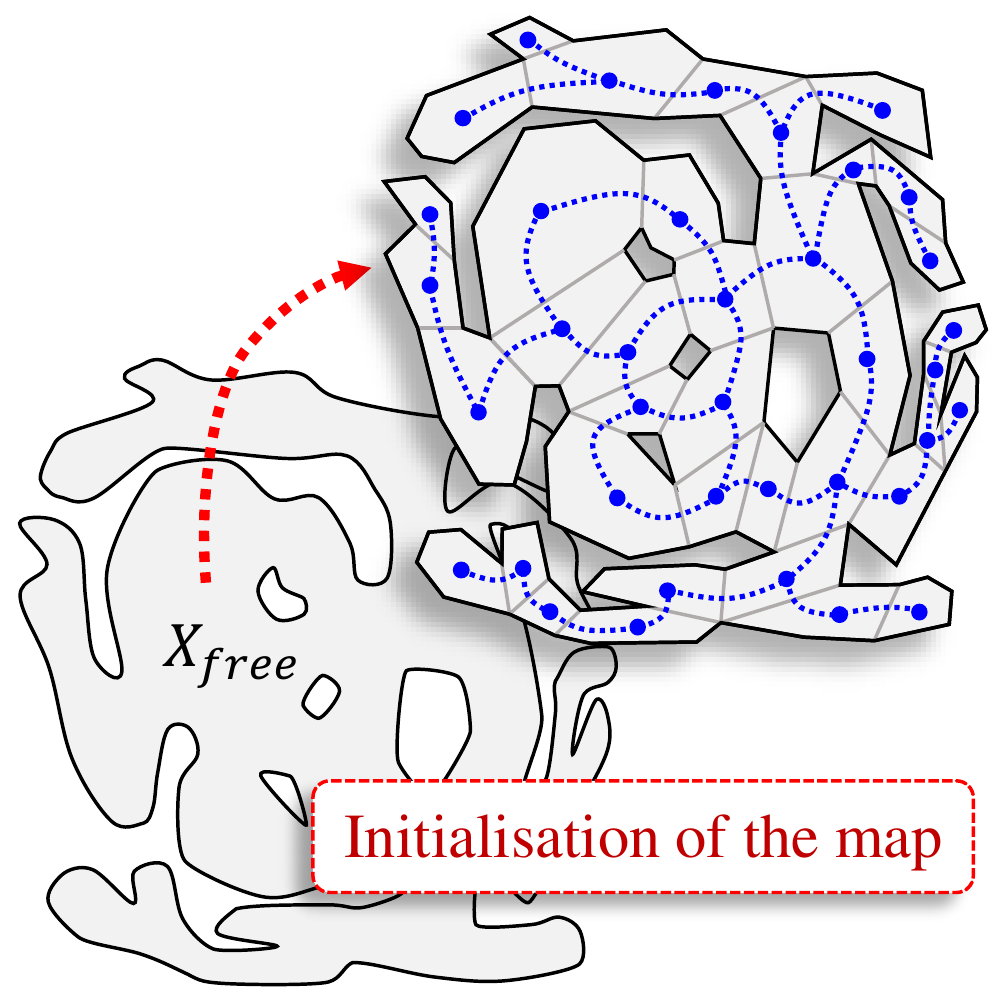}
    \label{fig_CDTRRT_1}}
  \hfil
  \subfloat[]{\includegraphics[width=2.0in]{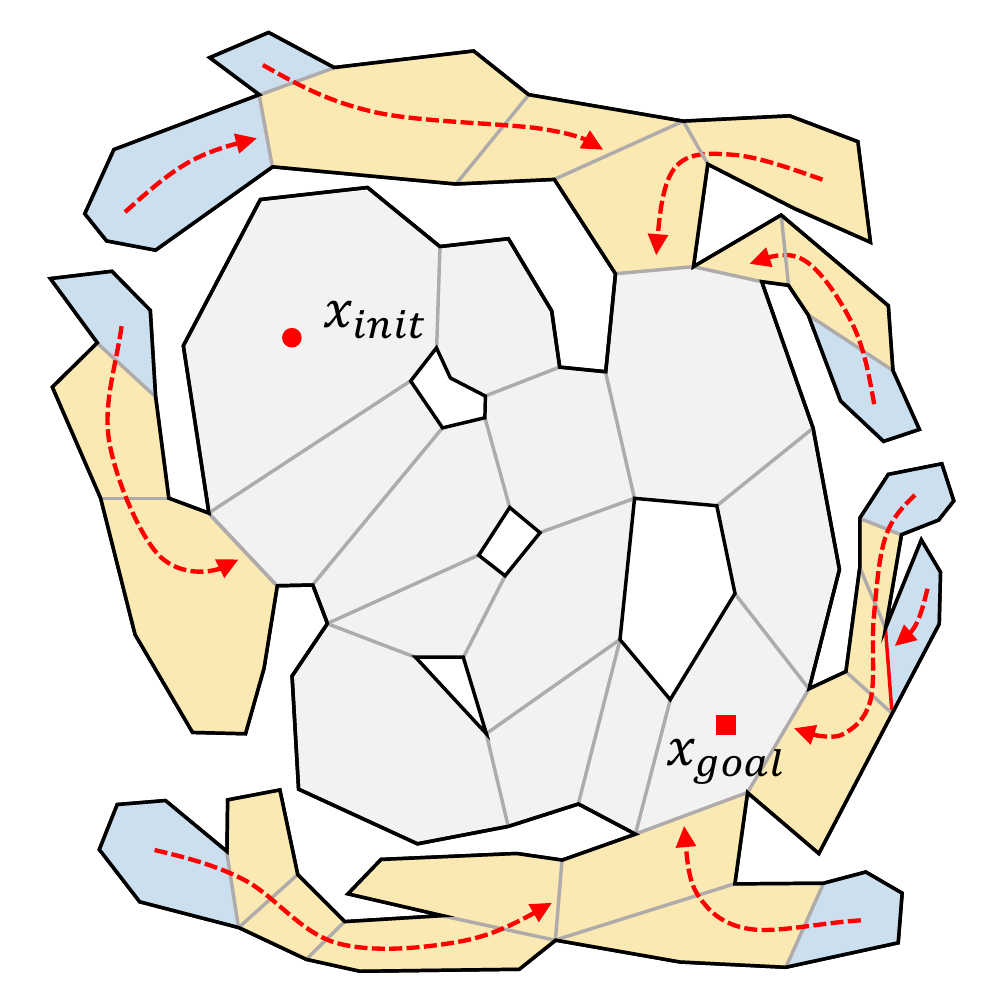}
    \label{fig_CDTRRT_2}}
  \hfil
  \subfloat[]{\includegraphics[width=2.0in]{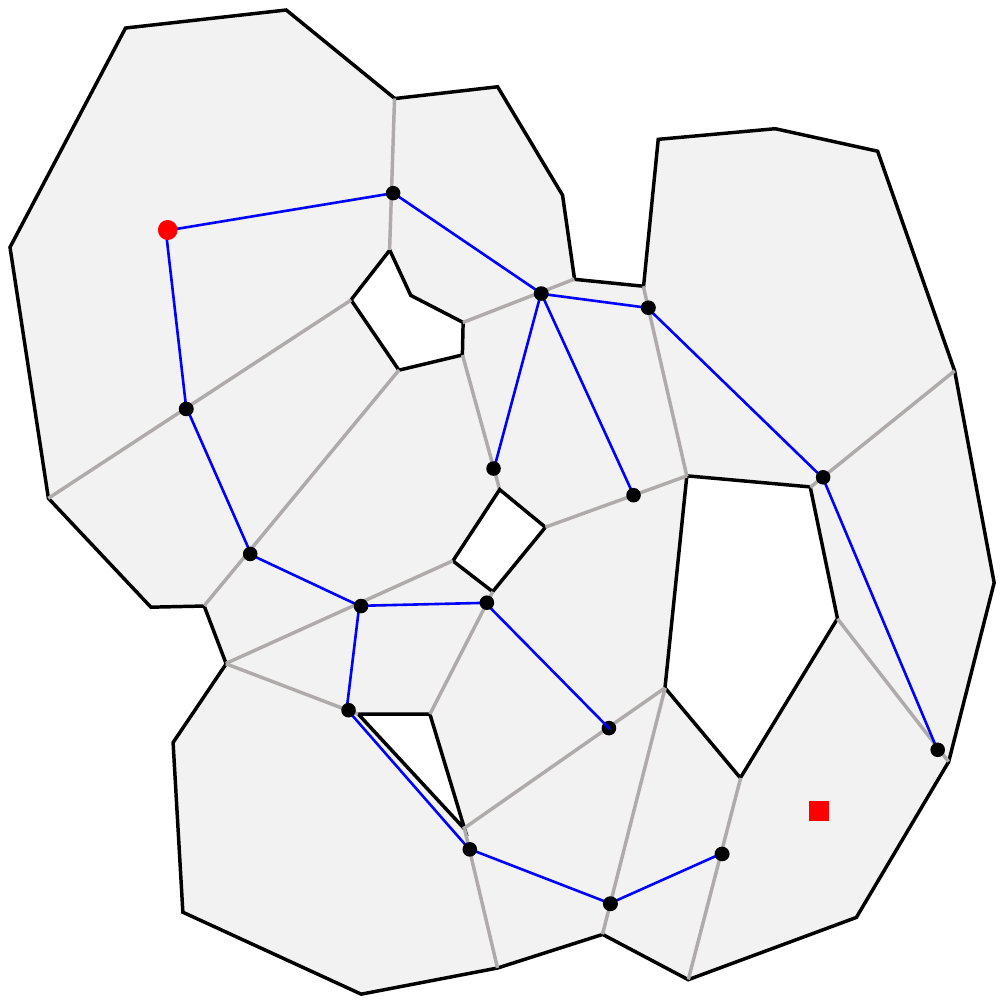}
    \label{fig_CDTRRT_3}}
  \hfil
  \subfloat[]{\includegraphics[width=2.0in]{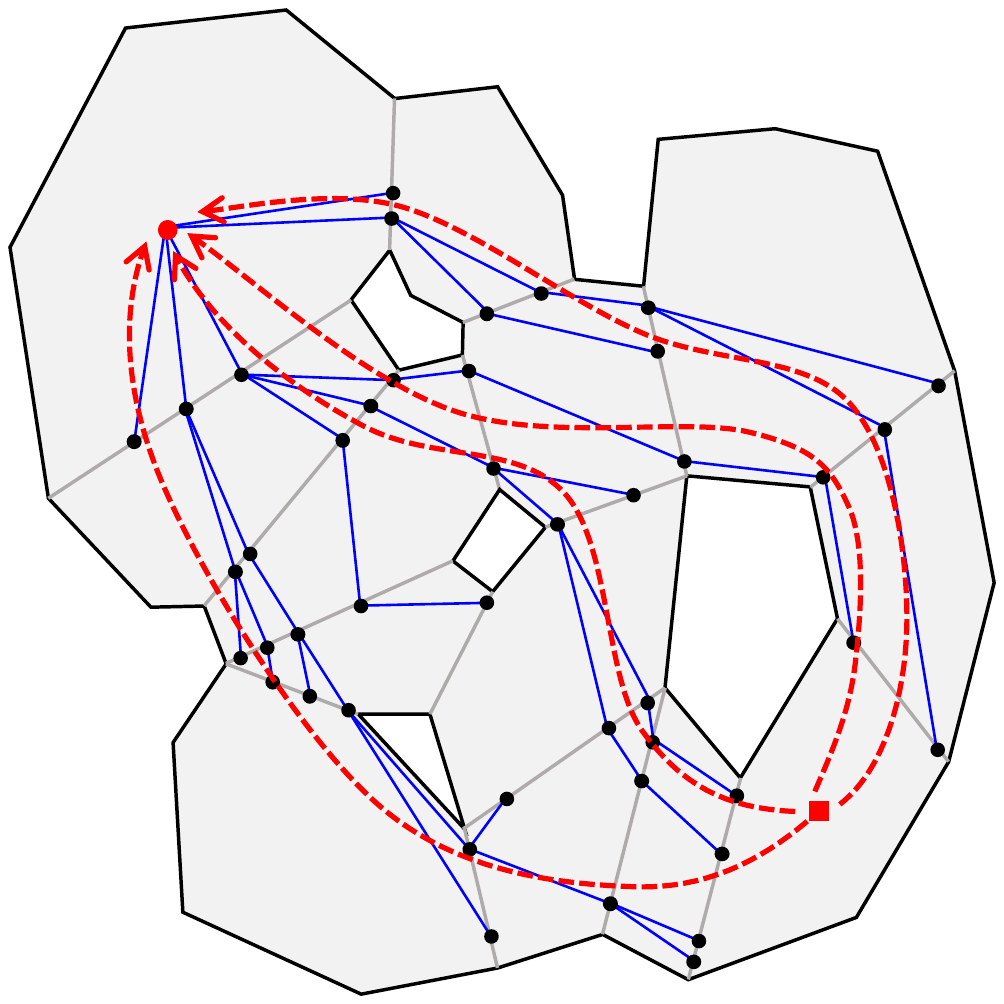}
    \label{fig_CDTRRT_4}}
  \hfil
  \subfloat[]{\includegraphics[width=2.0in]{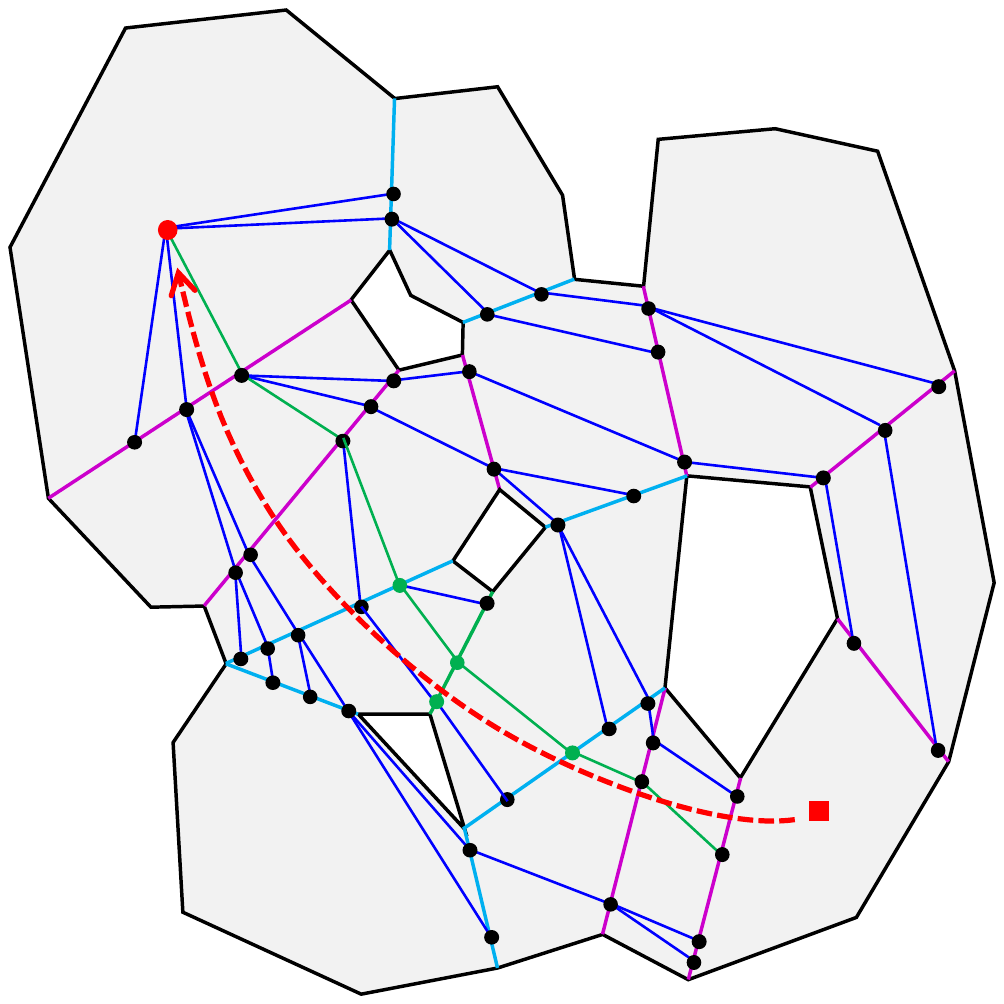}
    \label{fig_CDTRRT_5}}
  \hfil
  \subfloat[]{\includegraphics[width=2.0in]{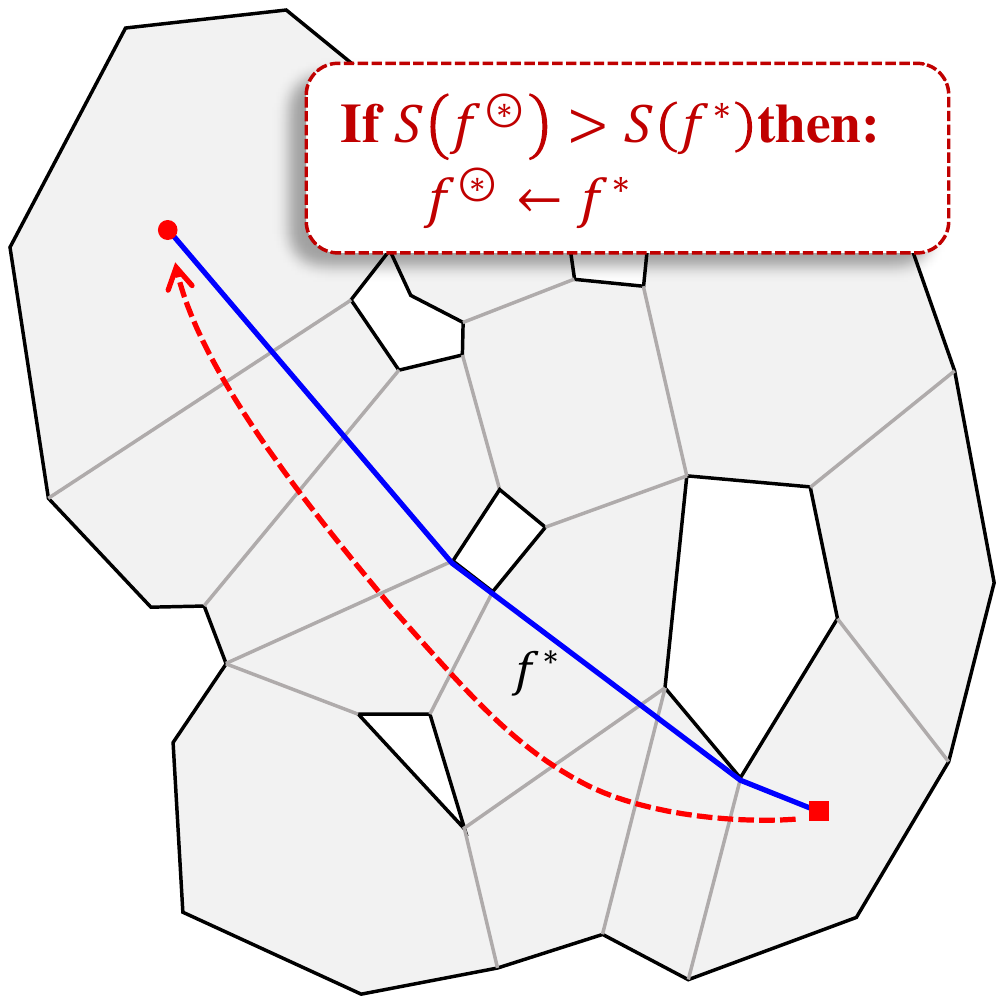}
    \label{fig_CDTRRT_6}}
  \caption{Illustration of the CDT-RRT* algorithm, the red circles and squares in the figure are the start and end points, respectively. (a) Initialize the map by fitting the free space with simple polygons, performing convex division, and constructing the topology graph. (b) After $x_{init}$ and $x_{goal}$ are determined, the CDT-RRT* proceeds with a pruning process, gradually removing redundant branches starting from the singly-connected convex polygon along the direction indicated by the red arrow. (c) CDT-RRT* will only sample one point ever cutline in the initial stage. (d) CDT-RRT* only samples on the cutline, using the sampling points adjacent to the $x_{goal}$ to backtrack the coding of feasible homotopy classes in $P(X_{topo})$, the red dotted lines are the four homotopy path classes currently obtained. (e) For the four homotopy classes already present in subfigure (d), sampling equation (\ref{eq_ST}) increases the sampling probability on the green and blue cutlines in this subfigure in order to sample some other optimal homotopy classes. (f) Using \textbf{Algorithm~\ref{alg5}} to determine the shortest path (blue solid line) in the newly sampled path homotopy class (red dashed line).}
  \label{fig_CDTRRT}
\end{figure*}

The pseudocode of CDT-RRT* is shown in \textbf{Algorithm~\ref{alg2}}. CDT-RRT* is divided into two stages: map initialisation stage (lines 2-5) and path search stage (lines 6-32).

In the map initialisation stage, the algorithm obtains the free space $X_{free}$ of the agent on the basis of the input map, and uses a simple polygon to fit $X_{free}$. Thereafter, it performs convex division, and builds the topology graph $X_{topo}$ according to the division result.

In the path search stage, the algorithm first obtains the start and end points of the planning task and initialises the tree $\mathcal{T}$ with $x_{init}$ as its root (lines 7,8). $\xi_{old}$ and $Q_r$ are two different queues initialised in line 8, where $\xi_{old}$ is used to record the homotopy path class that has been searched, and $Q_r$ is used for the rewiring process. Thereafter, the CDT-RRT* uses \textbf{Algorithm~\ref{alg3}} to prune $X_{topo}$ according to $x_{init}$ and $x_{goal}$ (line 9).  \textbf{Algorithm~\ref{alg3}} not only returns the pruned topological graph $X'_{topo}$, but also attempts to return the optimal CDT encoding $\mathrm{f}^\circledast$ between $x_{init}$ and $x_{goal}$. If $\mathrm{f}^\circledast$ is not None, the global optimal path $f^\circledast$ between $x_{init}$ and $x_{goal}$ can be obtained directly using $\mathrm{f}^\circledast$ (lines 10,11).

In lines 13-30 of the pseudocode, CDT-RRT* performs $N$ sampling search iterations. As shown in Fig.~\ref{fig_CDTRRT}, sampled nodes $x_{rand}$ (line 10) will be connected to a nearby node $x_{closest}$ and added to the tree (lines 15-17). The mathematical expression for line 16 of the algorithm is as
\begin{equation}
  \label{eq_FindClosest}
  x_{closest} = \arg \min _{x \in Q_{near}} \mathrm{Cost}(x) + \|x_{rand}-x\|,
\end{equation}
here $\mathrm{Cost}(x)$ as
\begin{equation}
  \label{eq_PointCost}
  \mathrm{Cost}(x)=S(f_x),
\end{equation}
and here $f_x$ is the path from $x_{init}$ to $x$ using $\mathcal{T}$.
If the node $x_{rand}$ is close to the node $x_{goal}$, then it is stored in $Q_g$ (lines 19,20). Otherwise, \textbf{Algorithm~\ref{alg4}} is used to perform rewiring at the tree-changing locations (line 22).

\begin{algorithm}[H]
  \caption{CDT-RRT*.}\label{alg:alg2}
  \begin{algorithmic}[1]
    \STATE {\textbf{Input: }}$Map$
    \STATE $X_{free} \gets $GetFreeSpace($Map$)
    \STATE $SimplePolygon \gets $PolygonFit($X_{free}$)
    \STATE Use \textbf{Algorithm~\ref{alg1}} to get $polygonlist$
    \STATE $X_{topo} \gets $BuildTopologyGraph($polygonlist$)
    \STATE \textbf{loop}
    \STATE \hspace{0.5cm}$x_{init},x_{goal},N \gets$ WaitingTaskInput()
    \STATE \hspace{0.5cm}Initialise $\mathcal{T}$ with $x_{init}$, $\xi_{old}$, $Q_r$, $Q_g$
    \STATE \hspace{0.5cm}$X'_{topo}, \mathrm{f}^\circledast \gets$ ReduceBranches($X_{topo},x_{init},x_{goal}$)
    \STATE \hspace{0.5cm}\textbf{if $\mathrm{f}^\circledast$} is not None \textbf{then}
    \STATE \hspace{1.0cm}$f^\circledast\gets$GetShortestPath($\mathrm{f}^\circledast,x_{init},x_{goal}$)
    \STATE \hspace{0.5cm}\textbf{else}
    \STATE \hspace{1.0cm}\textbf{for $i=0$ to $N$ do}
    \STATE \hspace{1.5cm}Sample $x_{rand}$ on cutlines of $X'_{topo}$ using (\ref{eq_ST})
    \STATE \hspace{1.5cm}$Q_{near} \gets$ FindNodesNear($X'_{topo},\mathcal{T},x_{rand}$)
    \STATE \hspace{1.5cm}$x_{closest} \gets$ FindClosest($x_{rand},Q_{near}$)
    \STATE \hspace{1.5cm}AddNodeToTree($\mathcal{T},x_{rand},x_{closest}$)
    \STATE \hspace{1.5cm}Push $x_{rand}$ to the end of $Q_r$
    \STATE \hspace{1.5cm}\textbf{if} $x_{rand}$ is near $x_{goal}$ \textbf{then}
    \STATE \hspace{2.0cm}Push $x_{rand}$ to the end of $Q_g$
    \STATE \hspace{1.5cm}\textbf{else}
    \STATE \hspace{2.0cm}RewireRandomNode($Q_r,\mathcal{T}$)
    \STATE \hspace{1.5cm}\textbf{for $x_e$ in $Q_g$ do}
    \STATE \hspace{2.0cm}Use $\mathcal{T}$ to backtrack $x_e$ to get the path $f$
    \STATE \hspace{2.0cm}$\mathrm{f} \gets \Gamma \circ (f*l^{x_{goal}}_{x_e})$
    \STATE \hspace{2.0cm}\textbf{if $\mathrm{f} \notin \xi_{old}$ then}
    \STATE \hspace{2.5cm}Push $\mathrm{f}$ into $\xi_{old}$
    \STATE \hspace{2.5cm}$f^*\gets$GetShortestPath($\mathrm{f},x_{init},x_{goal}$)
    \STATE \hspace{2.5cm}\textbf{if $\xi_{old} = \varnothing$ or $S(f^*)<S(f^\circledast)$ then}
    \STATE \hspace{3.0cm}$f^\circledast \gets f^*$
    \STATE \hspace{0.5cm}\textbf{output $f^\circledast$}
    \STATE \textbf{end loop}
  \end{algorithmic}
  \label{alg2}
\end{algorithm}

\begin{remark}
  \label{remarkRRN}
  In the tree $\mathcal{T}$, any path $f$ from any point to $x_{init}$ has the property that there are no repeated elements in the sequence $\Gamma \circ f$.
\end{remark}
\begin{proof}
  The aforementioned property of the tree $\mathcal{T}$ is mainly ensured by (\ref{eq_FindClosest}). Suppose the selection of $x_{closest}$ leads to the addition of $x_{rand}$ to the tree, and a repeated element appears in $\mathrm{f}_{x_{rand}}$, where $\mathrm{f}_{x_{rand}} = \Gamma \circ f_{x_{rand}}$ and $f_{x_{rand}}$ is the path from $x_{init}$ to $x_{rand}$ on $\mathcal{T}$. In other words, $f_{x_{rand}}$ enters the convex polygon $\mathrm{f}_{x_{rand}}(-1)$ twice (the second is from $x_{rand}$). Denote the point where $f_{x_{rand}}$ enters $\mathrm{f}_{x_{rand}}(-1)$ for the first time as $x_k$ (which is also a node in $\mathcal{T}$). Then we can decompose $f_{x_{rand}}$ as
  \begin{equation}
    \label{eq_RRN1}
    f_{x_{rand}} \cong_p f_1 * f_2,
  \end{equation}
  where $f_1$ is the path from $x_{init}$ to $x_k$ and $f_2$ is the path from $x_k$ to $x_{rand}$. Because both $x_k$ and $x_{rand}$ are on the convex polygon $\mathrm{f}_{x_{rand}}(-1)$, a new path $f'_{x_{rand}}$ from $x_{init}$ to $x_{rand}$ as
  \begin{equation}
    \label{eq_RRN2}
    f'_{x_{rand}} \cong_p f_1 * l^{x_{rand}}_{x_k},
  \end{equation}
  and
  \begin{equation}
    \label{eq_RRN3}
    \begin{aligned}
      S(f'_{x_{rand}}) & = S(f_1) * S(l^{x_{rand}}_{x_k}) \\
                       & < S(f_1) * S(f_2)                \\
                       & = S(f_{x_{rand}}).
    \end{aligned}
  \end{equation}
  This means that $x_k$ can result in a smaller cost for $x_{rand}$ compared to $x_{closest}$, which contradicts formula (\ref{eq_FindClosest}). Therefore, \textbf{Remark~\ref{remarkRRN}} holds.
\end{proof}
At the final stage of each iteration (lines 23-30), each element $x_e$ in the list $Q_g$ is used to obtain the path from $x_{init}$ to $x_e$ using $\mathcal{T}$ (line 24), and these paths will be encoded after concatenated with $l^{x_{goal}}_{x_e}$ (line 25). By determining if the encoded result is in $\xi_{old}$, we can determine if the homotopy class where this path is located has been searched (line 26). If the encoded result is not within $\xi_{old}$, we will push it into $\xi_{old}$ (line 27). Thereafter, we use \textbf{Algorithm~\ref{alg5}} to obtain the shortest path in this homotopy class (line 28) and determine if the current global optimal path requires replacement (lines 29,30). Finally, when the number of iterations reaches the preset $N$, the search task is terminated. Output planning results at the end of the task (line 31).

\noindent\textbf{Random point sampling:} Sampling in line 14 of \textbf{Algorithm~\ref{alg2}} uses (\ref{eq_ST}). $\rho_1$ and $\rho_2$ are random numbers between $[0,1]$. $c_1,c_2,\dots,c_n$ are the corresponding cutlines in $X'_{topo}$. $th_1,th_2,\dots,th_{n-1}$ are thresh used to select the cutlines when sampling. $\alpha $ is a very large number, $\beta $ is a number between $(0,1]$, $\mu_i$ is the number of sampling points adjacent to the $i$-th cutline, $\eta_i$ is the number of sampling points on the $i$-th cutline, and $\kappa_i$ is the number of occurrences of the $i$-th cutline in the sampled homotopy classes (list $\xi_{old}$).
\begin{equation}
  \label{eq_ST}
  \begin{aligned}
    {x_{rand}} = &
    \begin{cases}
      c_1(\rho_1), & \rho_2 \in[0,th_1),     \\
      c_2(\rho_1), & \rho_2 \in[th_1,th_2),  \\
                   & \vdots                  \\
      c_n(\rho_1), & \rho_2 \in[th_{n-1},1].
    \end{cases},                                                                                                                         \\
                 & \hspace{3em} th_k = \frac{\sum_{i = 1}^{k} \min(1,\mu_i)(1+\alpha 0^{\eta_i}) \beta^{\kappa_i}}{\sum_{i = 1}^{n} \min(1,\mu_i)(1+\alpha 0^{\eta_i}) \beta^{\kappa_i}}.
  \end{aligned}
\end{equation}
In (\ref{eq_ST}), \textbf{Algorithm~\ref{alg2}} samples on the cutlines of $X'_{topo}$. We select these cutlines using a roulette selection procedure. The probability of selecting these cutlines is non-uniform, and the probability that cutline is selected is affected by the following three factors:
\begin{enumerate}
  \item{If there are no nodes on the convex polygons adjacent to the cutline, then this cutline will not be selected. This is to ensure that each sampling point is valid during the initial growth phase of the tree (ensure that $Q_{near}$ is always non-empty in line 15 of \textbf{Algorithm~\ref{alg2}}). In (\ref{eq_ST}) this factor is realised by $\min(1,\mu_i)$.}
  \item{The number of occurrences of the cutline in the sampled homotopy classes. This factor is represented by $\beta^{\kappa_i}$ in (32). This part is used to control the tendency to explore unknown homotopy classes. The closer $\beta$ is to 0, the more inclined the algorithm is to sample on the cutlines where $\kappa$ is relatively small, and if a split line has a larger $\kappa$, then it has a lower probability of being selected.}
  \item{Rapid growth of the tree at the beginning. To rapidly make all the cutlines in $X'_{topo}$ satisfy the sampling condition (not limited by $\min(1,\mu_i)$), each cutline is sampled only once during the initial sampling. This sampling is achieved by means of the $\alpha 0^{\eta_i}$ term in (\ref{eq_ST}). For a cutline where $\eta_i$ is zero, an additional very large probability of being sampled will be obtained.}
\end{enumerate}

\subsection{Reduce Branches}
\begin{figure}[!t]
  \centering
  \subfloat[]{\includegraphics[width=1.6in]{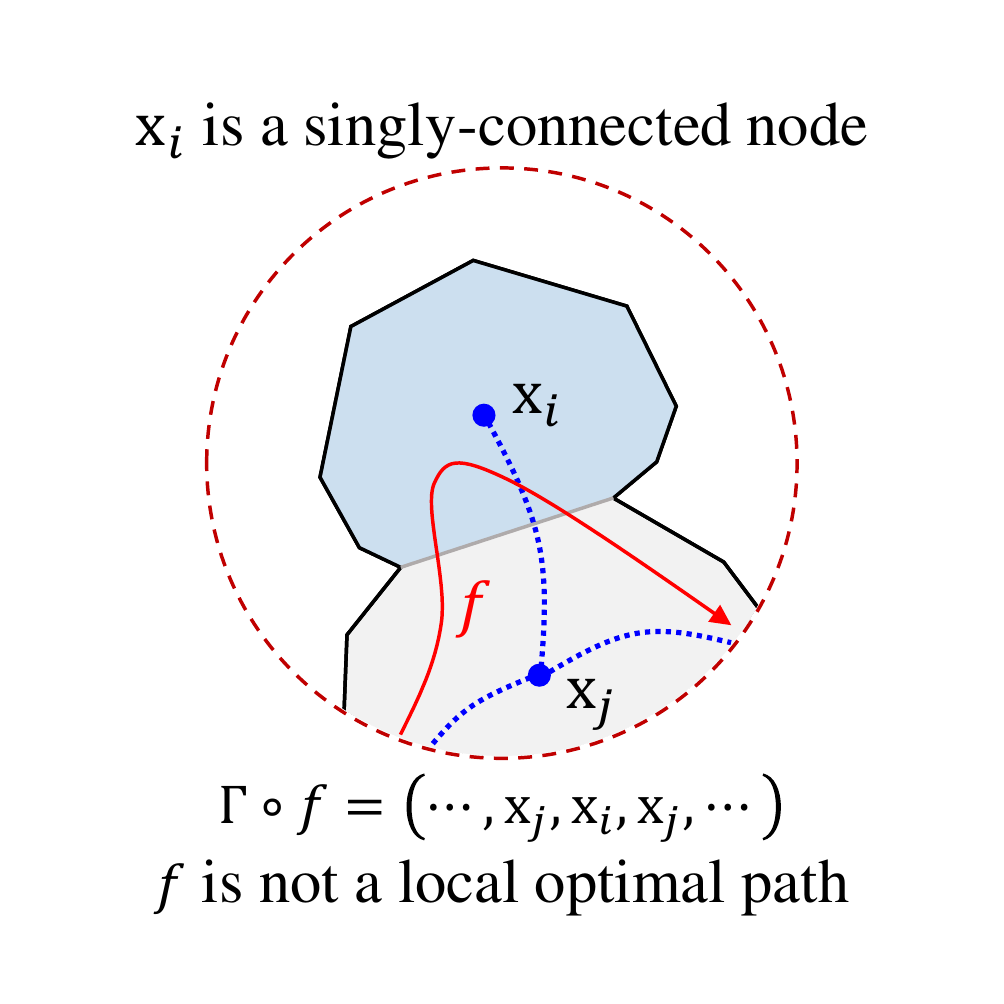}
    \label{fig_RB_1}}
  \hfil
  \subfloat[]{\includegraphics[width=1.6in]{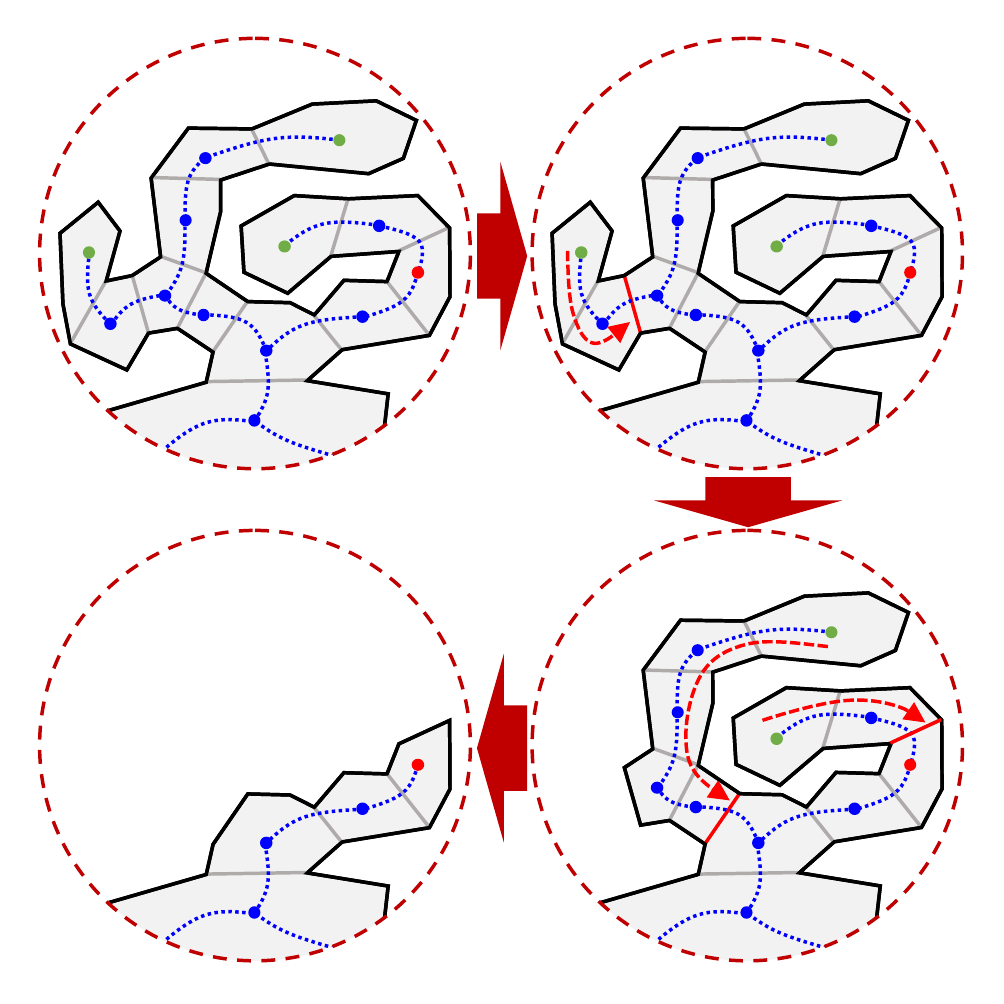}
    \label{fig_RB_2}}
  \hfil
  \subfloat[]{\includegraphics[width=1.6in]{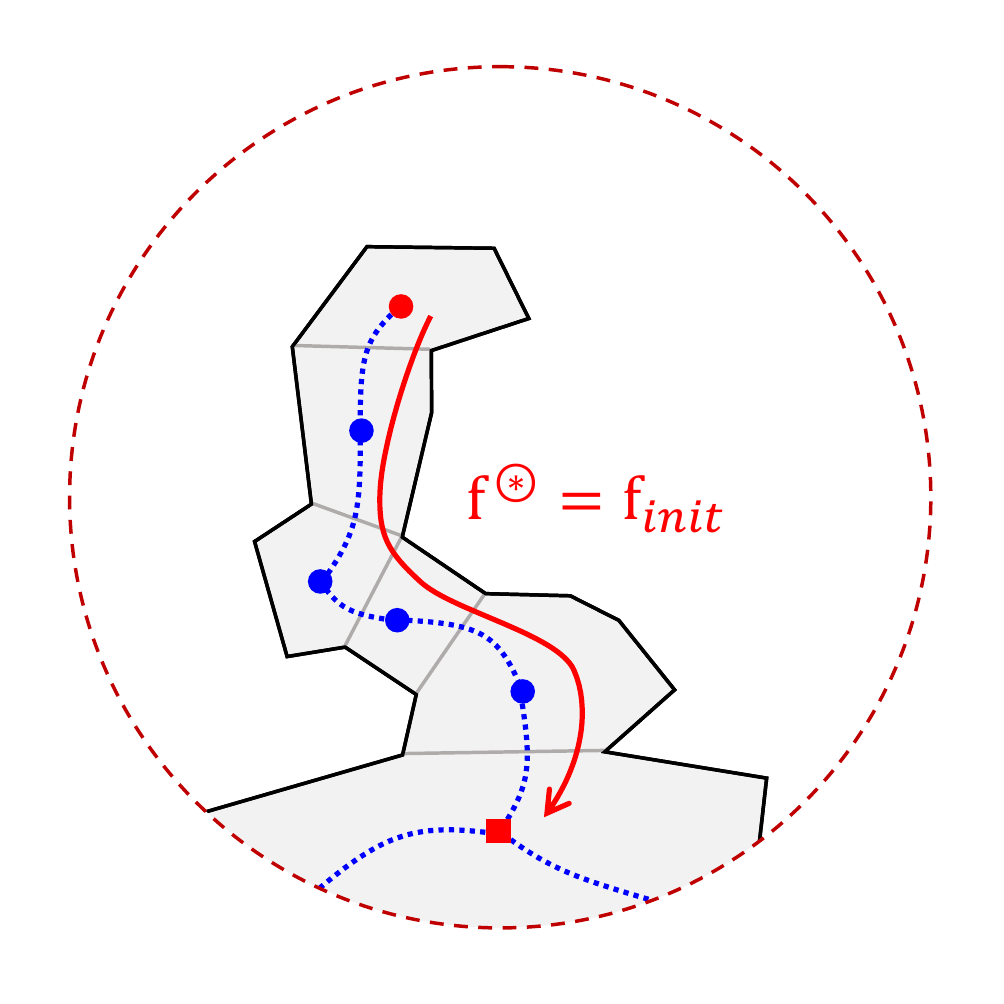}
    \label{fig_RB_3}}
  \hfil
  \subfloat[]{\includegraphics[width=1.6in]{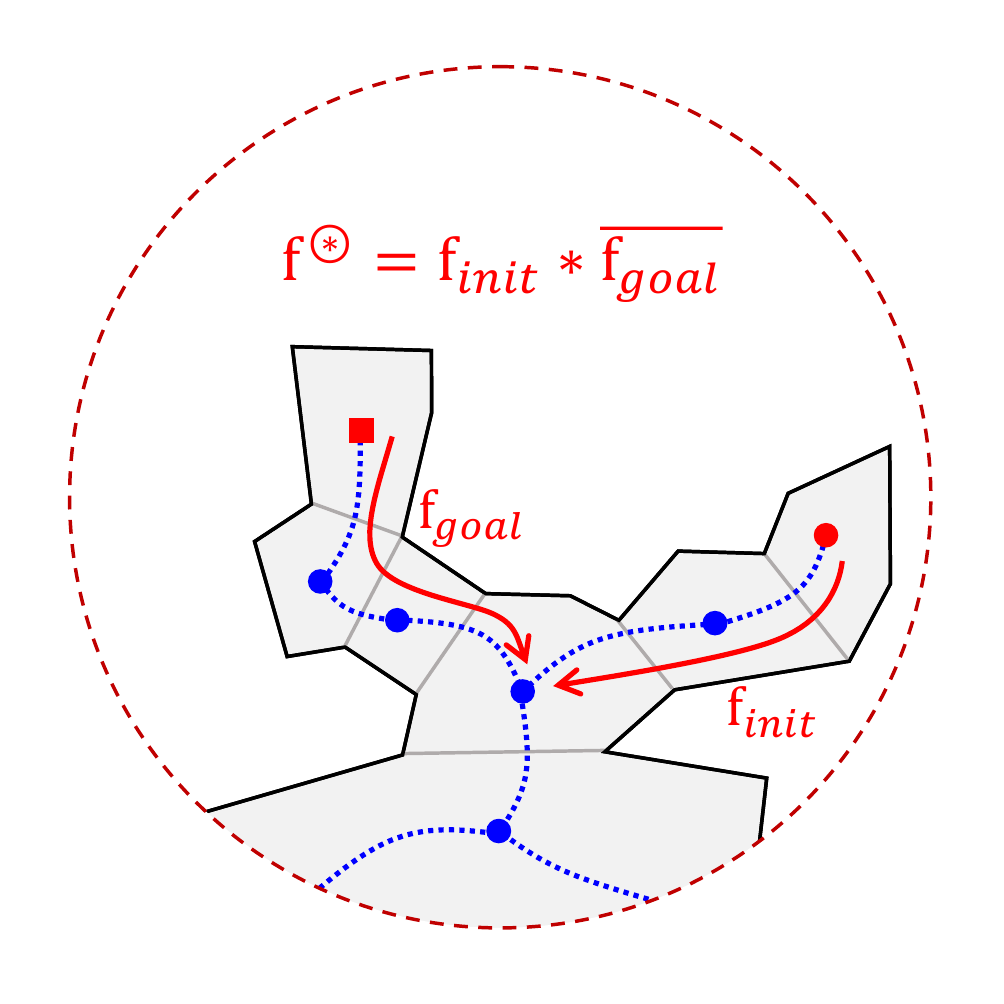}
    \label{fig_RB_4}}
  \caption{Illustration of the reduce branches algorithm. (a) If $x_{init}$ and $x_{goal}$ are not inside the single-connected convex polygon $\mathrm{x}_i$, then the optimal path will not pass through $\mathrm{x}_i$. (b) Starting from the green simply-connected node, irrelevant branches are successively removed until the currently considered node has two or more neighbors, or the node is either $\mathrm{x}_{init}$ or $\mathrm{x}_{goal}$ (the red nodes). (c) In further backtracking of $\mathrm{x}_{init}$ (or $\mathrm{x}_{goal}$), if $\mathrm{f}_{init}$ (or $\mathrm{f}_{goal}$) meets $\mathrm{x}_{goal}$ (or $\mathrm{x}_{init}$) along the way, the CDT encoding of the global optimal path can be directly returned as $\mathrm{f}_{init}$ (or $\overline{\mathrm{f}_{goal}}$). (d) If $\mathrm{f}_{init}$ meets $\mathrm{f}_{goal}$, then the CDT encoding of the global optimal path can be directly returned as $\mathrm{f}_{init} * \overline{\mathrm{f}_{goal}}$.}
  \label{fig_RB}
\end{figure}
If the start and end points of the path planning task are determined, then some special paths will be determined to be non-optimal paths. Hence these paths do not require searching.
\begin{corollary}
  \label{theorem10}
  If a convex polygon has only one neighbour and neither $x_{init}$ nor $x_{goal}$ is contained within it, then the optimal route must not pass through this convex polygon.
\end{corollary}
\textbf{Corollary~\ref{theorem10}} is easily obtained using \textbf{Corollary~\ref{theorem9}}. As shown in Fig.~\ref{fig_RB} (a), Since this type of convex polygon only has one neighbor, if the path $f\in P(X_{free},x_{init},x_{goal})$ passes through this convex polygon, then $f$ will pass through its neighbor at least twice. According to \textbf{Corollary~\ref{theorem9}}, $f$ must not be the global optimal path from $x_{init}$ to $x_{goal}$. Therefore, these types of convex polygons do not need to be considered during planning. The nodes in the topological graph $X_{topo}$ that correspond to such convex polygons are referred to as singly-connected node, and such nodes can be removed from $X_{topo}$ (except for the nodes where $x_{init}$ and $x_{goal}$ are located). As shown in Fig.~\ref{fig_RB} (b), with the removal of the original singly-connected nodes, some nodes in the topology graph may become new singly-connected nodes. Therefore, the algorithm requires backtracking to remove all possible singly-connected nodes. The pseudocode for removing singly-connected nodes is shown in lines 2-11 of \textbf{Algorithm~\ref{alg3}}.
\begin{algorithm}[H]
  \caption{ReduceBranches.}\label{alg:alg3}
  \begin{algorithmic}[1]
    \STATE {\textbf{Input: }}$X_{topo},x_{init},x_{goal}$
    \STATE Copy $X_{topo}$ to $X'_{topo}$
    \STATE $\mathrm{x}_{init}, \mathrm{x}_{goal} \gets \Gamma \circ x_{init}, \Gamma \circ x_{goal}$
    \STATE \textbf{for $\mathrm{x}$ in} SingleConnectedNodes of $X_{topo}$ \textbf{do}
    \STATE \hspace{0.5cm}$\mathrm{x}_{temp} \gets \mathrm{x}$
    \STATE \hspace{0.5cm}\textbf{while} $\mathrm{x}_{temp}$ is a Single-Connected Node of $X'_{topo}$ \textbf{do}
    \STATE \hspace{1.0cm}\textbf{if} $\mathrm{x}_{temp} \notin \{\mathrm{x}_{init}, \mathrm{x}_{goal}\}$ \textbf{then}
    \STATE \hspace{1.5cm}Remove $\mathrm{x}_{temp}$ from $X'_{topo}$
    \STATE \hspace{1.0cm}\textbf{else}
    \STATE \hspace{1.5cm}\textbf{break}
    \STATE \hspace{1.0cm}$\mathrm{x}_{temp} \gets$ neighbor($\mathrm{x}_{temp}$)

    \STATE $\mathrm{f}_{init} \gets e_{\mathrm{x}_{init}}$
    \STATE \textbf{while} $\text{neighbors}(\mathrm{f}_{init}(-1))/\mathrm{f}_{init} < 2$ \textbf{do}
    \STATE \hspace{0.5cm}$\mathrm{x}_{near} \gets$ neighbors$(\mathrm{f}_{init}(-1))/\mathrm{f}_{init}$
    \STATE \hspace{0.5cm}Push $\mathrm{x}_{near}$ to the end of $\mathrm{f}_{init}$
    \STATE \hspace{0.5cm}\textbf{if} $\mathrm{x}_{near}=\mathrm{x}_{goal}$ \textbf{then}
    \STATE \hspace{1.0cm}\textbf{return} $X'_{topo}, \mathrm{f}_{init}$

    \STATE $\mathrm{f}_{goal} \gets e_{\mathrm{x}_{goal}}$
    \STATE \textbf{while} $\text{neighbors}(\mathrm{f}_{goal}(-1))/\mathrm{f}_{goal} < 2$ \textbf{do}
    \STATE \hspace{0.5cm}$\mathrm{x}_{near} \gets$ neighbors$(\mathrm{f}_{goal}(-1))/\mathrm{f}_{goal}$
    \STATE \hspace{0.5cm}Push $\mathrm{x}_{near}$ to the end of $\mathrm{f}_{goal}$
    \STATE \hspace{0.5cm}\textbf{if} $\mathrm{x}_{near}=\mathrm{x}_{init}$ \textbf{then}
    \STATE \hspace{1.0cm}\textbf{return} $X'_{topo}, \overline{\mathrm{f}_{goal}}$

    \STATE \textbf{if} $\mathrm{f}_{goal}(-1)=\mathrm{f}_{init}(-1)$ \textbf{then}
    \STATE \hspace{0.5cm}\textbf{return} $X'_{topo}, \mathrm{f}_{init} * \overline{\mathrm{f}_{goal}}$
    \STATE \textbf{else}
    \STATE \hspace{0.5cm}\textbf{return} $X'_{topo}$, None
  \end{algorithmic}
  \label{alg3}
\end{algorithm}

In the algorithm, $X_{topo}$ is first cloned and stored in $X'_{topo}$ (line 2). Thereafter, every singly-connected node in $X_{topo}$ (lines 4-11) is backtracked. In each backtracking process, the current node is deleted in $X'_{topo}$ until the current node belongs to $\{\mathrm{x}_{init}, \mathrm{x}_{goal}\}$ or is not a singly-connected node in $X'_{topo}$. After that, $X'_{topo}$ is the topological graph after reduce branches.

In the second half of \textbf{Algorithm~\ref{alg3}} (lines 12-27), we present a method that can directly provide the CDT encoding of the globally optimal path from $x_{init}$ to $x_{goal}$ in certain special cases. In the pruned topological graph, we can continue to backtrack from $\mathrm{x}_{init}$ and $\mathrm{x}_{goal}$ to construct the backtrack paths $\mathrm{f}_{init}$, $\mathrm{f}_{goal}$, respectively (lines 12-17 and lines 18-23). As shown in Fig.~\ref{fig_RB} (c), if $\mathrm{x}_{init}$ meets $\mathrm{f}_{goal}$ or $\mathrm{x}_{goal}$ meets $\mathrm{f}_{init}$, then we will return $X'_{topo}$, and the CDT encoding of the global optimal path is $\mathrm{f}_{init}$ or $\overline{\mathrm{f}_{goal}}$ (lines 17,23). As shown in Fig.~\ref{fig_RB} (d), when $\mathrm{x}_{init}$ fails to meet $\mathrm{f}_{goal}$ and $x_{init}$ fails to meet $\mathrm{f}_{goal}$, the path must first depart from the initial point, pass through $\mathrm{f}_{init}(-1)$, then through $\mathrm{f}_{goal}(-1)$, and finally reach the goal point. If $\mathrm{f}_{init}(-1) = \mathrm{f}_{goal}(-1)$ in this case, we can directly give the CDT encoding of the global optimal path as $\mathrm{f}_{init} * \overline{\mathrm{f}_{goal}}$.\footnote{This is because for any $f \in P(X_{free},x_{init},x_{goal})$, if $\Gamma \circ f$ passes through nodes other than $\mathrm{f}_{init} * \overline{\mathrm{f}_{goal}}$, then $\Gamma \circ f$ will pass through $\mathrm{f}_{init}(-1)$ twice, so according to \textbf{Corollary~\ref{theorem9}}, $f$ is not optimal.}
Finally, \textbf{Algorithm~\ref{alg3}} returns the pruned topological graph $X'_{topo}$, and the CDT encoding $\mathrm{f}_{init} * \overline{\mathrm{f}_{goal}}$ of the global optimal path or an empty result if $\mathrm{f}^\circledast$ cannot be obtained (lines 25,27).

\subsection{Rewire Random Node}
Rewire random node is necessary when $x_{rand}$ is required to be added to the tree, similar as with RRT*. If $x_{rand}$ is not in the vicinity of $x_{goal}$, then $x_{rand}$ may cause a change in the connectivity of other nodes in $\mathcal{T}$. We consider the rewiring of nodes near $x_{rand}$. The pseudocode for rewire random node is shown in \textbf{Algorithm~\ref{alg4}}.

Rewiring is done when a node $x_{near}$ gets a lower cost value by passing from node $x_r$ instead of its parent (lines 6-10).
If rewiring happens to any node $x_{near}$, \textbf{Algorithm~\ref{alg4}} adds $x_{near}$ to $Q_r$ since nodes around $x_{near}$ have the potential to get rewired (line 9). \textbf{Algorithm~\ref{alg4}} continues to iterate until $Q_r$ is empty, at which point the rewiring optimisation of $\mathcal{T}$ is completed.
\begin{algorithm}[H]
  \caption{RewireRandomNode.}\label{alg:alg4}
  \begin{algorithmic}[1]
    \STATE {\textbf{Input: }}$Q_r,\mathcal{T}$
    \STATE \textbf{repeat}
    \STATE \hspace{0.5cm}$x_r \gets$ PopFirst($Q_r$)
    \STATE \hspace{0.5cm}$Q_{near} \gets$ FindNodesNear($X'_{topo},\mathcal{T},x_r$)

    \STATE \hspace{0.5cm}\textbf{for $x_{near}$ in $Q_{near}$ do}
    \STATE \hspace{1.0cm}$cost_{old} = \mathrm{Cost}(x_{near})$
    \STATE \hspace{1.0cm}$cost_{new} = \mathrm{Cost}(x_r) + \|x_r - x_{near}\|$
    \STATE \hspace{1.0cm}\textbf{if} $cost_{old}>cost_{new}$ \textbf{then}
    \STATE \hspace{1.5cm}Push $x_{near}$ to the end of $Q_r$
    \STATE \hspace{1.5cm}Parent($x_{near}$) $\gets x_r$
    \STATE \textbf{until} $Q_r$ is empty.
  \end{algorithmic}
  \label{alg4}
\end{algorithm}

\subsection{Get Shortest Path}
\textbf{Algorithm~\ref{alg:alg5}} is primarily used to fast obtain the shortest path in the homotopy path class. In line 2, $c'_1,c'_2,\cdots ,c'_m$ are the cutlines through which $\mathrm{f}$ passes subsequently. In line 3, $x_0=x_{s}$, and $x_{m+1}=x_{e}$, $x_1, x_2, \cdots , x_m$ are the midpoints of the cutlines $c'_1,c'_2,\cdots ,c'_m$ respectively. 
In lines 5-10, we draw inspiration from the Elastic Band algorithm to obtain a local optimal path. According to \textbf{Theorem~\ref{SSDR}}, we only update the control points of $f'$ on these $m$ cutlines, which greatly improves the convergence speed of $f'$.
In addition, the solution to line 8 of the algorithm is simply given by the following: (i) If lines $l^{x_{k+1}}_{x_{k-1}}$ and $c'_k$ intersect, $x_k$ is their intersection, and (ii) if lines $l^{x_{k+1}}_{x_{k-1}}$ and $c'_k$ do not intersect, $x_k$ is the endpoint of $c'_k$ that is closest to $l^{x_{k+1}}_{x_{k-1}}$.

\begin{algorithm}[H]
    \caption{GetShortestPath.}\label{alg:alg5}
    \begin{algorithmic}[1]
        \STATE {\textbf{Input: }}$\mathrm{f},x_s,x_e$
        \STATE Use $\mathrm{f}$ to get the cutline sequence $\{c'_1,c'_2,\cdots ,c'_m\}$
        \STATE $\{x_0,x_1,\cdots ,x_m,x_{m+1}\} \gets \{x_s,c'_1(\frac{1}{2}),\cdots ,c'_m(\frac{1}{2}),x_e\}$
        \STATE $f' \gets l_{x_0}^{x_1} * l_{x_1}^{x_2} * \cdots * l_{x_n}^{x_{n+1}}$
        \STATE \textbf{repeat}
        \STATE \hspace{0.5cm} ${cost} \gets S(f') $
        \STATE \hspace{0.5cm} \textbf{for $k \in \mathbb{N}^m_1$ do}
        \STATE \hspace{1.0cm} $x_k = \mathop {\arg \min }\limits_{x \in c'_k}S(l_{x_{k-1}}^x * l_x^{x_{k+1}})$
        \STATE \hspace{0.5cm} $f' \gets l_{x_0}^{x_1} * l_{x_1}^{x_2} * \cdots * l_{x_n}^{x_{n+1}}$
        \STATE \textbf{until} $S(f') - cost < \varepsilon $
        \STATE \textbf{return} $f'$
    \end{algorithmic}
    \label{alg5}
\end{algorithm}

\section{Experimental Studies and Results}
In this section, we conduct a number of simulation experiments to demonstrate the effectiveness and efficiency of the proposed algorithm. The computer we use is an Intel NUC (Next Unit of Computing). The NUC setup is an Intel i7-1165G7 (4.7GHz) with 16 GB of RAM. For the simulation, we use C++ on Ubuntu 18.04. For the experimental implementation, robot operating system (ROS) on Ubuntu 18.04 and a designed differential drive robot were used. The proposed algorithms were implemented on an IPC (Industrial Personal Computer) as the computational core for testing, the IPC setup is Intel Celeron J4125 (2.70 GHz) with 8 GB of RAM.

In experiments, two parameters were often used to compare the performance of the algorithms: $t_{init}$, which is the time of the initial solution; and $t_{2\%}$, which is the time to determine a solution of cost ($1.02\times C_{optimal}$), where $ C_{optimal}$ is the optimal cost.
\subsection{Map Initialisation to Build Topology Graph}
\begin{figure*}[!t]
  \centering
  \subfloat[]{\includegraphics[width=1.10in]{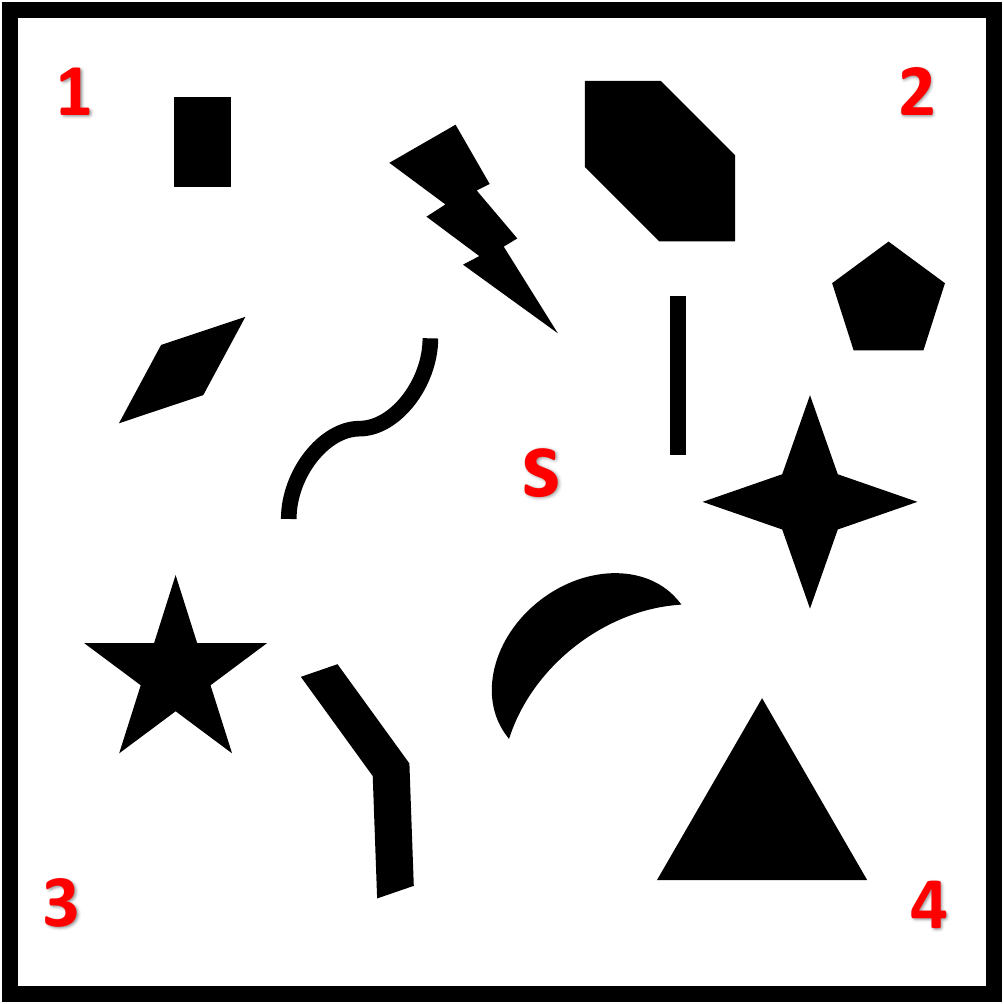}
    \label{fig_map_1}}
  \hfil
  \subfloat[]{\includegraphics[width=1.10in]{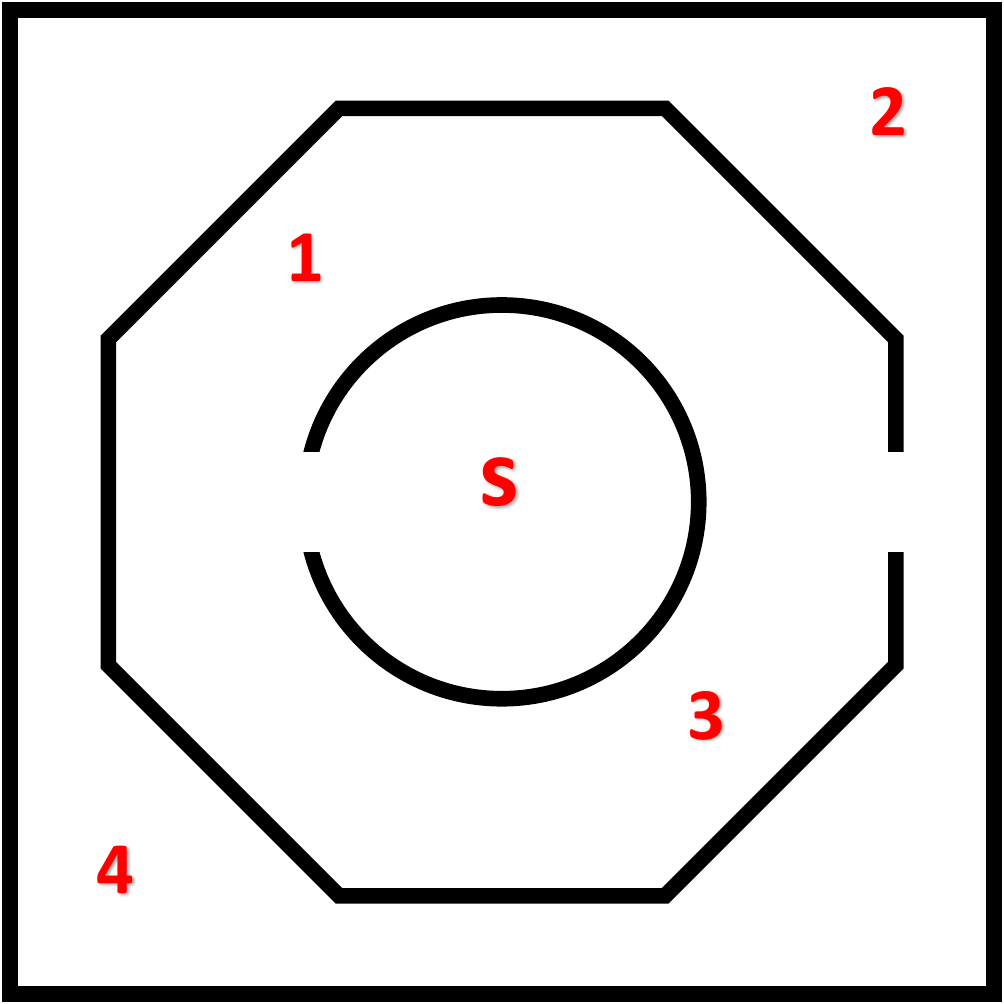}
    \label{fig_map_2}}
  \hfil
  \subfloat[]{\includegraphics[width=1.10in]{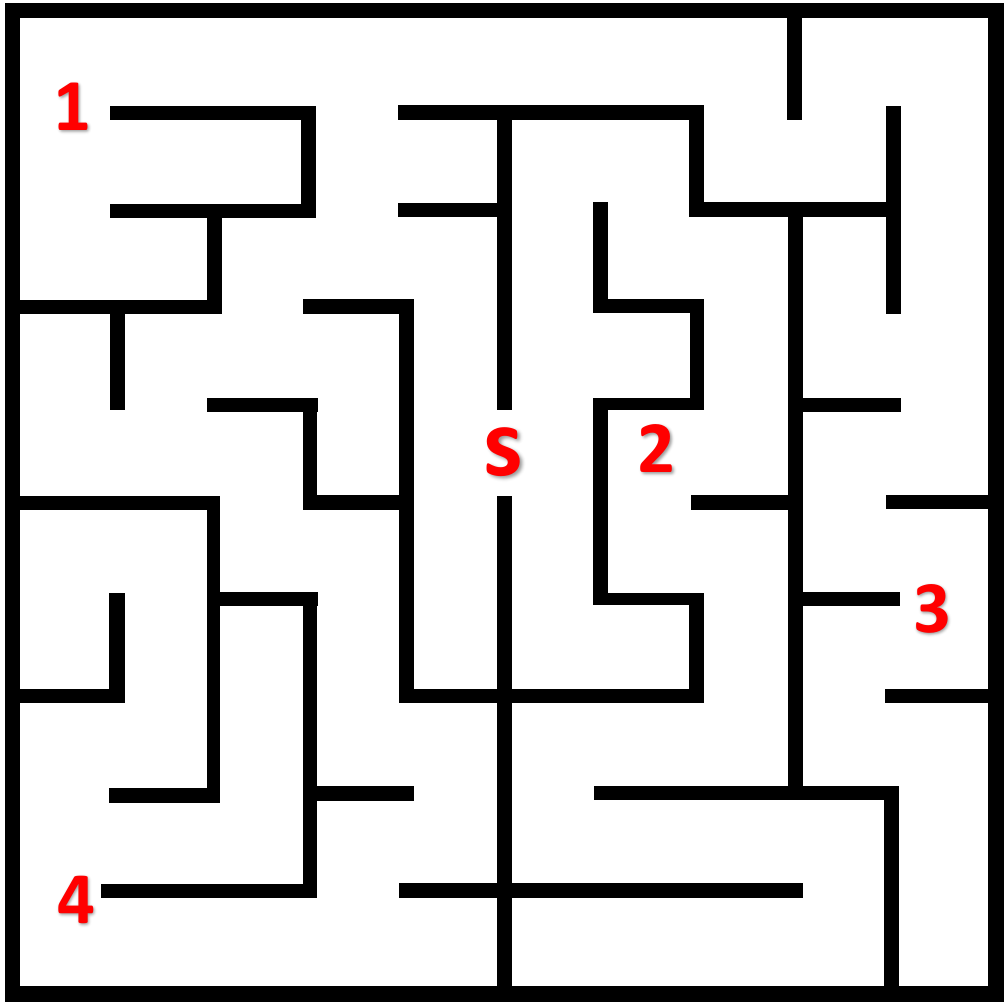}
    \label{fig_map_3}}
  \hfil
  \subfloat[]{\includegraphics[width=1.10in]{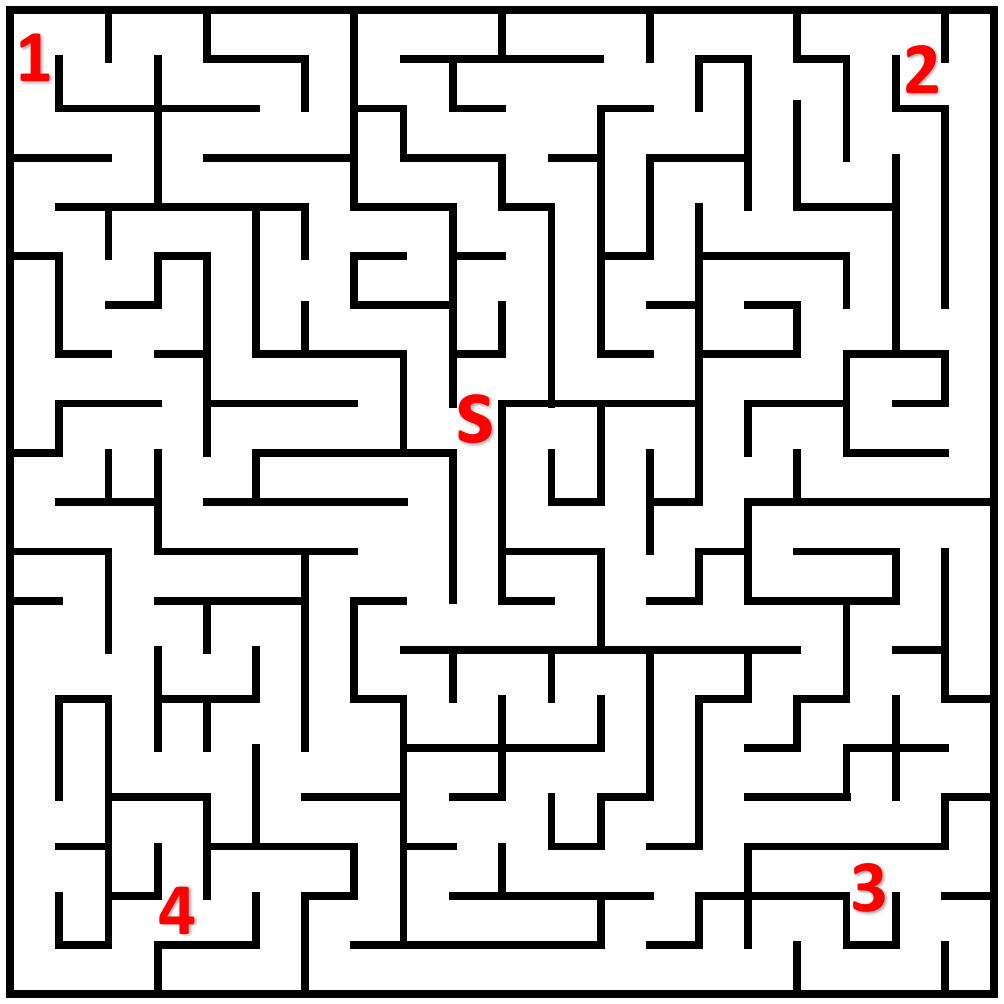}
    \label{fig_map_4}}
  \hfil
  \subfloat[]{\includegraphics[width=1.10in]{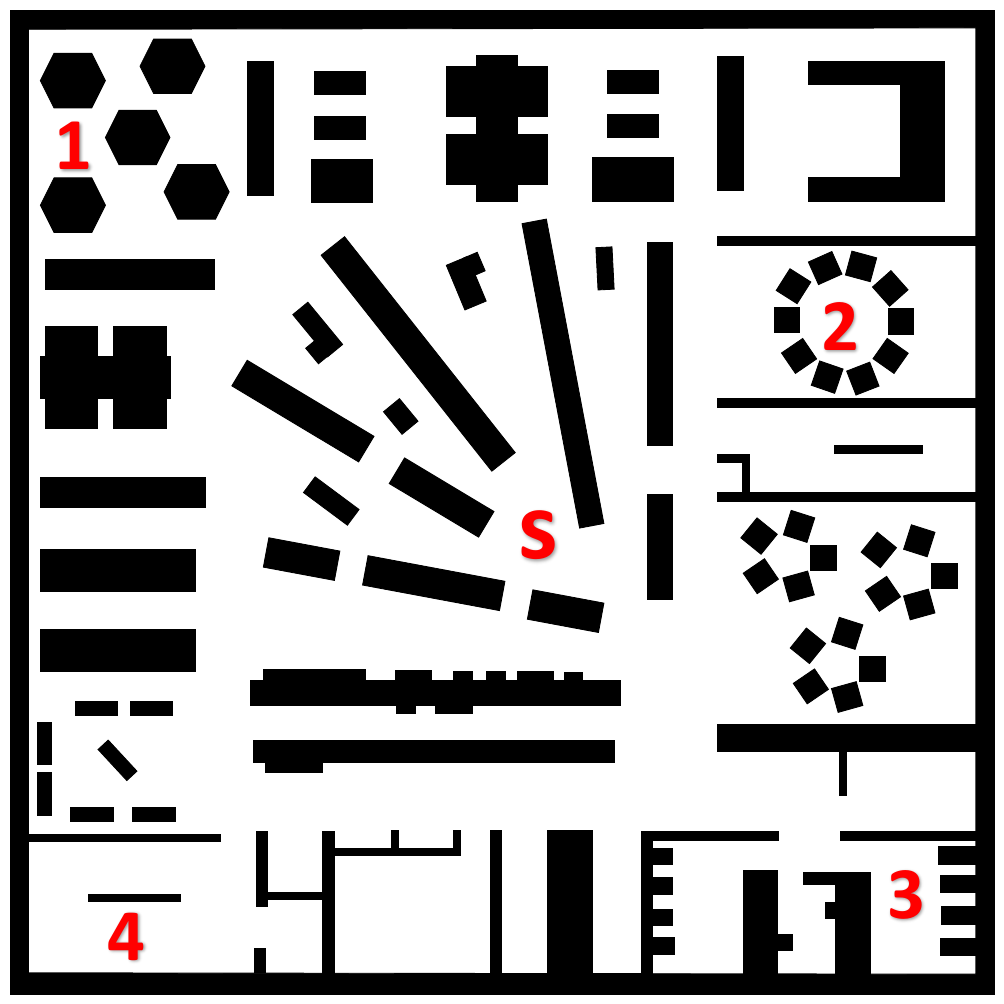}
    \label{fig_map_5}}
  \hfil
  \subfloat[]{\includegraphics[width=1.10in]{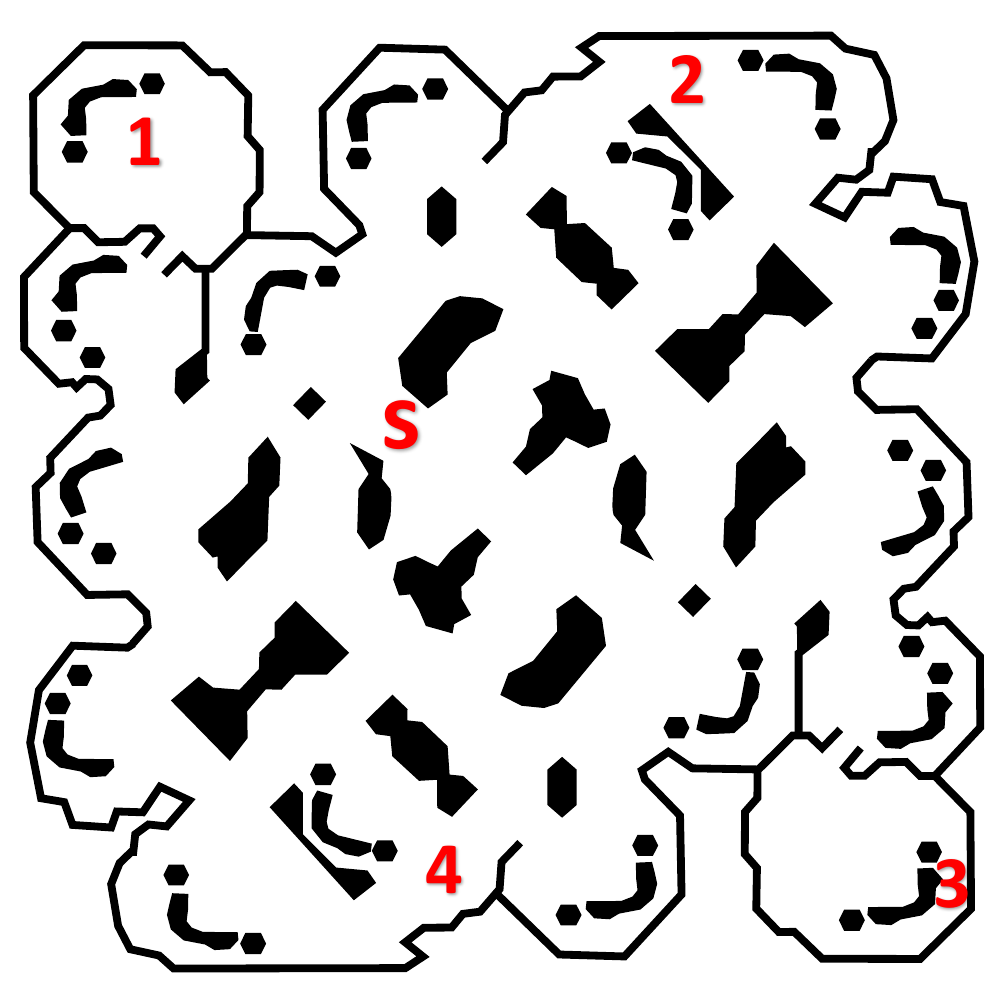}
    \label{fig_map_6}}
  \\
  \subfloat[]{\includegraphics[width=1.10in]{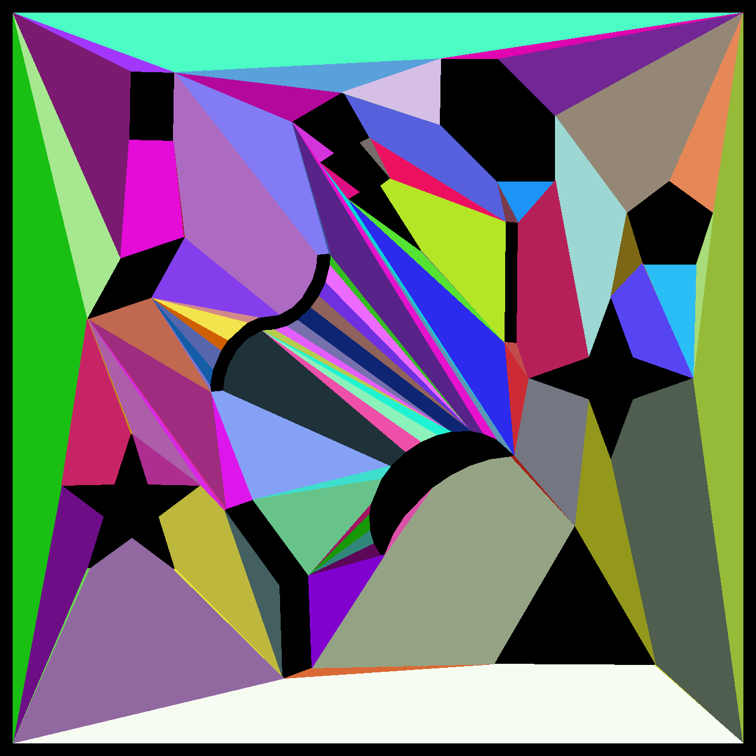}
    \label{fig_mapCD_1}}
  \hfil
  \subfloat[]{\includegraphics[width=1.10in]{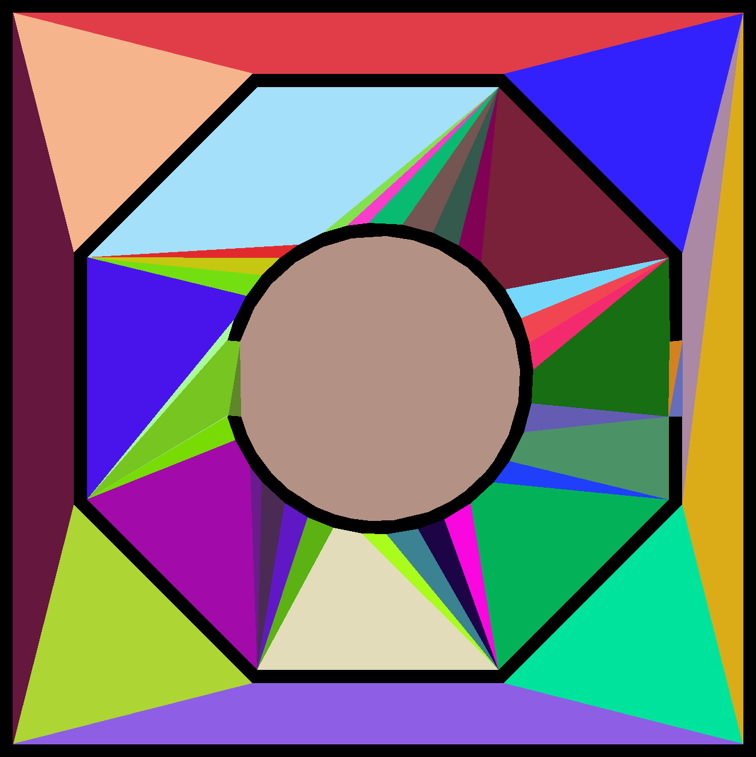}
    \label{fig_mapCD_2}}
  \hfil
  \subfloat[]{\includegraphics[width=1.10in]{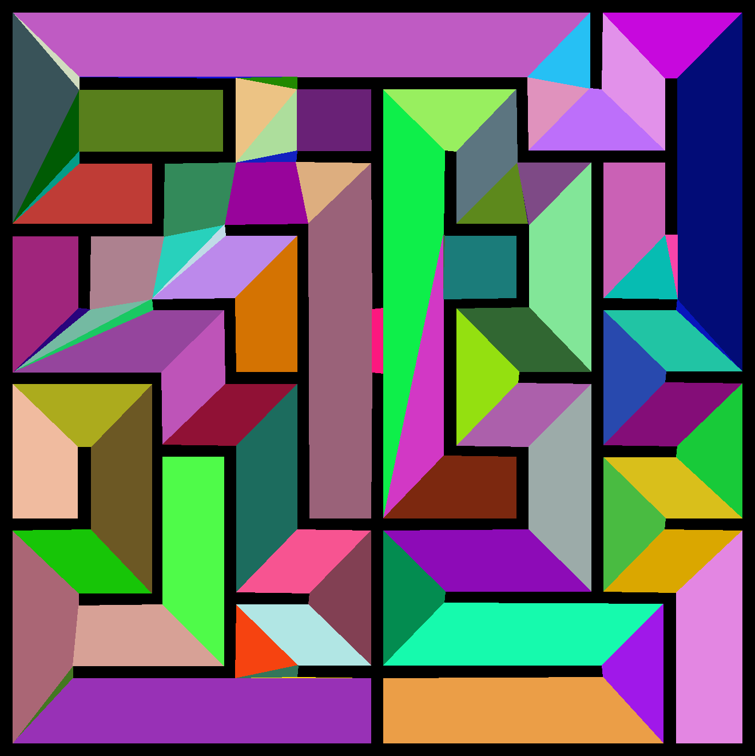}
    \label{fig_mapCD_3}}
  \hfil
  \subfloat[]{\includegraphics[width=1.10in]{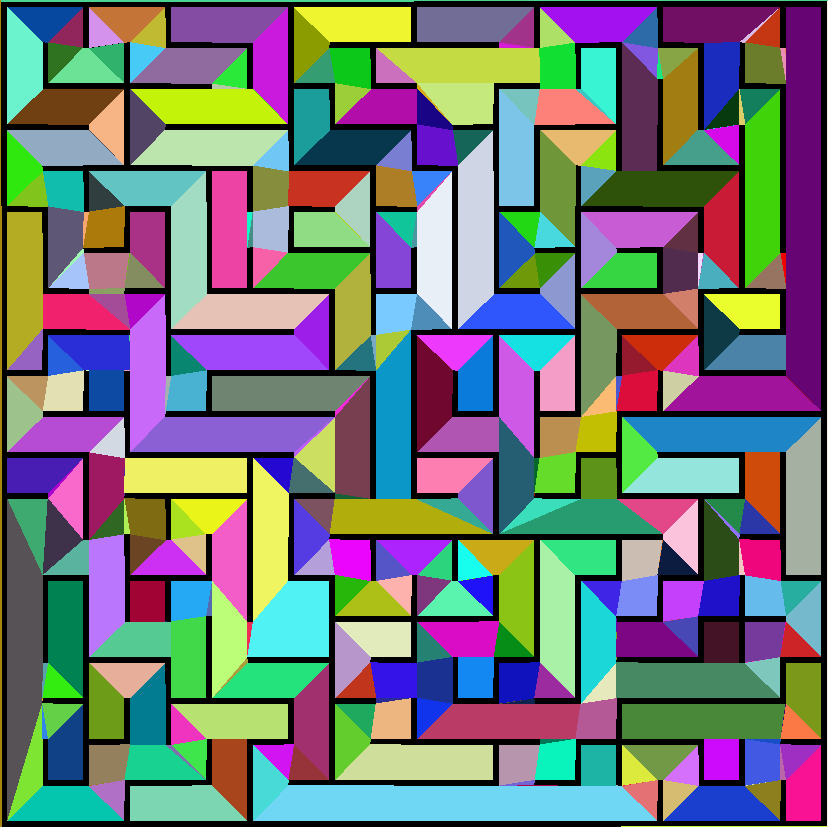}
    \label{fig_mapCD_4}}
  \hfil
  \subfloat[]{\includegraphics[width=1.10in]{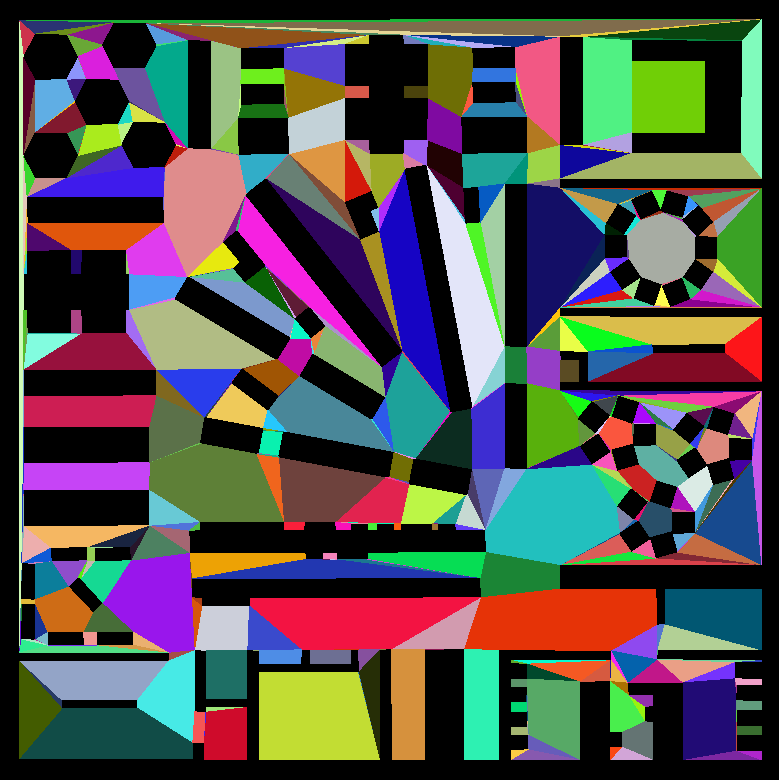}
    \label{fig_mapCD_5}}
  \hfil
  \subfloat[]{\includegraphics[width=1.10in]{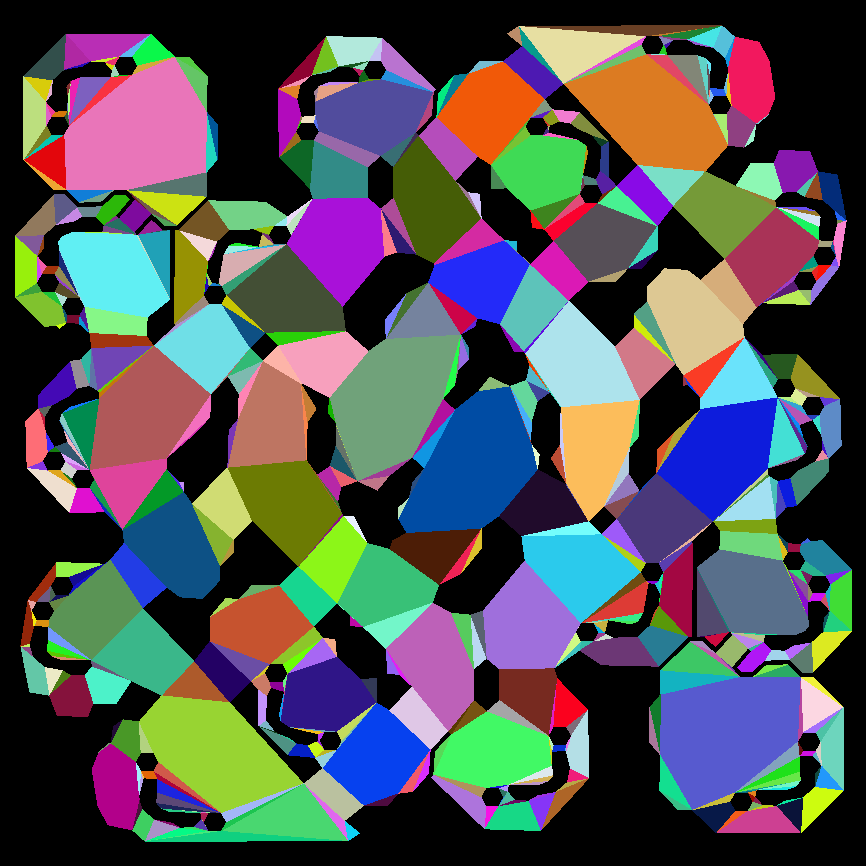}
    \label{fig_mapCD_6}}
  \\
  \subfloat[]{\includegraphics[width=1.14in]{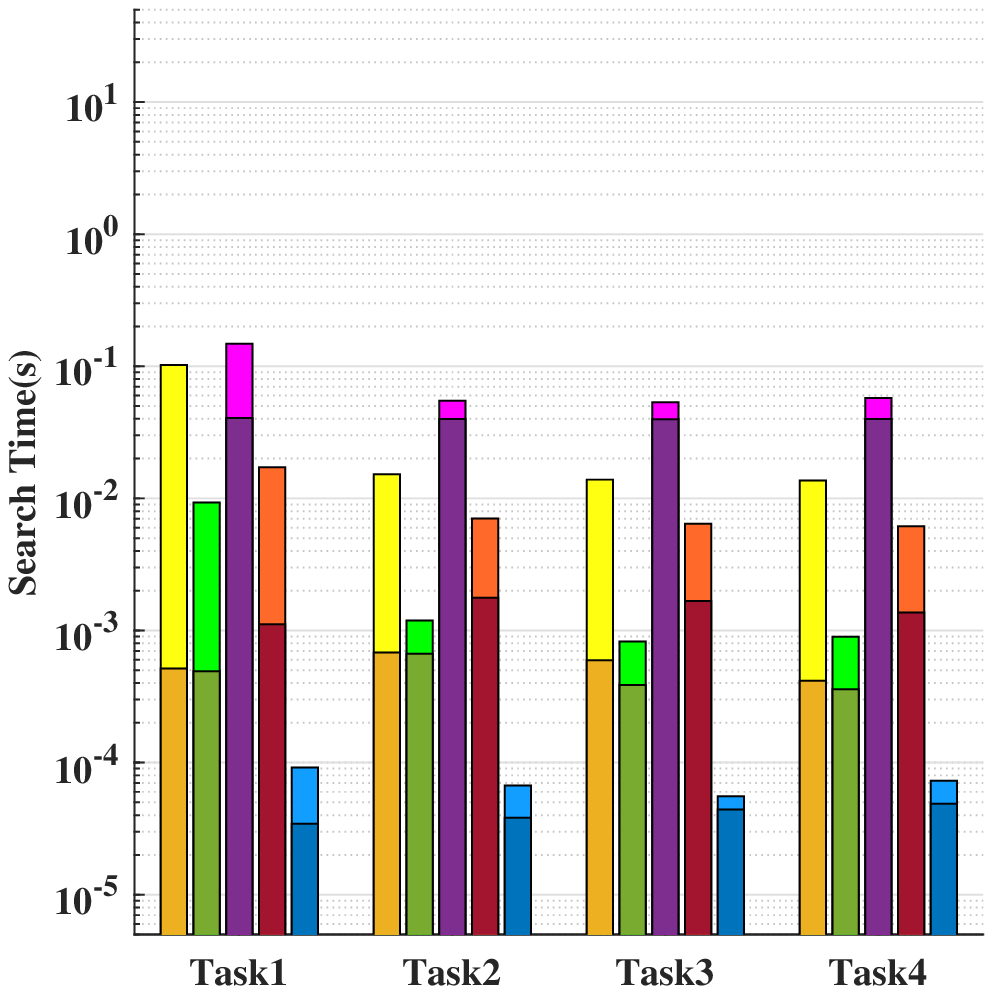}
    \label{fig_mapData_1}}
  \hfil
  \subfloat[]{\includegraphics[width=1.14in]{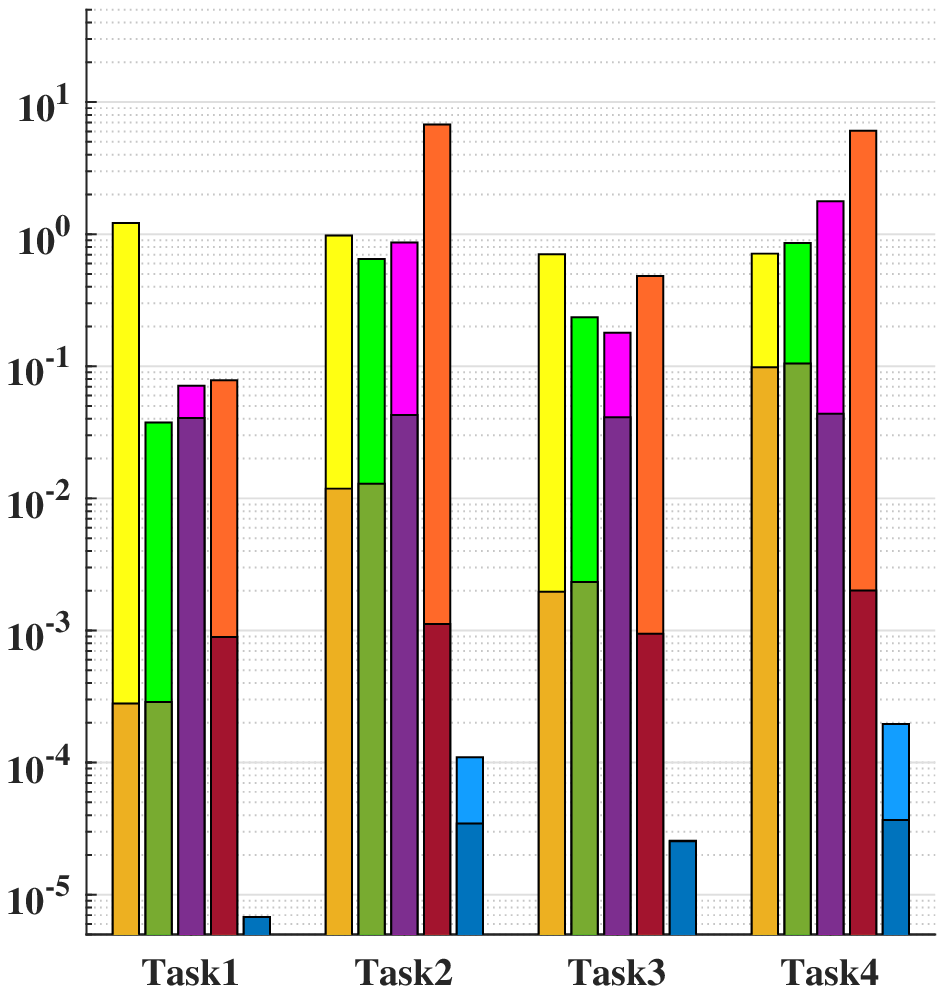}
    \label{fig_mapData_2}}
  \hfil
  \subfloat[]{\includegraphics[width=1.14in]{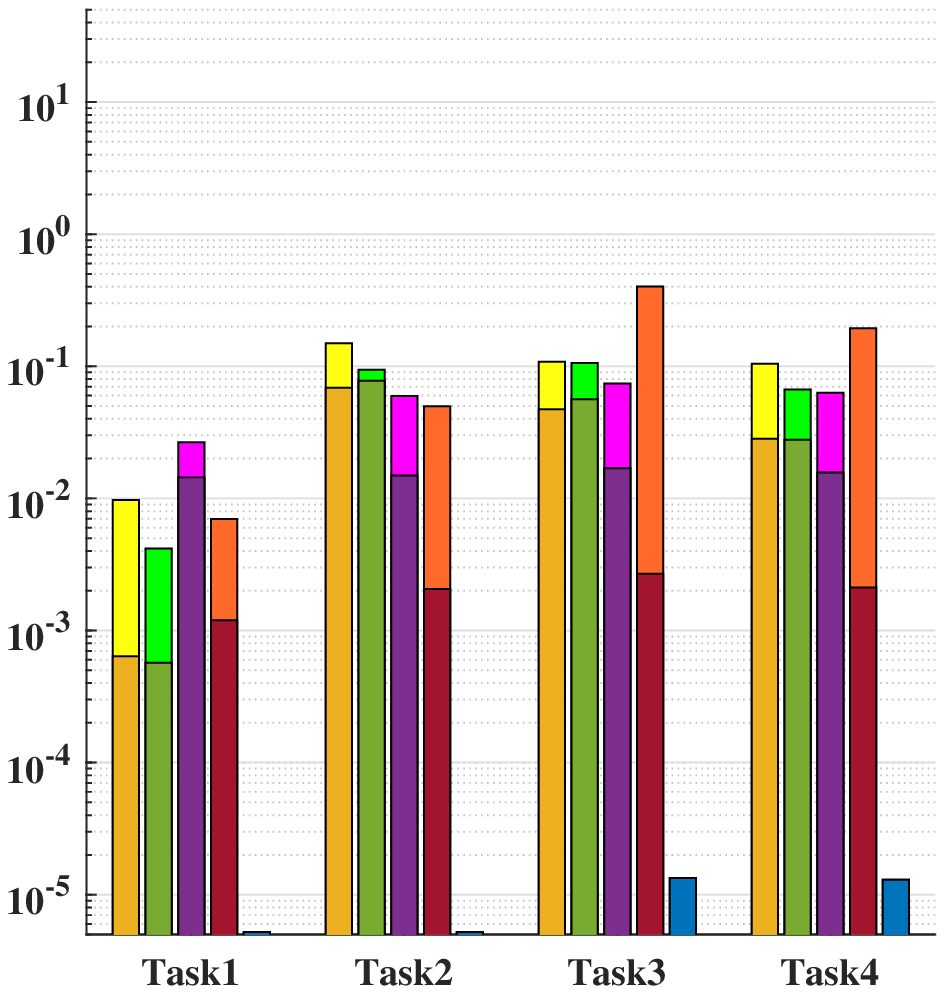}
    \label{fig_mapData_3}}
  \hfil
  \subfloat[]{\includegraphics[width=1.14in]{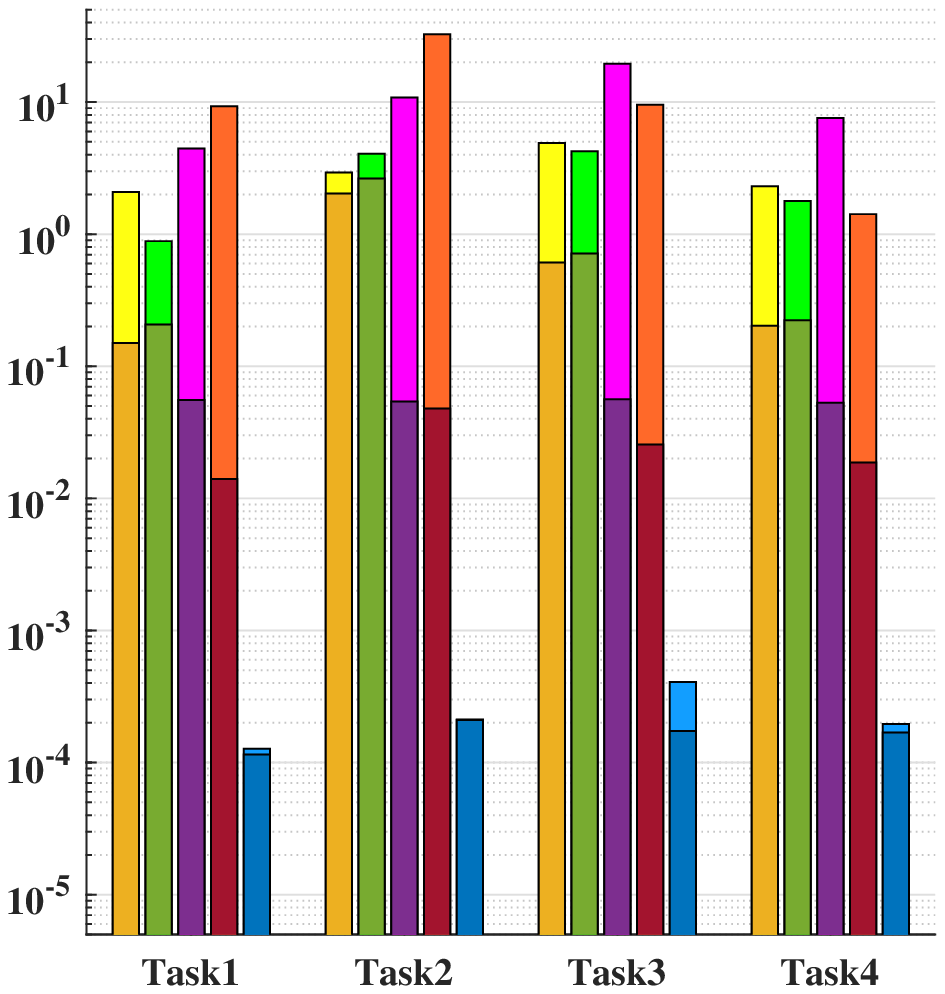}
    \label{fig_mapData_4}}
  \hfil
  \subfloat[]{\includegraphics[width=1.14in]{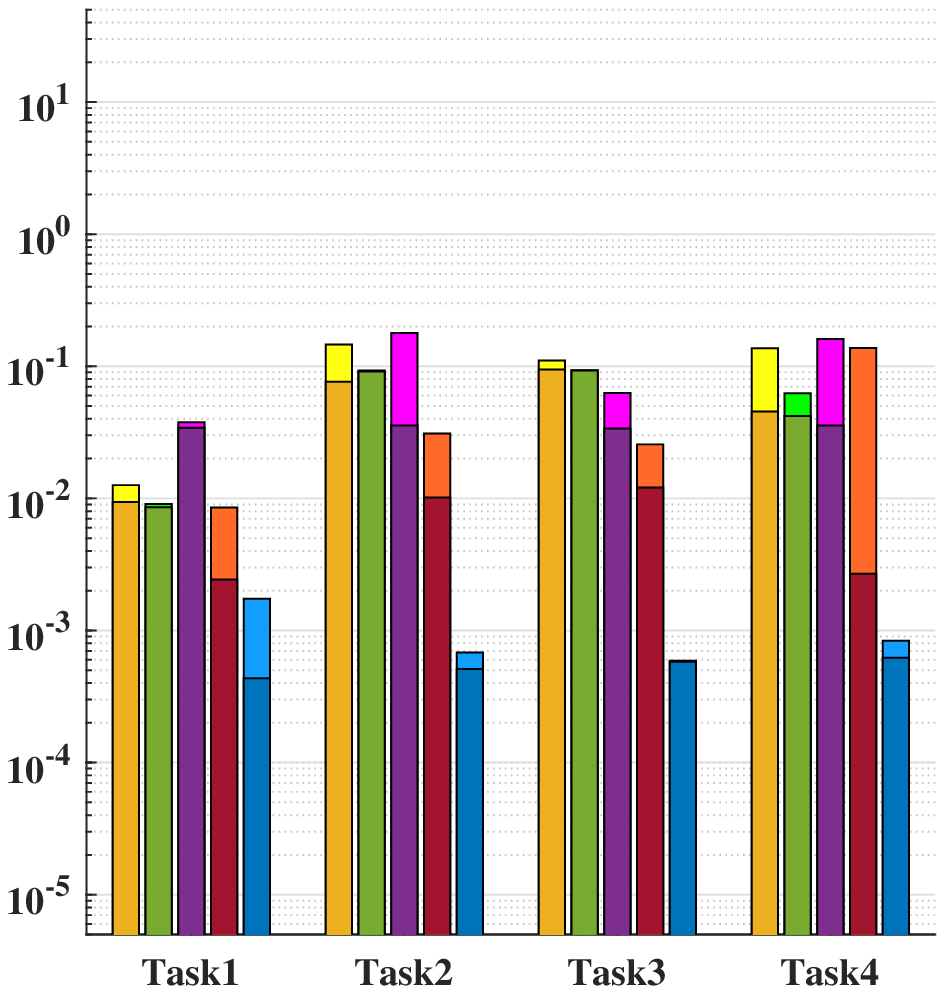}
    \label{fig_mapData_5}}
  \hfil
  \subfloat[]{\includegraphics[width=1.14in]{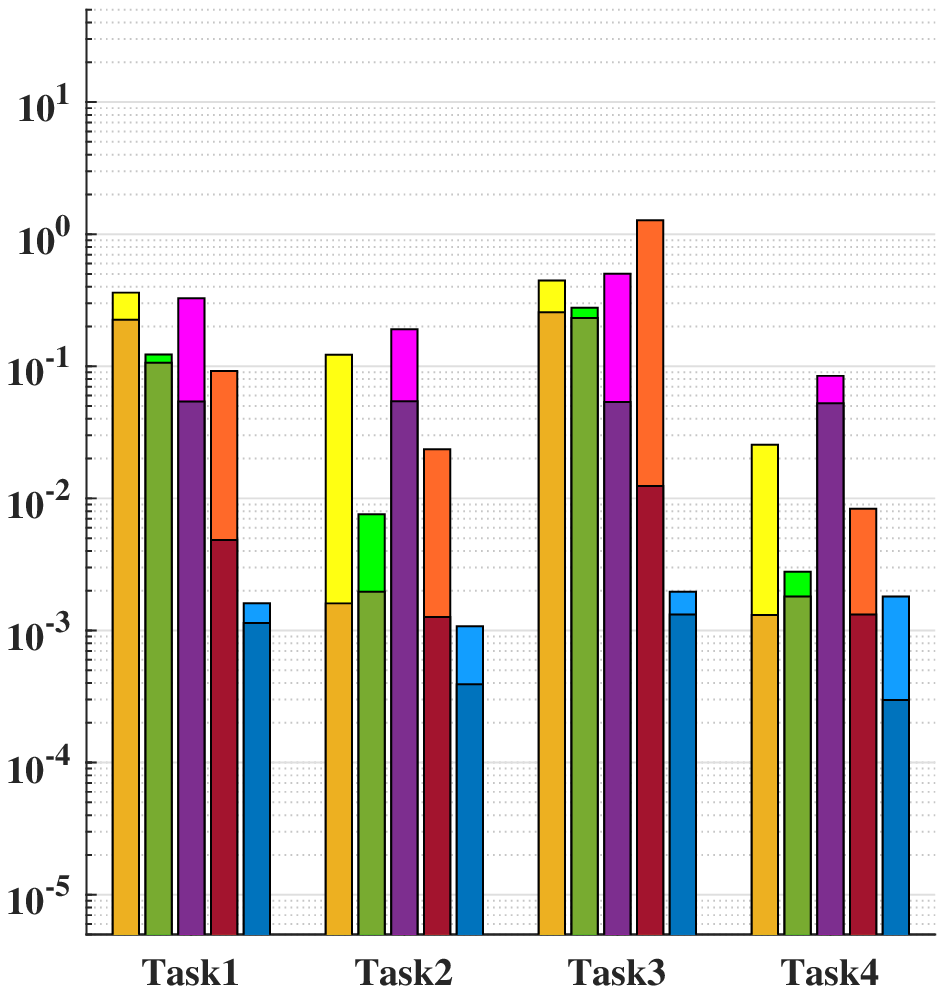}
    \label{fig_mapData_6}}
  \caption{Illustration of experiment map and planning tasks. (a) Cluttered environment, (b) Trap, (c) Maze1, (d) Maze2, (e) Floormap and (f) a StarCraft II map. Each environment contains four planning tasks starting from \textbf{S} and ending \textbf{1-4}. (g)-(l) correspond to the convex division results of each environment, respectively.}
  \label{fig_map}
\end{figure*}
The construction of the topology graph is an essential step in the proposed algorithms. In this subsection we present the results of the time required to construct the topology graph. The polygon fitting error $\varepsilon$ in (\ref{eq11}) it was set as $8$ pixels.

We conducted one hundred initialisation experiments on (a) Cluttered environment, (b) Trap, (c) Maze1, (d) Maze2, (e) Floormap and (f) a StarCraft II map in Fig.~\ref{fig_map}, respectively, and the results are shown in Fig.~\ref{fig_map} (g)-(l). Each environment contains four planning tasks for subsequent experiments, starting from \textbf{S} and ending \textbf{1-4}.

In \textbf{Table~\ref{tab:table1}}, we can find that the initialisation time of the algorithm for the map is not directly related to the resolution of the map.\footnote{It should be noted that in most cases, higher resolution maps tend to contain more details about the environment. Therefore, selecting a smaller fitting error $\varepsilon$ will result more complex simple polygons used to fit the free space, and the initialization time of the map will also be longer.} The initialisation result of the map is mainly affected by the chaos of the map itself. The chaos of the original map can affect the complexity of the simple polygons used to fit itself. This determines the time taken for initialisation and the resulting number of dividing lines and convex polygons.

From the test results, it can be found that the initialisation stage of the map in our algorithms does not require much time, even for a large and complex map such as the Floormap or StarCraft II map, the initialisation time is less than 0.4s. In actual use, only one initialisation is needed for one map, and we can save the initialisation result as a file and load it directly in future use. In fact, this is what we do.
\begin{table}
  \begin{center}
    \caption{Initialisation of Six Maps\label{tab:table1}}
    \centering
    \begin{tabular}{|c|c|c|c|c|}
      \hline
      \makecell*[c]{Map                                                      \\Name} & \makecell*[c]{Map Image\\Resolution} & \makecell*[c]{Cutlines} & \makecell*[c]{Convex\\Polygons} & \makecell*[c]{Time ($ms$)}\\
      \hline
      \makecell*[c]{Cluttered} & $1100*1100$ & $105$ & $94$  & $16.6\pm3.7$  \\
      \hline
      \makecell*[c]{Trap}      & $1320*1320$ & $48$  & $47$  & $8.9\pm1.4$   \\
      \hline
      \makecell*[c]{Maze1}     & $1320*1320$ & $83$  & $85$  & $14.0\pm2.8$  \\
      \hline
      \makecell*[c]{Maze2}     & $3310*3310$ & $403$ & $395$ & $112.7\pm6.6$ \\
      \hline
      \makecell*[c]{Floormap}  & $2338*2338$ & $535$ & $475$ & $147.5\pm4.3$ \\
      \hline
      \makecell*[c]{Starcraft} & $2599*2599$ & $727$ & $662$ & $368.1\pm5.9$ \\
      \hline
    \end{tabular}
  \end{center}
\end{table}

\subsection{Performance Analysis of Task Decoupling}
In this subsection, we will compare and analyze the performance impact of the task decoupling strategy \textbf{Remark~\ref{remarkDecouple}} on CDT-RRT*. To do this, we will disable lines 10-12 in \textbf{Algorithm~\ref{alg2}} and replace lines 25-30 with the following pseudocode:
\begin{algorithm}
  \begin{algorithmic}[1]
    \STATE $f \gets f*l^{x_{goal}}_{x_e}$
    \STATE \textbf{if} $f^\circledast$ is None \textbf{or $S(f)<S(f^\circledast)$ then}
    \STATE \hspace{0.5cm}$f^\circledast \gets f$
  \end{algorithmic}
\end{algorithm}

Afterwards, we conducted 1000 tests using the tasks shown in Fig.~\ref{fig_map} (e) to compare the performance of the algorithm before and after the modification. The experimental results are presented in \textbf{Table~\ref{tab:table2}}. Fig.~\ref{fig_ueDCost} shows the time-cost diagrams of the two algorithms in a single experiment.\footnote{Before analyzing the parameter $\beta$ in (\ref{eq_ST}) (in subsection D), all $\beta$ values were set to 1.} It is evident from \textbf{Table~\ref{tab:table2}} that the average time required for the undecoupled algorithm to obtain the optimal path is significantly longer than that of the decoupled algorithm. This is mainly because the decoupled algorithm, after sampling a new homotopy class using RRT*, directly obtains the optimal path within that homotopy class through \textbf{Algorithm~\ref{alg5}}, rather than slowly converging to the optimal path using the asymptotic optimality property of RRT*. This can be clearly observed in Fig.~\ref{fig_ueDCost}, where the decoupled algorithm shows an instantaneous drop in the Cost when discovery of new and better homotopy classes, while the undecoupled algorithm wastes significant time in the asymptotic convergence process, even after having found the homotopic path of the globally optimal path.

\begin{table}
  \begin{center}
    \caption{Performance Analysis of Task Decoupling\label{tab:table2}}
    \centering
    \begin{tabular}{|c|c|c|c|c|}
      \hline
      \multirow{3.0}*{Algorithm} & \multicolumn{4}{c|}{\makecell*[c]{Time $t_{2\%}$ ($us$)}}  \\
      \cline{2-5}
      ~ & \makecell*[c]{Task 1} & \makecell*[c]{Task 2} & \makecell*[c]{Task 3} & \makecell*[c]{Task 4}\\
      \hline
      \makecell*[c]{Decoupled} & $3608.8$ & $662.8$ & $592.6$  & $974.3$  \\
      \hline
      \makecell*[c]{Undecoupled}  & $8949.5$ & $13931.9$  & $14203.9$  & $13781.8$   \\
      \hline
    \end{tabular}
  \end{center}
\end{table}

\begin{figure}[!t]
  \centering
  \subfloat[]{\includegraphics[width=1.6in]{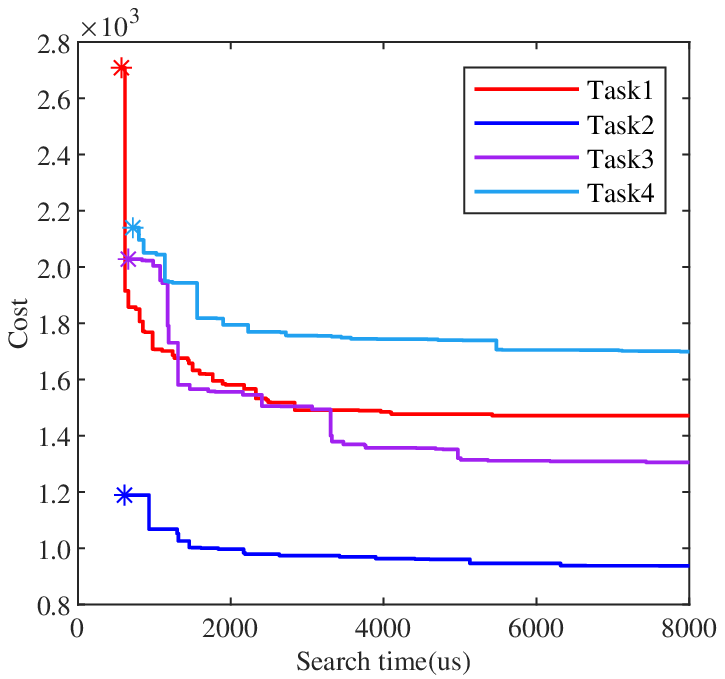}
    \label{fig_ueDCost_1}}
  \hfil
  \subfloat[]{\includegraphics[width=1.6in]{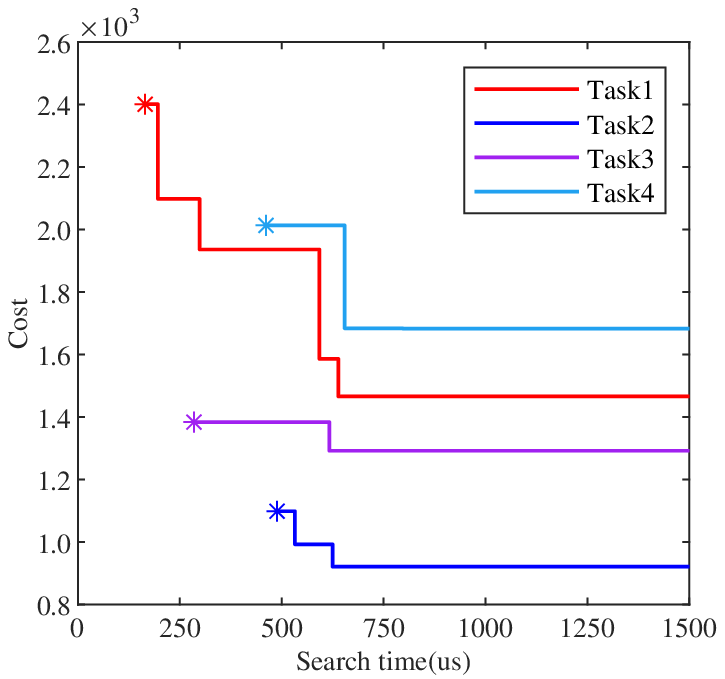}
    \label{fig_ueDCost_2}}
  \caption{Time-cost diagrams of CDT-RRT* with and without task decoupling. (a) Undecoupled. (b) Decoupled.}
  \label{fig_ueDCost}
\end{figure}

\subsection{The Effect of Reduce Branches on CDT-RRT*}
To analyze the impact of Reduce Branches on the performance of CDT-RRT*, in this subsection, we conducted experiments on tasks in the three representative maps shown in Fig.~\ref{fig_map} (a)(c)(d). For each task, we performed 1000 planning and recorded the mean time required to obtain the optimal path. The results are shown in \textbf{Table~\ref{tab:table3}}.

For the planning tasks in the Maze1 environment, CDT-RRT* achieves a fast planning speed regardless of whether Reduce Branches is used or not. In the first place, since the free space of Maze1 is simply connected. Therefore, after applying Reduce Branches, lines 10 and 11 of \textbf{Algorithm~\ref{alg2}} can always be triggered, thus directly obtaining the globally optimal path. For the case of not applying Reduce Branches, similarly, since the free space of Maze1 is simply connected, which means there is only one homotopy class between any two points in the space, therefore, the first homotopy class found by CDT-RRT* is the homotopy class where the global optimal path for the task resides. In addition, due to the term $\alpha 0^{\eta_i}$ in (\ref{eq_ST}) that enables $\mathcal{T}$ to rapidly expand in the initial stage and find the initial solutions for the task, therefore, CDT-RRT* can also quickly find the global optimal solution in the environment similar to Maze1 even without using Reduce Branches.

Different from Maze1, there are 9 and 12 independent obstacles in Maze2 and Cluttered, respectively, and none of the 4 tasks set in these two environments will trigger lines 10 and 11 in \textbf{Algorithm~\ref{alg2}}. However, there is a significant difference in the performance of CDT-RRT* with Reduce Branches in the two environments. In Maze2, the improvement in planning speed with Reduce Branches is more pronounced. This is mainly because Reduce Branches can eliminate many irrelevant nodes in $X_{topo}$, thereby reducing the number of cutlines that CDT-RRT* needs to consider when sampling. In the tasks of Maze2, more irrelevant cutlines are shielded. \textbf{Table~\ref{tab:table4}} shows the number of cutlines that CDT-RRT* needs to consider with and without using Reduce Branches in these tasks. The performance improvement of CDT-RRT* by using Reduce Branches is positively correlated with the number of cutlines that it shields, which is greatly influenced by the choice of environment and task, Although Reduce Branches mostly results in significant improvements for CDT-RRT*, it is not always effective, as seen in the experiments conducted in the Cluttered environment.

\begin{table}
  \begin{center}
    \caption{Impact of Reduce Branches on Performance\label{tab:table3}}
    \centering
    \begin{tabular}{|c|c|c|c|c|c|}
      \hline
      \multicolumn{2}{|c|}{\multirow{3.0}*{Environment}} & \multicolumn{4}{c|}{\makecell*[c]{Time ($us$)}}  \\
      \cline{3-6}
      \multicolumn{2}{|c|}{~} & \makecell*[c]{Task 1} & \makecell*[c]{Task 2} & \makecell*[c]{Task 3} & \makecell*[c]{Task 4}\\
      \hline
      \multirow{3.5}*{Maze1} &  \makecell*[c]{Reduce\\Branches}  & $6.6$ & $6.8$ & $15.3$  & $16.6$  \\
      \cline{2-6}
       & \makecell*[c]{No\\Action}       & $11.2$  & $24.3$  & $38.9$  & $39.6$   \\
       \hline
       \multirow{3.5}*{Maze2} &  \makecell*[c]{Reduce\\Branches}  & $138.1$ & $237.4$ & $816.6$  & $190.7$  \\
       \cline{2-6}
        & \makecell*[c]{No\\Action}       & $296.9$  & $572.2$  & $1451.9$  & $379.7$   \\
       \hline
       \multirow{3.5}*{Cluttered} &  \makecell*[c]{Reduce\\Branches}  & $111.3$ & $71.2$ & $61.6$  & $81.3$  \\
       \cline{2-6}
        & \makecell*[c]{No\\Action}       & $118.0$  & $78.4$  & $69.9$  & $89.8$   \\
       \hline
    \end{tabular}
  \end{center}
\end{table}

\begin{table}
  \begin{center}
    \caption{The Impact of Reduce Branches on the Number of Cutlines During Planning\label{tab:table4}}
    \centering
    \begin{tabular}{|c|c|c|c|c|c|}
      \hline
      \multicolumn{2}{|c|}{\multirow{3.0}*{Environment}} & \multicolumn{4}{c|}{\makecell*[c]{Number of Cutlines}}  \\
      \cline{3-6}
      \multicolumn{2}{|c|}{~} & \makecell*[c]{Task 1} & \makecell*[c]{Task 2} & \makecell*[c]{Task 3} & \makecell*[c]{Task 4}\\
      \hline
      \multirow{3.5}*{Maze1} &  \makecell*[c]{Reduce\\Branches}  & $12$ & $8$ & $23$  & $23$  \\
      \cline{2-6}
       & \makecell*[c]{No\\Action}       & \multicolumn{4}{c|}{\makecell*[c]{83}}   \\
       \hline
       \multirow{3.5}*{Maze2} &  \makecell*[c]{Reduce\\Branches}  & $242$ & $242$ & $256$  & $253$  \\
       \cline{2-6}
        & \makecell*[c]{No\\Action}       & \multicolumn{4}{c|}{\makecell*[c]{403}} \\
       \hline
       \multirow{3.5}*{Cluttered} &  \makecell*[c]{Reduce\\Branches}  & $99$ & $99$ & $99$  & $99$  \\
       \cline{2-6}
        & \makecell*[c]{No\\Action}       & \multicolumn{4}{c|}{$105$} \\
       \hline
    \end{tabular}
  \end{center}
\end{table}

\subsection{Analysis of the Random Sampling Equation}
In this subsection, we conducted experiments to analyze the effect of the $\alpha 0^{\eta_i}$ term in (\ref{eq_ST}) on the ability of CDT-RRT* to quickly find initial solutions, as well as the influence of the parameter $\beta $ on the optimization time ($t_{2\%}-t_{init}$) used by CDT-RRT* during the search for global optimal paths.

First, we tested the CDT-RRT* algorithm with and without the $\alpha 0^{\eta_i}$ term removed in all tasks in Fig.~\ref{fig_map}. We conducted 1000 experiments for each task and recorded the mean $t_{init}$ values, which are presented in \textbf{Table~\ref{tab:table5}}. We observed that using the $\alpha 0^{\eta_i}$ term generally resulted in a 3-6 times improvement in the algorithm's ability to find initial solutions. It should be noted that the $\alpha 0^{\eta_i}$ term is necessary in CDT-RRT* because without it, CDT-RRT* may never find an initial path, as reflected in the experimental results in the Trap task. In Tasks 2-3 of Trap, the $t_{init}$ without using $\alpha 0^{\eta_i}$ appears to be very abnormal. When processing experimental data, we found that in these three tasks, even though the algorithm exceeded the set limit of 40,000 sampling attempts, it still failed to find the initial path. The main cause of this phenomenon is related to the $\min(1,\mu_i)$ term in (\ref{eq_ST}). As explained for (\ref{eq_ST}), the role of $\min(1,\mu_i)$ is to "ensure that $Q_{near}$ is always non-empty in line 15 of \textbf{Algorithm~\ref{alg2}}". $\min(1,\mu_i)$ controls (\ref{eq_ST}) to only sample the cutlines and their neighbors that have been sampled. That is to say, when (\ref{eq_ST}) does not have $\alpha 0^{\eta_i}$, the probability of sampling the unsampled cutline is $|C_{ex}|/|C_{ex} \cup  C_{nb}|$, where $C_{ex}$ represents the set of cutlines that have been sampled, and $C_{nb}$ represents the set of cutlines adjacent to $C_{ex}$.  If $|C_{ex}|\gg |C_{ex}|$, the probability of unsampled branches in (\ref{eq_ST}) approaches zero, which means that CDT-RRT* will be trapped in the trap. Therefore, the $\alpha 0^{\eta_i}$ term is indispensable.

Parameter $\beta$ is used to explore other possible homology path classes after the initial solution is searched by CDT-RRT*. Therefore, in order to analyse the effect of parameter $\beta$ on CDT-RRT*, we performed 2000 path planning experiments for each value of $\beta$ from $(0,1]$ at 0.05 intervals and recorded the path optimisation time ($t_{2\%}-t_{init}$) for each experiment. To facilitate the display and analysis of the effect of parameter $\beta$, we additionally selected a planning task from point \textbf{3} to point \textbf{1} in the environment of Fig.~\ref{fig_map} (e) for experimentation.\footnote{The reason for doing this is that in the environment of Fig.~\ref{fig_map} (e), there are many local optimal paths between point \textbf{3} and point \textbf{1} that have lengths similar to the global optimal path. The purpose of setting the $\beta^{\kappa_i}$ term in (\ref{eq_ST}) is to enable CDT-RRT* to actively explore unknown homotopy classes, thereby finding the global optimal path faster.} Fig.~\ref{fig_AnalyzeBeta} shows the fit of the probability density function of the optimal time for each value of $\beta$ in the task, and the curve of the optimisation time mean with respect to parameter $\beta$. 

From the experimental results, we can determine that the value of $\beta$ has a large impact on the performance of CDT-RRT*. When $\beta = 1$, the sampling mode of CDT-RRT* degenerates to ordinary uniform sampling and the average optimisation time of CDT-RRT* is approximately $2.3$ times longer than when $\beta = 0.2$.  Selecting a smaller $\beta$ increases the sampling probability of cutlines that are used less frequently by the existing homotopy path classes, thus increasing the probability of CDT-RRT* discovery of other homology path classes. However, when $\beta$ is too small, the mean optimisation time value of CDT-RRT* increases instead, because a too small $\beta$ results in CDT-RRT* being more rigorous. CDT-RRT* primarily samples the least used cutlines, which unnecessarily consumes time on cutlines that do not fit within the optimal homotopy class. Following these results, CDT-RRT* with $\beta = 0.2$ was used for analysis in remaining experiments ().\footnote{In different environments and tasks, the optimal value of $B$ is different. It is difficult to quantitatively analyze how to choose the best $\beta$. However, based on our experience from a large number of experiments, when $\beta \in [0.07, 0.25]$, there is generally a good performance.}

\begin{figure}[!t]
  \centering
  \subfloat[]{\includegraphics[width=3.0in]{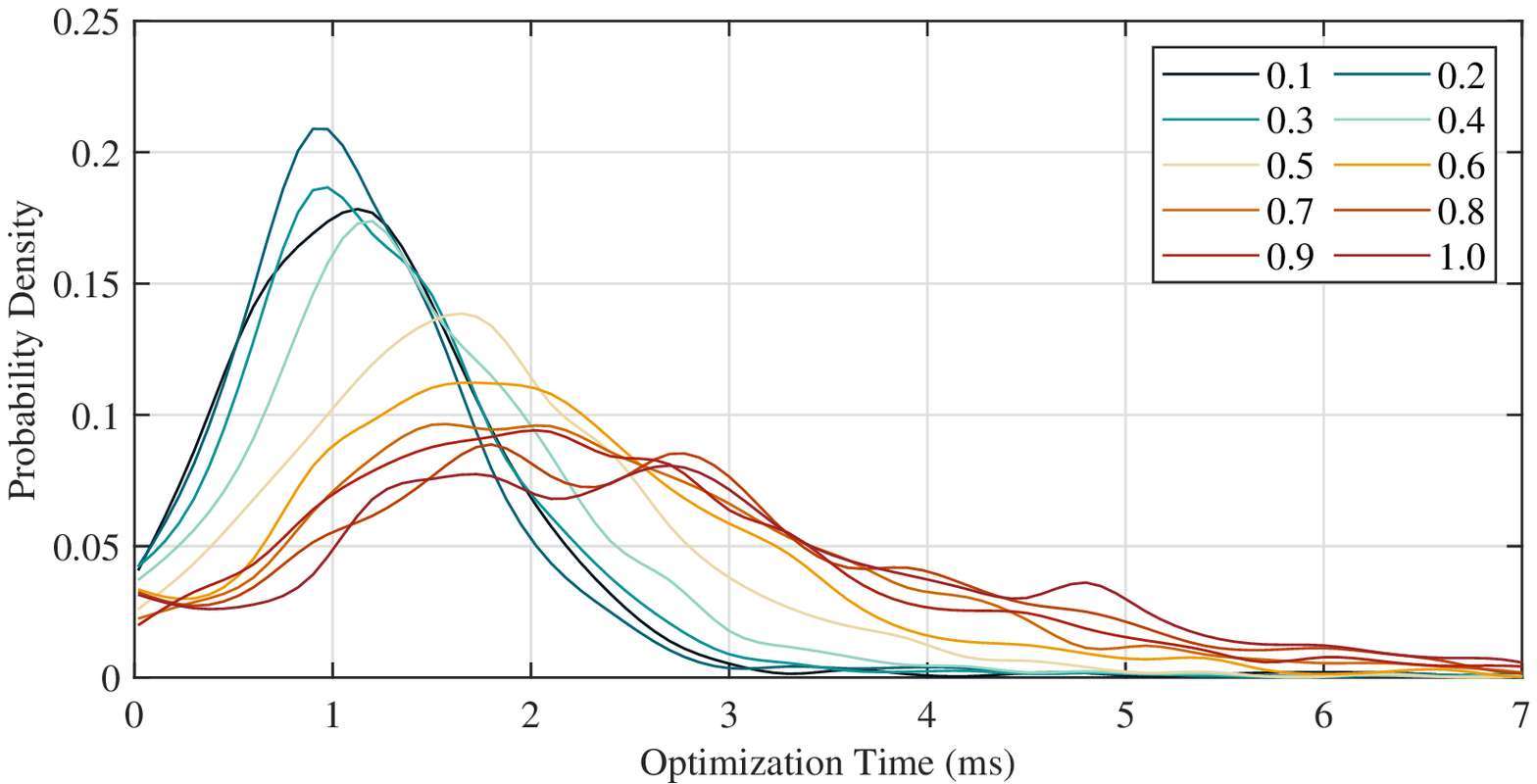}
    \label{fig_ABeta_1}}
  \hfil
  \subfloat[]{\includegraphics[width=3.0in]{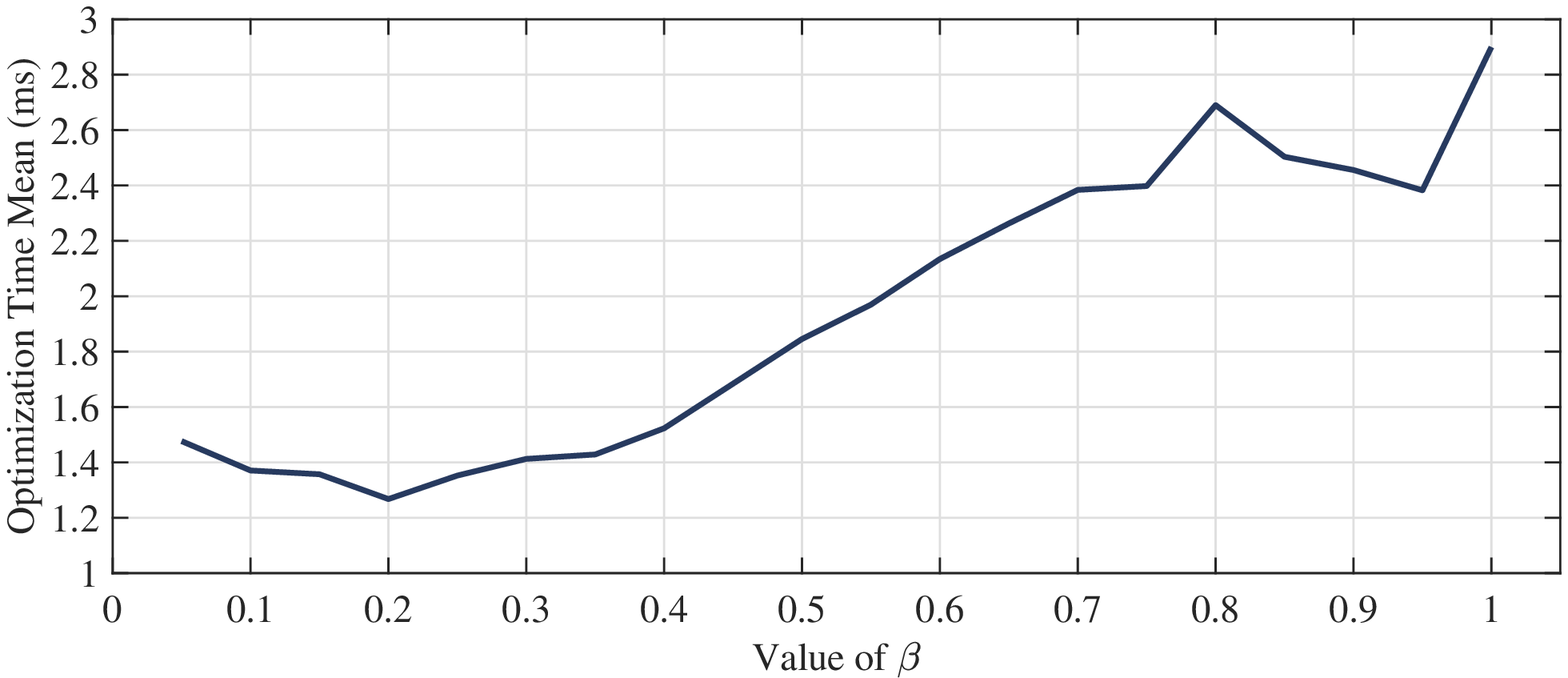}
    \label{fig_ABeta_2}}
  \caption{(a) Probability density plot of CDT-RRT* optimization time for different values of $\beta$. (b) The mean of the optimization time when $\beta$ takes different values. }
  \label{fig_AnalyzeBeta}
\end{figure}
\begin{table*}
  \begin{center}
    \caption{Average Value of $t_{init}$ for Each Task obtained by the CDT-RRT* when the $\alpha 0^{\eta_i}$ term is Enabled and Disabled.\label{tab:table5}}
    \centering
    \begin{tabular}{|c|c|c|c|c|c|c|c|c|}
      \hline
      \multirow{3.0}*{Environment} & \multicolumn{4}{c|}{\makecell*[c]{$t_{init}$ when $\alpha 0^{\eta_i}$ is disabled ($us$)}} & \multicolumn{4}{c|}{\makecell*[c]{$t_{init}$ when $\alpha 0^{\eta_i}$ is enabled ($us$)}}  \\
      \cline{2-9}
      & \makecell*[c]{Task 1} & \makecell*[c]{Task 2} & \makecell*[c]{Task 3} & \makecell*[c]{Task 4}
      & \makecell*[c]{Task 1} & \makecell*[c]{Task 2} & \makecell*[c]{Task 3} & \makecell*[c]{Task 4}\\
      \hline
      \makecell*[c]{Cluttered}  & $97.5$ & $103.1$ & $142.5$  & $156.8$  & $36.4$ & $40.1$ & $48.2$  & $55.9$  \\
      \hline
      \makecell*[c]{Trap}   & $29.5$ & $908.5$ & $459.6$  & $1163.7$  & $6.7$ & $34.5$ & $25.3$  & $36.6$  \\
      \hline
      \makecell*[c]{Maze1}   & $7.1$ & $27.2$ & $56.3$  & $58.4$  & $6.3$ & $7.2$ & $14.6$  & $14.3$  \\
      \hline
      \makecell*[c]{Maze2}   & $617.6$ & $1411.6$ & $1166.3$  & $952.3$  & $115.0$ & $209.8$ & $173.6$  & $169.2$  \\
      \hline
      \makecell*[c]{Floormap}   & $1015.4$ & $1472.8$ & $1860.1$  & $1794.5$  & $434.8$ & $511.7$ & $578.4$  & $621.4$  \\
      \hline
      \makecell*[c]{StarCraft}   & $4190.8$ & $965.9$ & $4814.6$  & $651.6$  & $1141.1$ & $391.6$ & $1319.2$  & $297.4$  \\
      \hline
    \end{tabular}
  \end{center}
\end{table*}

\subsection{Path Planning Efficiency of CDT-RRT*}
In this subsection, we compare CDT-RRT* with four state-of-the-art algorithms, RRT* algorithm, PRM* algorithm \cite{karaman2011sampling}, Informed-RRT* \cite{gammell2018informed}, and advanced batch informed trees algorithm (ABIT*) \cite{strub2020advanced}. This experiment continues to use the six environments in Fig.~\ref{fig_map}. These algorithms are run 200 times for each planning task and the planning time and cost are recorded for the evaluation.

The mean values of $t_{init}$ and $t_{2\%}$ for these algorithms for each task are shown in Fig.~\ref{fig_map} (m)-(r), and the legends in the figures correspond to the algorithms listed in \textbf{Table~\ref{tab:table6}}. We can find that with the proposed algorithm, the path planning time is obviously decreased. In every environment, $t_{init}$ and $t_{2\%}$ of CDT-RRT* are generally 2-4 orders of magnitude lower than other algorithms. Moreover, in the complex Fig.~\ref{fig_map} (e)(f), on average, CDT-RRT* successfully determined the initial solution to the task within 0.5 ms and the optimal path was planned within 6 ms. 

\begin{table}
  \caption{The Abbreviation for each criteria type for each algorithm\label{tab:table6}}
  \centering
  \begin{tabular}{|c|c|c|c|c|c|}
    \hline
    Criteria   & RRT*           & PRM*           & Informed-RRT*  & ABIT*          & CDT-RRT*          \\
    \hline
    $t_{init}$ & $R_\mathrm{f}$ & $P_\mathrm{f}$ & $I_\mathrm{f}$ & $A_\mathrm{f}$ & $C_\mathrm{f}$ \\
    \hline
    $t_{2\%}$  & $R_\mathrm{o}$ & $P_\mathrm{o}$ & $I_\mathrm{o}$ & $A_\mathrm{o}$ & $C_\mathrm{o}$ \\
    \hline
  \end{tabular}
\end{table}

\subsection{Performance in Real-World}

\begin{figure}[!t]
  \centering
  \subfloat[]{\includegraphics[width=3.0in]{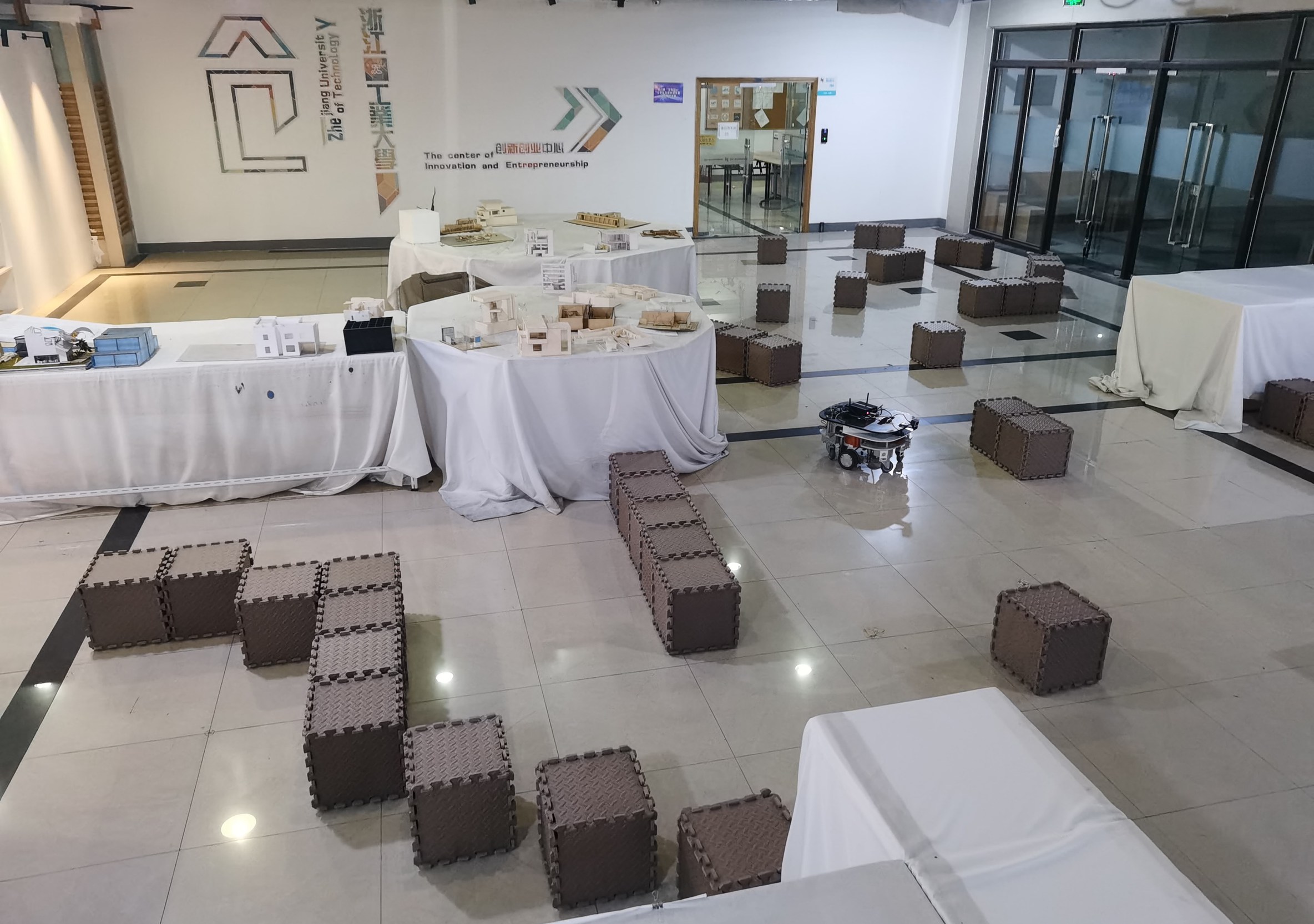}
    \label{fig_RealWorld_1}}
  \\
  \subfloat[]{\includegraphics[width=1.5in]{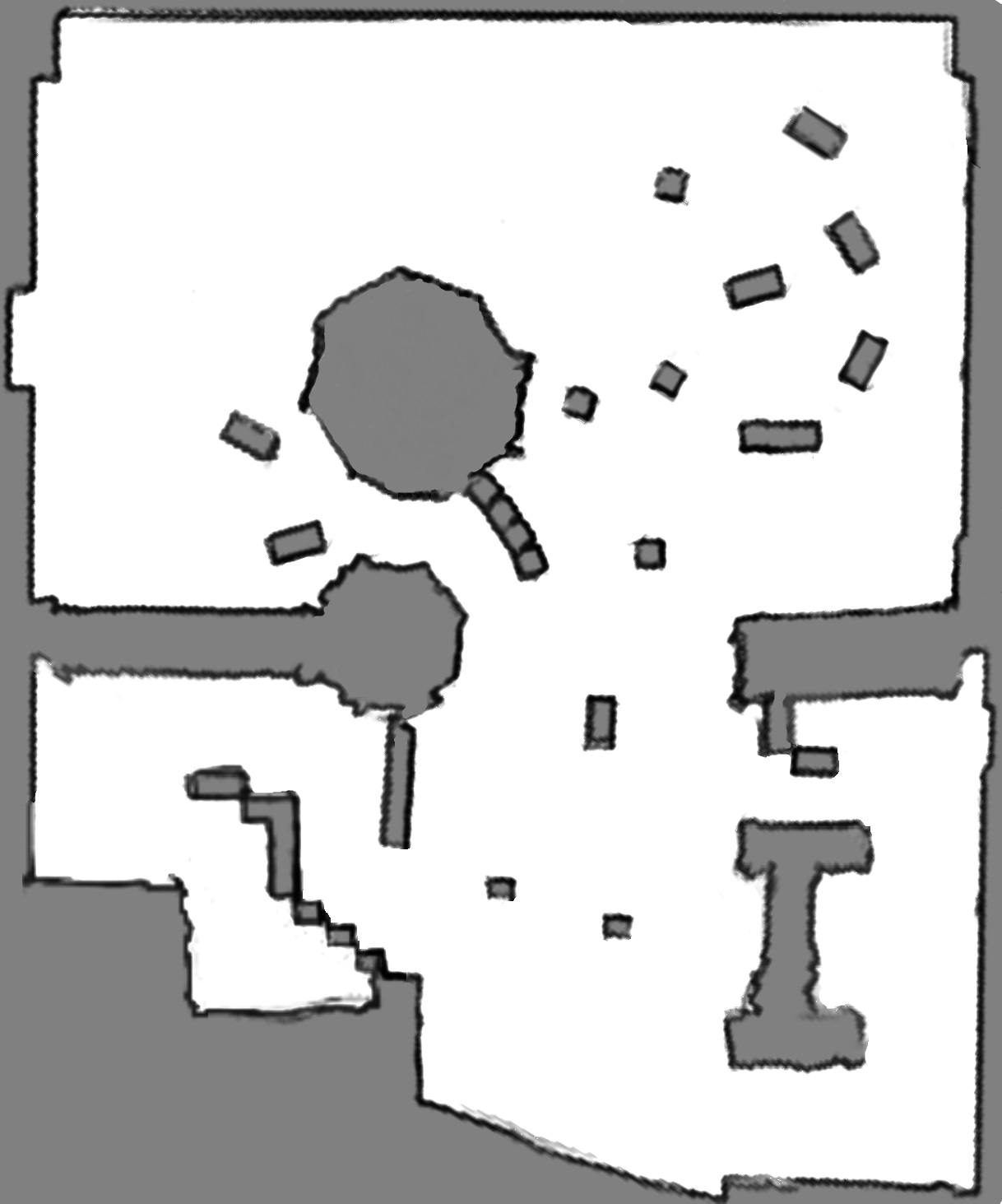}
    \label{fig_RealWorld_2}}
  \hfil
  \subfloat[]{\includegraphics[width=1.5in]{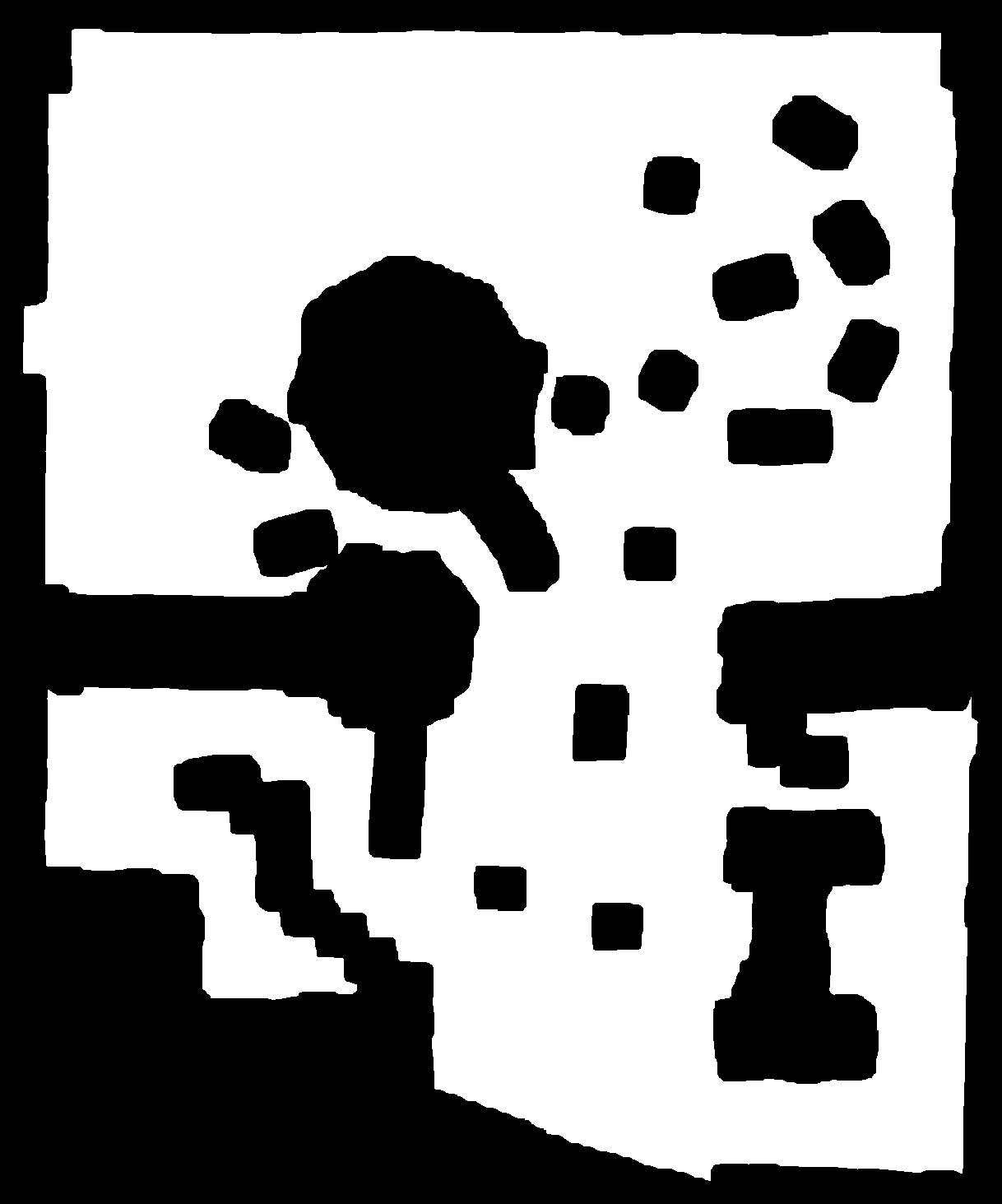}
    \label{fig_RealWorld_3}}
  \caption{Illustration of a real-world experimental site. (a) Photos of the exhibition hall. (b) 2D map of the exhibition hall created with cartographer. (c) Map inflated with robot radius.}
  \label{fig_RealWorld}
\end{figure}

To demonstrate the effectiveness of the proposed algorithms in practical applications, we present, hereunder, the results of experiments conducted in realistic scenarios. The experimental environment is shown in Fig.~\ref{fig_RealWorld}, where the test site was a $12 \times 14\ m^2$ exhibition hall and the resolution of the map was $1\ cm$. The robot system used in the experiment is shown in Fig.~\ref{fig_RobotSystem}. The system hardware is a self-designed wheeled mobile robot with a minimum external circle radius $r = 25\ cm$, maximum linear speed $v_{max} = 0.6\ m/s$, maximum angular speed $\omega_{max} = 1.57\ rad/s$, maximum acceleration $\dot{v} _{max} = 6.0\ m/s^2$, and maximum angular acceleration $\dot{\omega} _{max} = 15.7\ rad/s^2$. The robot used a fusion of laser, IMU and wheel tachometer for positioning. The movement control of the robot adopted a global and local motion planning strategy. The global path planning algorithms adopted RRT*, PRM* and CDT-RRT* algorithms, and the local path planning uniformly adopted the Dynamic Window Approach (DWA) method \cite{fox1997dynamic}. 

Considering that the global planning should ensure that the robot cannot touch the obstacles, we inflated the obstacles on the original map with radius $r$, as shown in Fig.~\ref{fig_RealWorld} (c).

\begin{figure*}[!t]
  \centering
  \includegraphics[width=6.5in]{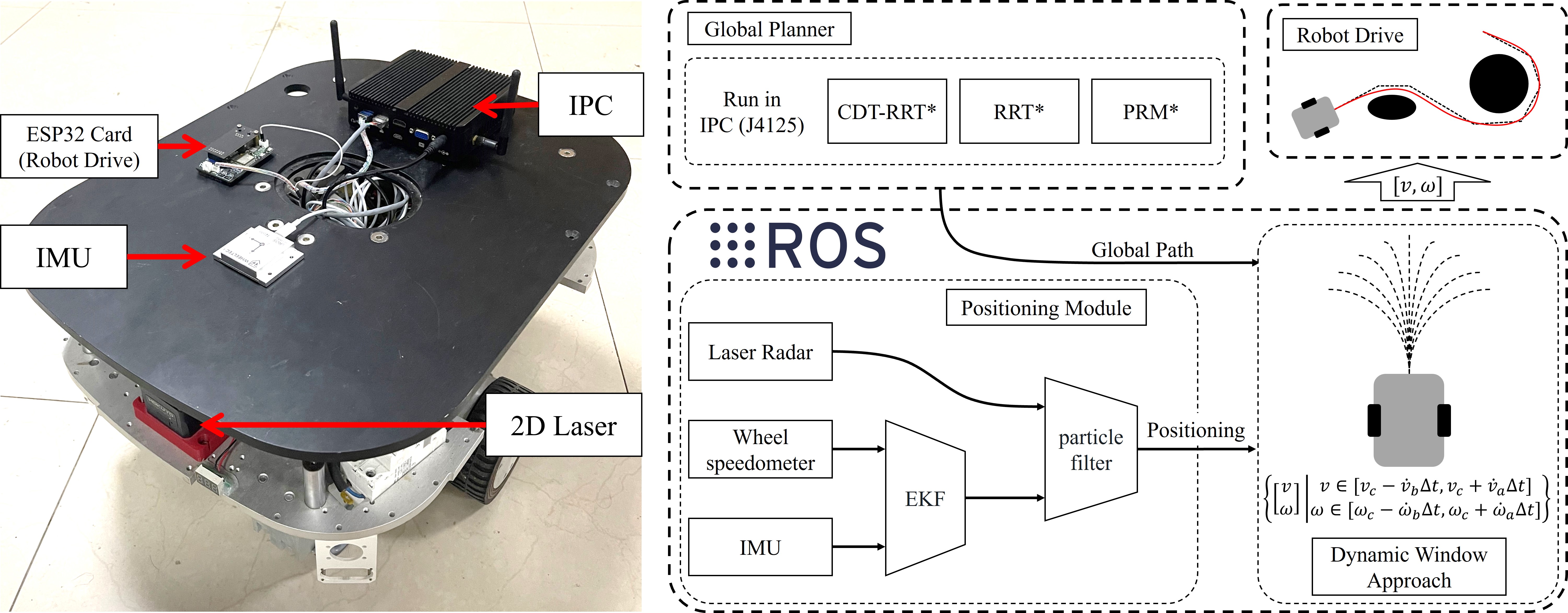}
  \caption{Robotic systems used in real-world experiments.}
  \label{fig_RobotSystem}
\end{figure*}

In the map we choose two representative planning tasks (approximate shortest non-homotopic paths and narrow passage) to test the three global planning algorithms. Because all three algorithms are optimal path planning algorithms, if the planning time is infinite, their final planning results must be equal. To comply with realistic demands of practical robotics applications, we set the upper time duration of global path planning to $0.5\ s$. \textbf{Table~\ref{tab:table7}} lists the data that we obtained from 10 repeated trials, and the results of two representative experiments are shown in Fig.~\ref{fig_RWtask}. Robots using the proposed algorithm for global path planning have shorter trajectory lengths and moving times for navigation compared to other algorithms. Furthermore, since the proposed algorithm was used to generate global paths for further motion planning, the real trajectories were generated by the DWA algorithm control. A key concern is to determine if the robot's trajectory and the optimal path belong to the same homotopy path class. From the experimental results, it is evident that when there are multiple approximate shortest non-homotopic paths, the trajectories planned by the RRT* and PRM* algorithms in the specified time had a higher probability of non-homotopy with the optimal path. Furthermore, using extreme narrow passage, optimal path planning could not be achieved. The robot trajectory using the proposed algorithm was always homotopic to the optimal path.
\begin{table}
  \begin{center}
    \caption{Real-World Experimental Data\label{tab:table7}}
    \centering
    \begin{tabular}{|c|c|c|c|c|c|}
      \hline
      \makecell*[c]{Experiment} & \makecell*[c]{Planning\\time ($ms$)} & \makecell*[c]{Moving\\time ($s$)}  & \makecell*[c]{Trajectory\\length ($m$)} & \makecell*[c]{Homotopy} \\
      \hline
      \multicolumn{5}{|c|}{\makecell*[c]{Approximate Shortest Non-Homotopic Paths Task}}\\
      \hline
      \makecell*[c]{RRT*}  & $475.5$  & $52.77$ & $19.76$ & $70\%$ \\
      \hline
      \makecell*[c]{PRM*}  & $461.6$  & $51.65$ & $19.66$ & $60\%$ \\
      \hline
      \makecell*[c]{CDT-RRT*}  & $5.32$  & $49.92$ & $18.65$ & $100\%$ \\
      \hline
      \multicolumn{5}{|c|}{\makecell*[c]{Narrow Passages Task}}\\
      \hline
      \makecell*[c]{RRT*}  & $500.0$  & $35.35$ & $15.21$ & $0\%$ \\
      \hline
      \makecell*[c]{PRM*}  & $500.0$  & $33.71$ & $15.13$ & $0\%$ \\
      \hline
      \makecell*[c]{CDT-RRT*}  & $4.03$  & $28.23$ & $13.37$ & $100\%$ \\
      \hline
    \end{tabular}
  \end{center}
\end{table}

\begin{figure}[!t]
  \centering
  \subfloat[]{\includegraphics[width=1.5in]{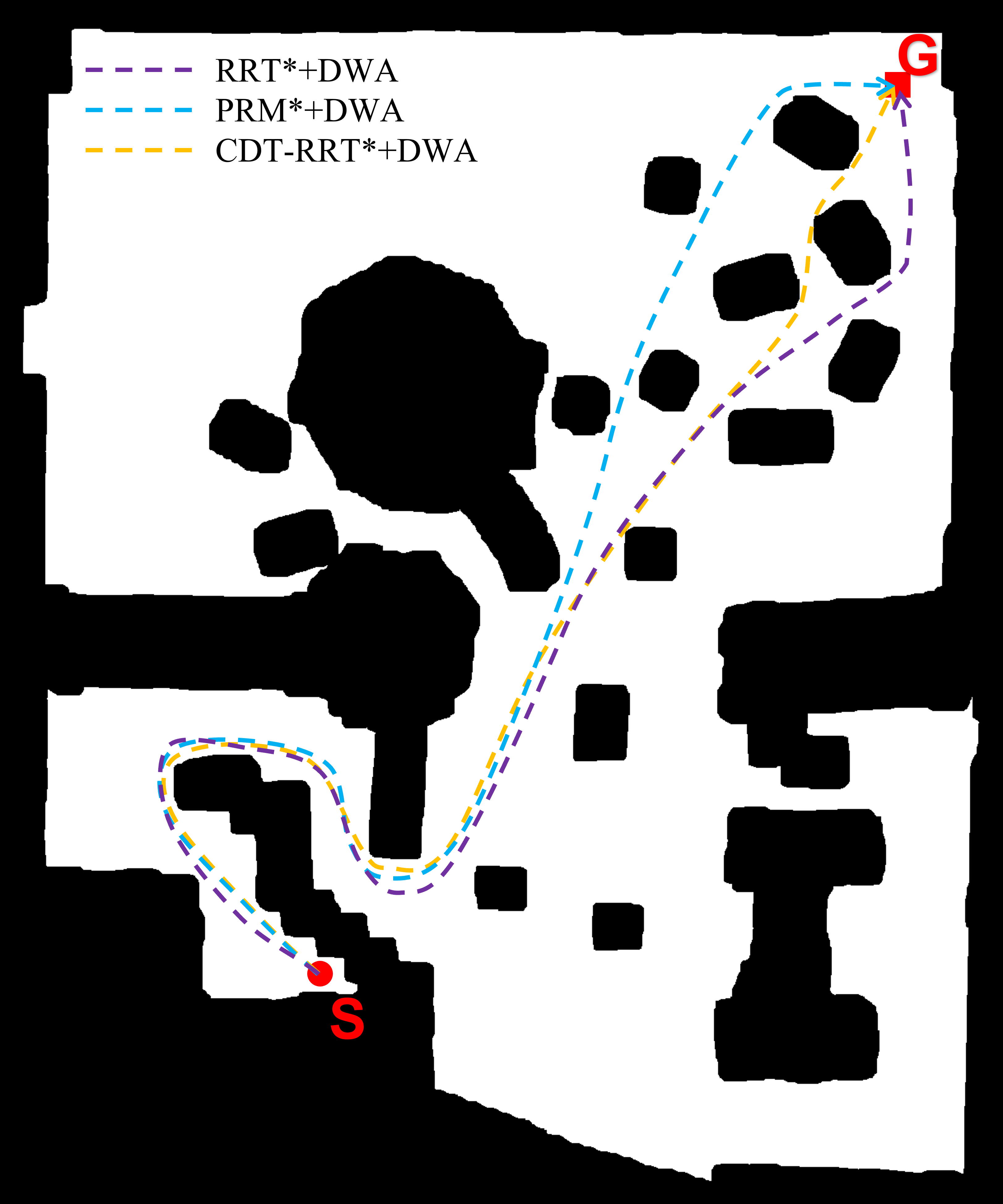}
    \label{fig_RWtask_1}}
  \subfloat[]{\includegraphics[width=1.5in]{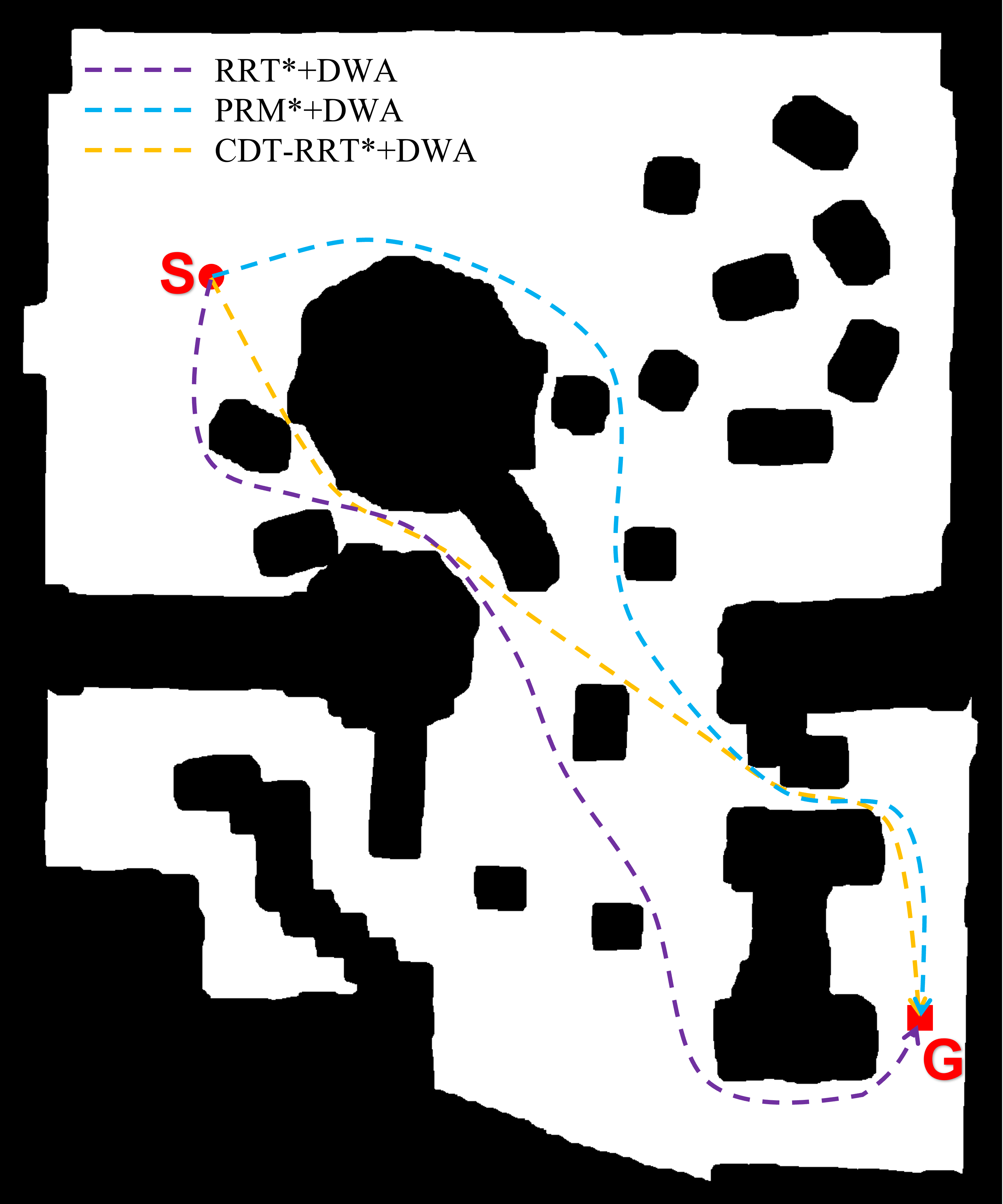}
    \label{fig_RWtask_2}}
  \caption{Robot trajectories using three different global path planning algorithms in two planning tasks. (a) Tasks where approximate shortest non-homotopic paths exist. (b) Tasks where narrow passages exist.}
  \label{fig_RWtask}
\end{figure}

\subsection{Discussions}
The results in the previous subsections demonstrate the efficiency of the proposed algorithms in a 2D connected space. However, most free spaces contain multiple connected components, as shown in Fig.~\ref{fig_8} (a). This can result in errors in the algorithm if the start and end points of the path planning task are not within the same connected component. This type of task is impractical; however, for some complex maps, the person assigning the task is often unable to determine in advance whether two points can be connected to each other.

This problem can be overcome by adding some logic to the algorithms. According to (\ref{eq7}), any 2D bounded free space can be expressed as a concatenation of multiple connected branches, such that we can fit such free space with multiple simple polygons, as shown in Fig.~\ref{fig_8} (b). For each simple polygon we can construct an independent topology graph based on convex division. The algorithm first determines whether the start and end points are in the same simple polygon each time the planning task is received, if not it implies that the two points cannot be connected, and the algorithms will not execute the subsequent sections and return a hint message.
\begin{figure}[!t]
  \centering
  \subfloat[]{\includegraphics[width=1.0in]{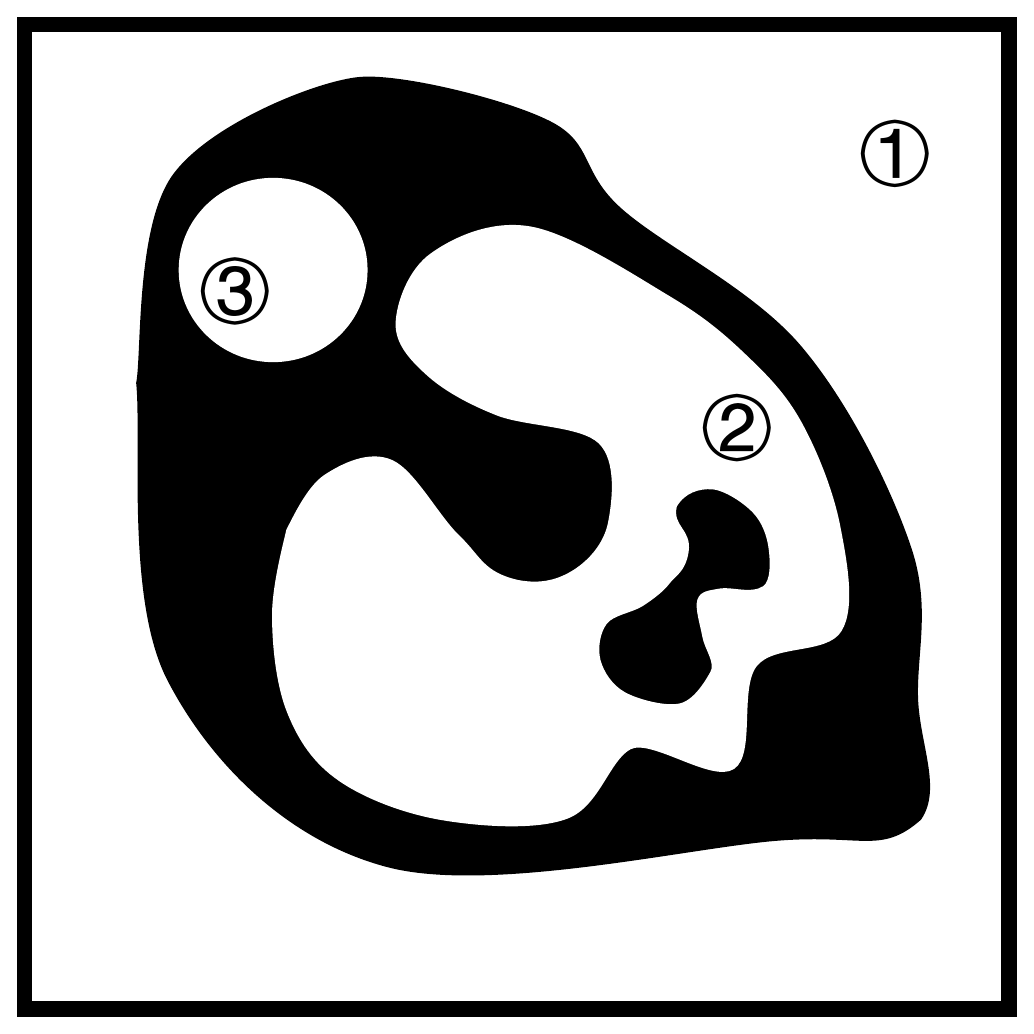}
    \label{fig_8_1}}
  \hfil
  \subfloat[]{\includegraphics[width=1.0in]{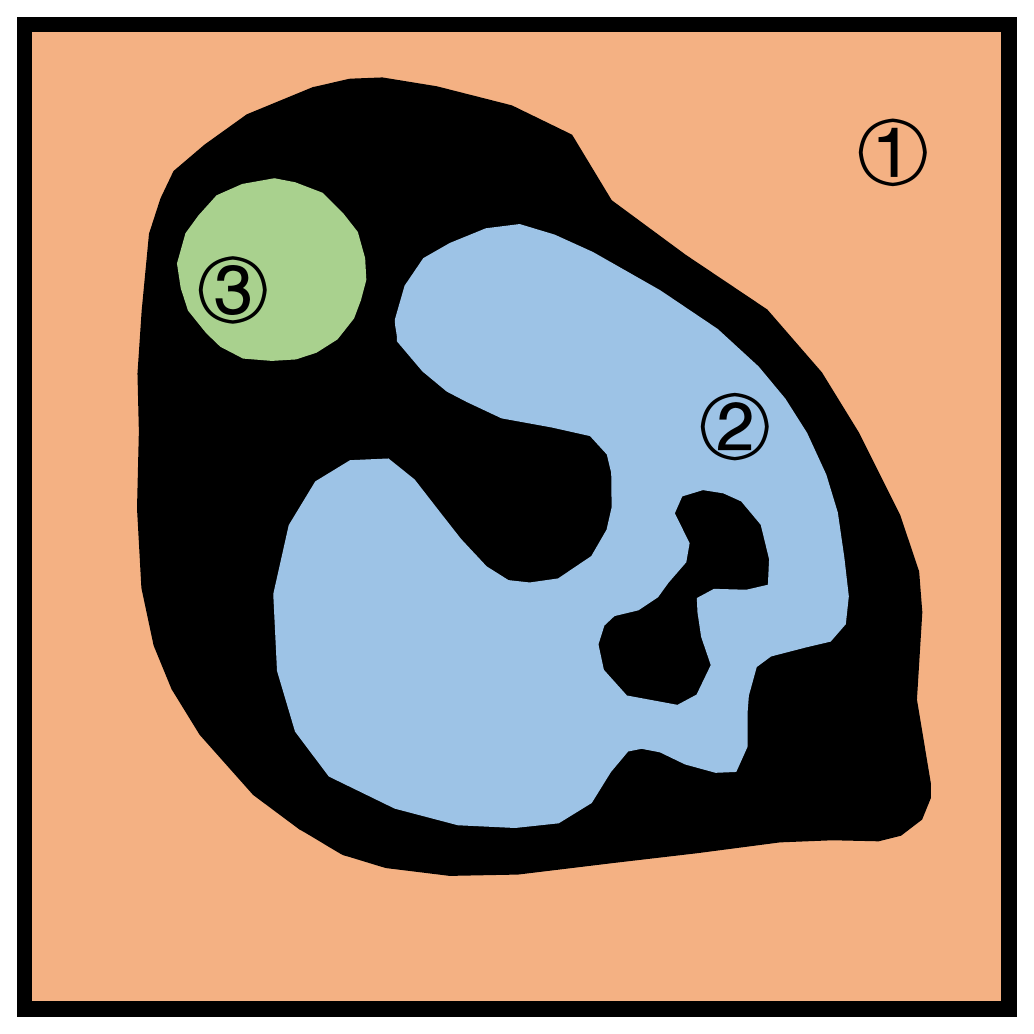}
    \label{fig_8_2}}
  \hfil
  \subfloat[]{\includegraphics[width=1.0in]{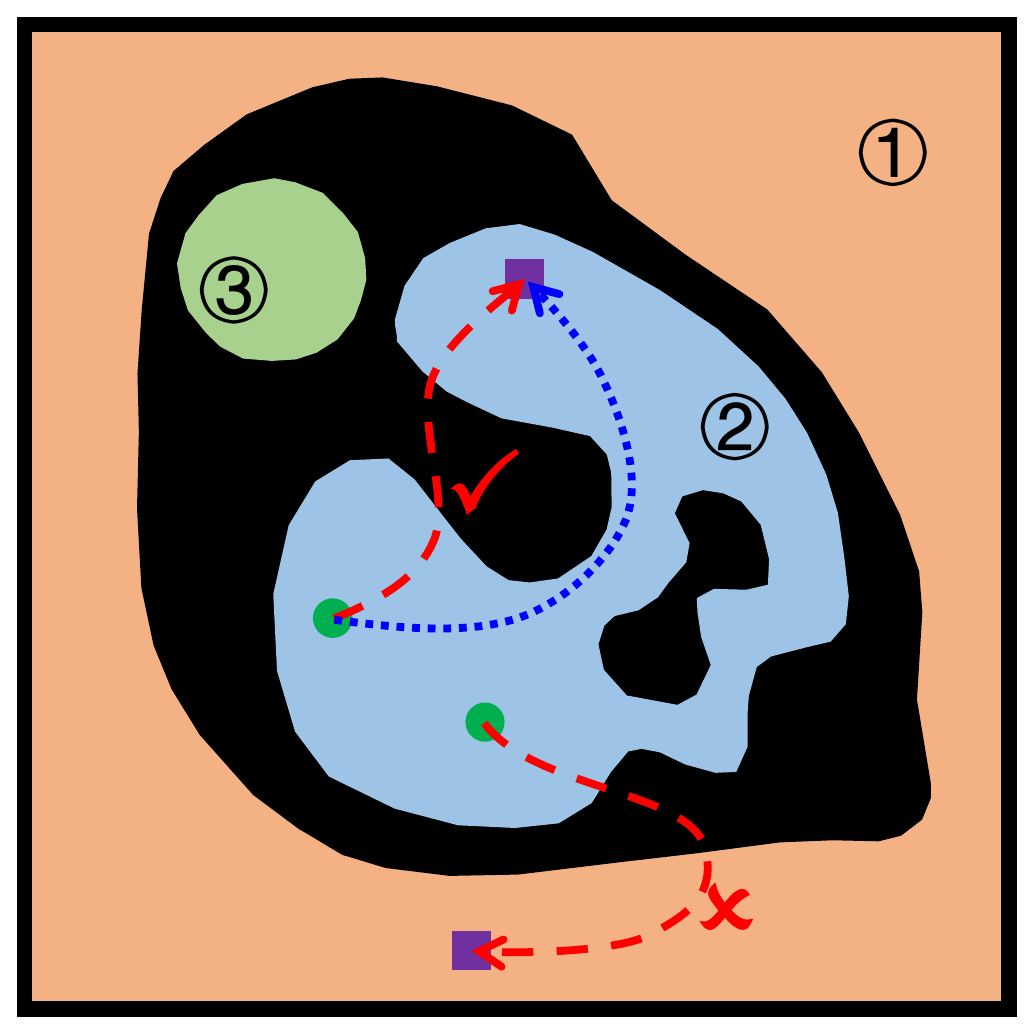}
    \label{fig_8_3}}
  \caption{(a) A free space with three connected branches. (b) Fitting to free space using three simple polygons. (c) Path planning is performed only when the start and end points are in the same polygon.}
  \label{fig_8}
\end{figure}
\section{Conclusion}
In this paper, an efficient optimal path planning framework has been proposed for mobile robot motion planning in 2D bounded environments. Firstly, based on the convex decomposition of the map, a special topological graph construction method is proposed to give a concise representation of map connectivity. On the basis of this topology graph, an encoder for the homotopy path class is proposed, proving its existence and uniqueness for the encoded values of all homotopy path classes between any two points and providing a feasible decoding method. Based on the convex decomposition of the map proposed in this article, a planning principle that considers only the points on the cutline is proposed, thus reducing the state space from 2D to 1D. Based on the above conclusions, an efficient optimal path planning algorithm CDT-RRT* are proposed, and we designed an efficient sampling formula for CDT-RRT*, which gives it a tendency to actively explore unknown homotopy classes. We conducted experiments and analyses on the topology construction, path planning efficiency, and performance of the algorithm in both simulation and real-world environments. The results verify that our proposed method can achieve satisfactory performance in 2D bounded environments. The application value of homotopy path class encoder in the field of path planning is proved.

As a simple application example, CDT-RRT* has achieved good results in experiments; however the encoder proposed in this paper has scope for improved application. The scope for future work can follow two directions. One is to design search algorithms for homotopy path classes with better performance based on the encoder in this paper, e.g., In learning-based algorithms, the encoder is used to record the collected homotopy classes of paths, and then to learn high-level features of the agent's motion in the environment using a learning-based method. The other is to improve the convex dissection method of the map, e.g., improving the speed of the dissection algorithm to make the present framework suitable for real-time path planning in dynamic environments \cite{naderi2015rt}. Alternatively, the framework can be applied to 2D manifold surfaces in high-dimensional spaces \cite{li2016novel} to broaden the applicability of the algorithm to provide efficient path planning services for more types of mobile robots \cite{hooks2020alphred}.

\bibliographystyle{IEEEtran}
\bibliography{IEEEabrv,uref}






\vfill

\end{document}